\documentclass[letterpaper,10pt,journal]{IEEEtran}

\usepackage[pdftex]{graphicx}
\usepackage{transparent}
\usepackage{amsmath,amsfonts,amsthm}
\usepackage{amssymb}
\usepackage{mathrsfs}
\usepackage{amsxtra}
\usepackage{dsfont}
\usepackage{algorithm}
\usepackage{algorithmicx}
\usepackage{algpseudocode}

\newcommand{\mbf}[1]{\mathbf{#1}}
\newcommand{\fat}[1]{\mathds{#1}}
\newcommand{\hsm}[1]{\mathscr{#1}}

\newcommand{\ZZ}{\fat{Z}}

\newcommand{\dd}{\mathrm{d}}
\newcommand{\at}[1]{\big|_{#1}}
\newcommand{\indicator}[1]{\fat{1}_{#1}}
\newcommand{\wild}{{\scriptscriptstyle\bullet}}

\newcommand{\set}[2]{\left\{#1\left|#2\right.\right\}}
\newcommand{\gen}[2]{\left\langle #1\left|#2\right.\right\rangle}
\newcommand{\THEN}{\Rightarrow}
\newcommand{\IFF}{\Leftrightarrow}
\newcommand{\DEFF}{\;\stackrel{\begin{array}{c}
	\scriptscriptstyle{def.}\\
\end{array}}{\Longleftrightarrow}}		
\newcommand{\minus}{\smallsetminus}		
\newcommand{\symplus}{\vartriangle}		
\newcommand{\into}{\hookrightarrow}		
\newcommand{\id}[1]{\mathtt{id}_{#1}}	
\newcommand{\cl}[1]{\overline{#1}}		
\newcommand{\power}[1]{\mathbf{2}^{#1}}	
\newcommand{\card}[1]{\left|#1\right|}	
\newcommand{\ev}[2]{\left\langle #1:#2\right\rangle} 
\newcommand{\inv}{{^{\scriptscriptstyle -1}}}	
\newcommand{\norm}[1]{\left\Vert #1\right\Vert}	
\newcommand{\res}[1]{\!\left|_{#1}\right.}		
\newcommand{\dist}[2]{\mathtt{dist}\!\left(#1,#2\right)}	

\newcommand{\catpoc}{\mbf{Poc}_f}
\newcommand{\catmed}{\mbf{Med}_f}
\newcommand{\catcub}{\mbf{Cub}_f}

\newcommand{\sens}{\Sigma}					
\newcommand{\spc}{\mbf{X}}					
\newcommand{\env}{\mbf{E}}					
\newcommand{\model}{\mathbf{M}}				
\newcommand{\punct}{\mathbf{M}^{\times}}	

\newcommand{\ellone}[2]{\mbf{\Delta}\!\left(#1,#2\right)} 
\newcommand{\minP}{\mathbf{0}}		
\newcommand{\maxP}{\mathbf{0^\ast}}	
\newcommand{\up}[1]{#1\!\uparrow}		
\newcommand{\down}[1]{#1\!\downarrow}	
\newcommand{\coh}[1]{\mathtt{coh}(#1)}			
\newcommand{\med}[3]{\mathrm{med}\!\left(#1,#2,#3\right)} 	
\newcommand{\morph}[2]{\mathrm{Hom}\!\left(#1,\,#2\right)}	
\newcommand{\dual}[1]{\mathtt{Dual}\!\left(#1\right)}		
\newcommand{\cube}[1]{\mathtt{Cube}\!\,(#1)}				
\newcommand{\punc}[1]{\mathtt{Cube}^{\times}\!(#1)}			
\newcommand{\flip}[2]{\left[#1\right]_{_{#2}}}		
\newcommand{\half}[1]{\mathfrak{h}\!\left(#1\right)}		
\newcommand{\sep}[1]{\mathtt{sep}\!\left(#1\right)} 

\newcommand{\pos}{\mathtt{pos}}			
\newcommand{\mytime}{\fat{T}}			

\newcommand{\ppoc}{\mbf{P}}		
\newcommand{\actions}{\Sigma_{\!_{act}}} 		
\newcommand{\pact}{\mbf{Act}} 		
\newcommand{\pobs}{\mbf{Obs}} 		
\newcommand{\pog}{\mbf{\Gamma}}
\newcommand{\observe}[1]{O\at{#1}}			
\newcommand{\decide}[1]{A\at{#1}}			
\newcommand{\current}[1]{S\at{#1}}		
\newcommand{\state}[1]{\#_{#1}} 		

\newcommand{\pocg}{\mbf{K}_{\sens}}		
\newcommand{\snap}[1]{\mathbf{#1}} 
\newcommand{\witness}[1]{w_{#1}} 		
\newcommand{\threshold}[1]{\tau_{#1}} 	
\newcommand{\ori}[1]{\omega(#1)} 		
\newcommand{\tri}[1]{\mu(#1)} 			
\newcommand{\poc}[1]{\mathtt{Poc}(#1)} 	
\newcommand{\snapfloor}[1]{\left\lfloor #1\right\rfloor} 	
\newcommand{\nullsnap}{\mathtt{Null}} 						
\newcommand{\clock}[1]{\mathtt{Clk}\left(#1\right)}			
\newcommand{\proj}[1]{\mathtt{proj}_{#1}}

\newcommand{\probabilistic}[1]{\hsm{P}_{_{#1}}} 	
\newcommand{\empirical}[1]{\hsm{E}_{_{#1}}} 	
\newcommand{\snapdir}[1]{\mathtt{Dir}(#1)} 		

\newcommand{\loc}[1]{\mathtt{loc}[#1]} 
\newcommand{\wait}{\mathtt{wait}} 
\newcommand{\fwd}{\mathtt{fwd}}	
\newcommand{\back}{\mathtt{bck}} 
\newcommand{\better}{\mathtt{better}}
\newcommand{\worse}{\mathtt{worse}}

\newcommand{\defstop}{\hfill $\square$}
\theoremstyle{plain}
\newtheorem{thm}{Theorem}[section]
\newtheorem{prop}[thm]{Proposition}
\newtheorem{lemma}[thm]{Lemma}
\newtheorem{cor}[thm]{Corollary}

\theoremstyle{definition}
\newtheorem{defn}[thm]{Definition}

\newtheorem{remark}[thm]{Remark}

\title{Universal Memory Architectures for Autonomous Machines}

\author{Dan~P.~Guralnik$^\dagger$,~%
Daniel~E.~Koditschek$^\ddagger$%
\thanks{$^{\dagger}$Post doctoral fellow, Electrical \& Systems Engineering, School of Engineering and Applied Sciences, University of Pennsylvania, , 200 S 33rd Str., Moore Bldg. \#203, Philadelphia, PA 19104, USA,~%
{\tt\small guraldan at seas.upenn.edu}}%
\thanks{$^{\ddagger}$Alfred Fitler Moore Professor, Electrical \& Systems Engineering, School of Engineering and Applied Sciences, University of Pennsylvania, 200 S 33rd Str., Moore Bldg. \#203, Philadelphia, PA 19104, USA,~%
{\tt\small kod at seas.upenn.edu}}%
}%

\begin{document}

\maketitle

\begin{abstract} We propose a self-organizing memory architecture for perceptual experience capable of supporting autonomous learning and goal-directed problem solving in the absence of any prior information about the agent's environment. The architecture is simple enough to ensure
(1) a quadratic bound (in the number of available sensors) on space requirements, and 
(2) a quadratic bound on the time-complexity of the update-execute cycle.  At the same time, it is sufficiently complex to provide the agent with an internal representation which is 
(3) minimal among all representations of its class which account for every sensory equivalence class subject to the agent's belief state; 
(4) capable, in principle, of recovering the homotopy type of the system's state space; 
(5) learnable with arbitrary precision through a random application of the available actions.

The provable properties of an effectively trained memory structure exploit a duality between {\it weak poc sets}~ ---~ a symbolic (discrete) representation of subset nesting relations~ ---~ and {\it non-positively curved cubical complexes}, whose rich convexity theory underlie the planning cycle of the proposed architecture.
\end{abstract}

\section{Introduction}

\subsection{Motivation} 
A major obstacle to autonomous systems synthesis is the absence of a capacious but efficient memory architecture. In humans, memory influences behaviour over a wide range of time scales, leading to the emergence of what seems to be a functional hierarchy of sub-systems \cite{Squire-memory_overview}: from non-declarative vs.  declarative through the split of declarative memory into semantic and episodic \cite{Tulving-episodic_memory}; and on to theories of attention and recall \cite{Baars-cognitive_theory_of_consciousness}. This variety of scales is mirrored in the collection of problems addressed by the synthetic sciences: from learning dependable actions/motion primitives \cite{Mugan_Kuipers-autonomous_learning_high_level_states,Thomas_Barto-motor_primitive_discovery}; through learning objects and their affordances \cite{Kraft_et_al-development_object_grasping_by_exploration,Ivaldi_et_al-object_learning_through_exploration} to demonstration-driven task execution \cite{Tellex-learning_with_grounding_graphs,Chernova_Thomaz-learning_from_human_teachers}; through exploring and
mapping an unknown environment \cite{Leonard_Durrant-Whyte-SLAM,Kuipers-history_of_SSH,Thrun-probabilistic_robotics,Milford_Wyeth_Prasser-RatSLAM} and motion planning \cite{Schwartz_Sharir-piano2,Koditschek_Rimon-artificial_potential_functions,Erdmann-IJRR'09-StrategySpace}; and on to general problem solving \cite{Newell_Simon-human_problem_solving} using artificial general intelligence (AGI) architectures \cite{Langley_Laird_Rogers-cognitive_architectures,Franklin_et_al-LIDA_overview,Nivel_Thorisson-AERA_overview}.

One idea stands out as common to all these approaches, beginning with the formal notion of a problem space introduced by Newell and Simon \cite{Newell_Simon-human_problem_solving,Newell-reasoning_about_problem_solving}: the purpose of a memory architecture is to learn the transition structure of the state space $\spc$ of the system comprised of the agent and its environment $\env$ while processing the history of observations into a format facilitating improved future control. 

It is often argued (e.g. \cite{Cobo_Isbell_Thomasz-object_focused_Qlearning,Stachowicz_Kruijff-episodic_like_memory,Milford_Wyeth-persistent_navigation_with_RatSLAM}) that memory architectures for general agents should enjoy a high degree of domain- and task-independence. In general, however, clear definitions of notions such as `domain' and `task' are not readily forthcoming across the vast breadth of literatures discussing memory, agents and autonomy. Notions of `universal learners' have been proposed \cite{Schmidhuber-Godel_machines_optimality} based on optimizing gain in estimators of predictive entropy, however there is evidence to suggest that the resulting level of generality may be insufficient for some tasks \cite{Martius_et_al-infotaxis_driven_self_organization}.

Absent broadly recognized formal foundations, we advance an architecture provably satisfying intuitive universality properties, including, most centrally: (1) interactions with the environment are encoded
in the most generic, yet minimal, manner possible, while requiring no prior semantic information; and (2) learning obtains from direct binary sensory input, automatically developing appropriate contextual links between sensations of arbitrary modality. A key outcome  is that the architecture encodes its observation history in a model space that supports the agent's problem solving as a form of reactive motion planning whereby atomic computations provably correspond to nearest point projection in the reachable set.

\subsection{Contribution} We consider a generic { \em discrete binary agent } (DBA): a machine sensing and interacting with its environment in discrete time, equipped with a finite collection $\sens$ of Boolean-valued sensors, some of which serve as triggers for actions/behaviors (switched on and off at will). 

Given an instance of a DBA interacting with an environment $\env$, it is natural to view the set $\Xi$ of sensory equivalence classes of the associated transition system $\spc$ as a subset of the power set $\{0,1\}^\sens$. It is generally accepted \cite{Donald_Jennings-planning_from_sens_equiv_classes,Tovar_et_al-Information_spaces_BITBOTS} that a memory architecture must be capable of supporting internal representations rich enough to account for the diversity \cite{Rivest_Schapire-diversity} of the transition system $\spc$: Exact  problem solving,  when construed as abstract motion planning, requires an internal representation capable, eventually, of accounting for all the classes in $\Xi$  and the transitions between them. Unfortunately, as expressed forcefully in \cite{Rivest_Schapire-diversity} and as we review at length below, the task of obtaining an exact description of $\Xi$  becomes intractable in the absence of strong simplifying assumptions about X, as the number of sensors grows.

To circumvent this obstacle, rather than imposing any specific structure on $\spc$, we propose to relax the requirement for precise reconstruction by introducing an approximation whose discrepancy from $\Xi$ we characterize exactly  and  show to be the smallest possible in its (computationally effective) class of objects.

The new memory and control architecture we propose here consists of two layers:
\begin{itemize}
	\item A data structure $\snap{S}$ -- called a {\it snapshot} -- keeping track of the current state and summarizing observations in terms of collection of real-valued registers, of size quadratic in the number of sensors, summarizing the history of observations made by the agents.
	\item A reactive planner, built on a {\it weak poc set structure} $\ppoc$ (\cite{Allerton_2012,Roller-duality} and defn.~\ref{defn:poc set}) constituting a record of pairwise implications among the atomic sensations as observed by the agent; $\ppoc$ is computed from $\snap{S}$ in each control cycle.
\end{itemize}
A crucial property of our architecture is that $\ppoc$ and $\model$ are formally reconstructible from each other. The model space $\model$ takes the form of a CAT(0) cubical complex, or {\it cubing}\footnote{For a good introduction CAT(0) cubical complexes, see \cite{Wise-riches_to_raags}. For a tutorial on cell complexes see \cite{Hatcher-alg_top_textbook}, chapter 0 and appendix.}, whose $0$-skeleton is contained in $\{0,1\}^\sens$. As the snapshot $\snap{S}$ is updated by incoming observations, the space $\model$, as encoded by $\ppoc$, is transformed along with it. We can state our main contributions -- albeit, necessarily, informally at this point -- in terms of provable properties of the architecture and its model spaces:
\begin{itemize}
	\item[{\bf (i)}] {\bf Universality of Representation. } $\model$ is the minimal model guaranteed to represent all the sensory equivalence classes of {\it any} sensorium $\sens$ satisfying the record $\ppoc$ (see \ref{realization}). In particular, in the absence of information not already encoded in $\ppoc$, it is impossible to distinguish the $0$-skeleton of $\model$ from the set of sensory equivalence classes, $\Xi$.
	\item[{\bf(ii)}] {\bf Topological Approximation. } As a topological space, $\model$ is always contractible\footnote{The formal notion of being `hole-free'~ ---~ see  \cite{Hatcher-alg_top_textbook}, chapter 0.}. Provided a sufficiently rich sensorium, the sub-complex $\punct\subset\model$ of faces all of whose vertices lie in $\Xi$ inherits from $\model$ the topology\footnote{Up to homotopy equivalence~ ---~ see definition in \cite{Hatcher-alg_top_textbook}, chapter 0.} of the observed space $\spc$ (see appendix \ref{realization}-\ref{homotopy type result}).
	\item[{\bf (iii)}]{\bf Low-complexity, effective learning. } The proposed architecture requires quadratic space (in the number of sensors) for storage, and no more than quadratic time for updating. Furthermore, an agent picking actions at random learns an approximation of the resulting walk's limiting distribution on $\spc$ (see \ref{empirical:performance}).
	\item[{\bf (iv)}]{\bf Efficiency of Planning. } Planning the next action given a target sensation takes quadratic time in the number of sensors, while eliminating the need for searching in the model space. With sufficient parallel processing power, this bound may be reduced to a constant multiple of the height~ ---~ the maximum length of a chain of implications of $\ppoc$ (see \ref{subsection:planning}).
\end{itemize}

\bigskip
To the best of our knowledge, this combination of provable properties has not previously appeared in the literature.

\subsection{Overview and Related Literature} 
To establish the novelty of our contribution we now briefly review the copious literature bearing on these topics as arising from three distinct traditions: robotics; connectionist computation; and artificial general intelligence.  After presenting our technical ideas we will explore at the end of the paper in a more discursive form their relation to and implications for the broader field.
\subsubsection{Relation to Mapping and Navigation} Formulating
navigation and mapping problems in terms of
a point agent moving through a homotopically trivial ambient
space  while avoiding a collection of
geometrically defined obstacle regions representing forbidden states is fundamental to motion planning \cite{Schwartz_Sharir-piano2,Koditschek_Rimon-artificial_potential_functions} and mapping \cite{Thrun-probabilistic_robotics,Kuipers_Remolina-topological_maps}. The ubiquity  of obstacles in these settings introduces topological considerations whose primacy is well established in  the algorithmic literature \cite{Kuipers_Remolina-topological_maps,Tomatis_et_al-hybrid_SLAM,Kuipers-topological_SLAM,Dellaert-probabilistic_topological_mapping,Choset-bayesian_topological_SLAM,Curto_Itskov-cell_groups,Ghrist-homological_SLAM}, governing the complexity of not only motion planning \cite{Farber-topological_complexity} but even set membership \cite{Yao-Decision_tree_complexity}. 

Our strategy is to reduce the general problem of memory storage and its use for motion planning in the underlying transition structure of a problem space (as sensed by a DBA) to the geometric problem of motion planning in the agent's model space $\model$ (playing the role traditionally assigned to Euclidean space). Generalizing the Euclidean setting, $\model$ has a very strong convexity theory \cite{Roller-duality, Chepoi-median} enabling low-cost greedy navigation.

The topological point of view has been shown to be well warranted in the discrete setting as well. As was demonstrated by Pratt \cite{Pratt-modeling_concurrency_with_geometry}, oriented topological structures (cubical complexes, in fact) may be used to encode the causal relations among actions and states in symbolic transition systems. Approaches generalizing Pratt's have since been used to formulate very general models of reconfigurability and self-assembly \cite{Ghrist_Peterson-reconfiguration,Ghrist_Klavins_Lipsky-self_assembly}.

\subsubsection{Mechanisms for Learning and Planning} Snapshots use an evolving estimate of pairwise intersections of sensor footprints to form a record of implications among the atomic sensations of the DBA. The necessity of such a record for planning goes back (at least) to \cite{McCarthy_Hayes-Philosophical_probs_of_AI}, yet ideas about applying it as a way to encode context are fairly recent and specialized \cite{Mugan_Kuipers-autonomous_learning_high_level_states,Stachowicz_Kruijff-episodic_like_memory,Milford_Jacobson-brain_inspired_sensor_fusion_for_SLAM}. Our internal representation takes the additional step of applying this principle to {\it all} the sensations available to a DBA, including the control signals it uses to interact with its environment. 

The resulting learning and control mechanisms may be realized in a highly simplified and idealized, yet highly plastic, network of neuron-like cells simulating the structure of $\ppoc$ \cite{Allerton_2012}. This analogy with neural networks is not a coincidence: estimating {\it arbitrary} intersections from near-synchronous activation of sensors in a planar sensor field has been explored as a means for topological \cite{Curto_Itskov-cell_groups} as well as metric mapping by competitive attractor networks (RatSLAM \cite{Milford_Wyeth_Prasser-RatSLAM,Milford_Wyeth-persistent_navigation_with_RatSLAM,Milford_Jacobson-brain_inspired_sensor_fusion_for_SLAM}), as the study of the structure of stability properties vis-a-vis topology and plasticity in more general networks is just taking off \cite{Hahnloser_Seung_Slotine-permitted_and_forbidden_sets,Curto_Degratu_Itskov-flexible_memory_networks}. 

\subsubsection{Model Spaces} The necessity to maintain high-dimensional representations of the state space $\spc$ poses a major challenge for current approaches to learning \cite{Barto_Sutton-reinforcement_learning,Barto_Mahadevan-recent_advances_hierarchical_RL,Cobo_Isbell_Thomasz-object_focused_Qlearning} and general problem solving \cite{Helgason-attention_for_AI_systems,Steunebrink_et_al-resource_bounded_machines}. The method closest to ours in its formalism seems to be that of \cite{Rivest_Schapire-diversity}~ ---~ and even lends itself to learning by a connectionist network \cite{Mozer_Bachrach-SLUG}~ ---~ but still requires an exponentially large representation for planning purposes. By contrast, in our case, the ability to translate action planning in $\model$ into what is essentially a flow problem in a network constructed from the underlying sensorium obviates the need for maintaining $\model$ in memory, allowing us to evade the curse of dimensionality. Nevertheless, we are still guaranteed a model space that is sufficiently rich to account for {\it all} sensory equivalence classes perceivable by the DBA \cite{Donald_Jennings-planning_from_sens_equiv_classes}. 

The computational advantages of our approach come at a cost that is largely driven by topology, as expressed in {\bf (ii)}: $\model$ necessarily has trivial topology\footnote{Again, in the sense of $\model$ being contractible.} \cite{Sageev-thesis,Roller-duality,Bridson_Haefliger-textbook}, and our own result \cite{Allerton_2012} establishes formal conditions on $\sens$ under which the complex $\punct$ reproduces the "topological shape"\footnote{In the sense of {\it homotopy type}~ ---~ \cite{Hatcher-alg_top_textbook}, chapter 0.} of $\spc$ (as discussed above), which may not be topologically trivial. The basic algorithm driving planning in our agents, however, achieves its efficiency by disregarding this mismatch. The introduction of auxiliary intrinsic motivation mechanisms \cite{Barto_Mirolli_Baldassarre-novelty_or_surprise,Cox_Krichmar-neuromodulation_as_controller} as a means of steering the agent away from obstacle states in $\model$ and towards desirable behaviours (not necessarily states!) seems to be a possible way out of this predicament, as well as towards a solution of the problem of closing the control loop. At this early stage, as a feasibility study for the overall approach, we only consider very simple excitation mechanisms causing the agent to choose actions with the desire to maximize {\it immediate} excitation gain, to the extent that may be sensed by $\sens$ (and otherwise to choose random actions).


\subsection{Organization of the Paper}
Having already given proofs of the formal results underlying {\bf (i)} and {\bf (ii)} in our previous paper \cite{Allerton_2012}, we defer the technical discussion of poc sets to appendix \ref{appendix:prelim}. This appendix is intended as an introductory overview of the theory of weak poc sets and the geometry of their dual spaces -- our agents' model spaces~ ---~ as well as a repository of proofs of technical results we require but could not find elsewhere in the literature. 

Section \ref{section:snapshots} discusses {\bf (iii)}. We formally state the observation model for DBAs, describe snapshots and their learning mechanisms, and present our early numerical work illustrating the practical implications of the claims regarding learning. 

Section \ref{section:control} is dedicated to item {\bf (iv)} in the list of contributions. Actions are introduced to the observation model, and control algorithms are defined and validated.


Finally, following an extended discussion of our results in relation to existing literature in section \ref{section:discussion} and the aforementioned appendix dealing with poc sets, a second appendix presents the proofs of technical results about snapshots.

\begin{table}[!t]\caption{Table of Mathematical Symbols}
\label{symbols table}
\begin{center}
	\renewcommand{\arraystretch}{1.6}
	\begin{tabular}[c]{c | p{.55\columnwidth} | c}
	\hline
		&	Topic/Notation	&	Ref.\\
	\hline\hline
		&	{\bf DBA Model (general)}	&	\\
	\hline\hline
	$\env$ & Environment (with points $p,q,\ldots$) & Sec.\ref{env and spc}\\
	$\spc$ & State space of the experiment (with points $x,y,\ldots$) & Sec.\ref{env and spc}\\
	$\pos$ & The position map $\spc\to\env$ & Sec.\ref{env and spc}\\
	\hline
	$\mytime$ & Time, the set of integers & Sec.\ref{transitions}\\
	$\at{t}$	& Reads as: "at time $t$" & Eqn.\eqref{traj} \\
	\hline
			&	{\bf DBA model (sensing)}	& \\
	\hline
	$\sens$ & Sensorium (elements are $a,b,c,\ldots$), with involution $a\mapsto a^\ast$ & Eqn.\eqref{eqn:involution}\\
	$\rho$ & Realization map of the sensorium $\sens$ & Eqn.\eqref{eqn:realization1} \\
	$\ev{a}{x}$	&	Evaluation, e.g. of $a\in\sens$ on $x\in\spc$ & Eqns.(\ref{eqn:evaluations1}-\ref{eqn:evaluations2})\\
	\hline
		&	{\bf DBA computational model (at time $t$)} & \\
	\hline\hline
	$\snap{S}\at{t}$ & Agent's snapshot & Sec.\ref{planning problem statement}\\
	$\pog\at{t}$ & The derived poc graph, $\snapdir{\snap{S}\at{t}}$ & Sec.\ref{planning problem statement}\\
	$\ppoc\at{t}$ & Derived (weak) poc set structure on $\sens$, $\poc{\snap{S}\at{t}}$ & Sec.\ref{poc sets for the first time}\\
	$\model\at{t}$ & The model space $\cube{\ppoc\at{t}}$ & Sec.\ref{intro to model spaces}\\
	$\punct\at{t}$ & The punctured model space $\punc{\ppoc\at{t}}$ & Def.\ref{defn:punctured model spc}\\
	$\observe{t}$ & Raw observation & Sec.\ref{current observation}\\
	$\current{t}$ & Recorded observation & Sec.\ref{current state}\\
	$\decide{t}$ & Decision (action) following the observation & \\
	\hline
		&	{\bf Contents/parameters of a snapshot $\snap{S}$} & \\
	\hline\hline
	$\pocg$	&	The complete graph on $\sens$ with all $aa^\ast$ edges removed & Def.\ref{defn:pocg} \\
	$\state{}\snap{S}$ & State of the snapshot & Def.\ref{defn:snapshot}(a)\\
	$\witness{ab}$	&	Weight on the edge $ab$ & Def.\ref{defn:snapshot}(b)\\
	$\threshold{ab}$	&	Learning threshold for the pair $a,b\in\sens$ & Def.\ref{defn:snapshot}(c)\\
	$\ori{ab}$	&	Orientation cocycle of $\snap{S}$ & Prop.\ref{prop:acyclicity lemma}\\
	$\tri{ab}$ 	&	Dissimilarity measure of $\snap{S}$ & App.\ref{proofs:adding equivalences}\\

	\hline
		&	{\bf Objects derived from a snapshot $\snap{S}$} & \\
	\hline\hline
	$\snapdir{\snap{S}}$	&	Derived poc graph & Prop.\ref{poc sets from snapshots} \\
	$\poc{\snap{S}}$	& 	Derived weak poc set structure & Def.\ref{defn:derived poc set}\\
	\hline
		&	{\bf Weak poc sets and their duals} & \\
	\hline\hline
	$P,Q,\ldots$	&	Poc sets (with and without indices) & Def.\ref{defn:poc set} \\
	$P^\circ$		&	The set dual of $P$, the $0$-skeleton of $\cube{P}$ & Def.\ref{defn:dual}(b) \\
	$\dual{P}$		&	Dual graph of $P$, the $1$-skeleton of $\cube{P}$ & Def.\ref{defn:dual}(c) \\
	$\cube{P}$		&	Dual cubing of the poc set $P$ & Def.\ref{defn:dual}(a) \\
	$\punc{P}$		&	The punctured dual $\cube{P,\rho}$ with respect to a realization $\rho$ & Def.\ref{defn:punctured dual}\\
	$f^\circ$	&	The dual map $f^\circ:Q^\circ\to P^\circ$ of a poc morphism $f:P\to Q$ & Defs.\ref{defn:poc morphism},\ref{defn:dual map} \\ 
	\hline
	\end{tabular}
\end{center}
\end{table}

\section{Snapshots: From Observation Sequences to a Memory Structure}\label{section:snapshots}
We begin with a formal statement of what we mean by a DBA and its observation model. We then proceed to construct snapshots, their updating mechanisms and the derived weak poc set structures, and conclude the section with results on the learning capabilities of snapshots. Table \ref{symbols table} reviews notation that will persist throughout the paper.

\subsection{Observation Model for DBAs}\label{subsection:DBA observation model} 
\subsubsection{Environment and State}\label{env and spc} We place an agent in an environment $\env$. The state space of the system will be denoted by $\spc$, where we assume there is a map $\pos:\spc\to\env$ producing the location $\pos(x)$ of the agent in $\env$, given the state $x\in\spc$ of the system as a whole. As it turns out, no further mathematical structure on $\spc$ is required for the results that follow, hence, with a mind toward inviting the broadest range of applications, we impose none, much in the spirit of McCarthy and HayesÕ discussion of situation calculus \cite{McCarthy_Hayes-Philosophical_probs_of_AI}. 

\subsubsection{Time and Transitions}\label{transitions} We model time $\mytime$ as the set of integers (the subjective time of the agent), with $t=0$ corresponding to the initial time. The basic objects of study are then {\it trajectories}, or maps of the form 
\begin{equation}\label{traj}
	\varphi:\mytime\to\spc\,,\qquad
	\varphi=(\varphi\at{t})_{t\in\mytime}
\end{equation}
We define abstract transitions in $\spc$ as follows:
\begin{defn}[$n$-transitions] An element of the $(n+1)$-fold Cartesian power $\spc^{n+1}$ will be referred to as an $n$-transition. For any trajectory $\varphi:\mytime\to\spc$ and $n\geq 0$ we define the map
\begin{equation}
	\dd^n\varphi:\left\{\begin{array}{ccc}
		\mytime	&\to&	\spc^{n+1}\\
		t	&\mapsto&	\varphi\at{t-n}\times\cdots\times\varphi\at{t}
	\end{array}\right.
\end{equation}
We refer to $0$-transitions as {\it states}, and to $1$-transitions simply as {\it transitions}.\defstop
\end{defn}

Any setting where $\env$, $\spc$ and the transition structure of the system are specified (though, possibly, in an implicit fashion), implies constraints on the set of achievable trajectories. We will refer to such settings as {\it experiments}, within the framework of which each allowed trajectory will be referred to as a {\it run of the experiment}, while observations produced by sensors during a run (below) will be called {\it experiences}. 

\subsubsection{Discrete Binary Agents}\label{DBAs} A {\it discrete binary agent} (DBA) is endowed with a collection of binary sensors indexed by a finite set $\sens$. We will assume that each $a\in\sens$ is assigned an {\it order} $n_a\geq 1$, and a realization $\rho(a)\subseteq\spc^{n_a+1}$. We then say that $a$ is a {\it $n_a$-sensor}, or a {\it sensor of degree $n_a$}. For example, a $0$-sensor -- or {\it state sensor} -- responds to the system entering a certain subset of $\spc$ (a "macro-state"), while a $1$-sensor responds to the system experiencing a transition of a particular kind.

Evaluation of sensors is best viewed in the context of trajectories: a $n$-sensor $a\in\sens$ is applied to a trajectory $\varphi$ and assigned a value at time $t\in\mytime$ according to the rule
\begin{equation}\label{eqn:evaluation of state sensors}
	\ev{a}{\varphi}\at{t}=1\DEFF\dd^n\varphi\at{t}\in \rho(a)
\end{equation}
Here $\ev{a}{\varphi}\at{t}$ denotes the measurement provided by the sensor $a$ at time $t$ given the trajectory $\varphi$. To avoid a profusion of parentheses and subscripts we will generally use bracket notation to denote the evaluation of Boolean- and scalar-valued functions:
\begin{eqnarray}
	\ev{a}{x}&:=&\indicator{\rho(a)}(x)\text{ whenever } a\in\sens\label{eqn:evaluations1}\\
	\ev{g}{s}&:=&g(s)\text{ whenever }g:S\to[0,1]\label{eqn:evaluations2}
\end{eqnarray}
and so forth. The symbol $\indicator{A}$ will always denote the indicator function of a set $A$ with respect to the appropriate super-set.

\medskip
We will assume that $\sens$ comes endowed with a map $a\mapsto a^\ast$ satisfying the following for all $a\in\sens$:
\begin{equation}\label{eqn:involution}
	a^{\ast\ast}=a\,,\quad a^\ast\neq a\,,
\end{equation}
as well as
\begin{equation}\label{eqn:realization1}
	\quad \rho(a^\ast)=\rho(a)^c:=\spc^{n_a+1}\minus\rho(a)
\end{equation}
We also introduce the virtual sensors $\minP,\maxP\in P$ evaluating to
\begin{equation}\label{trivial sensors}
	\ev{\minP}{\varphi}\at{t}\equiv 0\,,\quad
	\ev{\maxP}{\varphi}\at{t}\equiv 1\,
\end{equation}
on any trajectory $\varphi$ and at any time $t\in\mytime$. For subsets $A\subseteq\sens$ we will always use the notation $A^\ast$ to denote the set of all $a^\ast$, $a$ ranging over $A$.

The database structure we will be using is designed to maintain an approximate record of the relations among sensors in $\sens$ believed by the agent to hold true throughout time. This record at time $t\in\mytime$ is encoded in a weak poc set structure $\ppoc\at{t}$ over $\sens$ (definition \ref{defn:poc set}).

For two $n$-sensors $a,b\in\sens$ this requirement translates into $a\leq b$ in $\ppoc\at{t}$ being treated (for planning purposes) as the inclusion $\rho(a)\subseteq \rho(b)$ in the space $\spc^{n+1}$. Note how the equivalent containment $\rho(b^\ast)\subseteq\rho(a^\ast)$ is encoded by the contra-positive implication $b^\ast\leq a^\ast$, which, by the definition of a weak poc set, holds if and only if $a\leq b$ does.

When $a,b$ have different orders we are forced to replace this requirement by a weaker one: at any time $t$, our agents will interpret $a\leq b$ in $\ppoc\at{t}$ as 
\begin{equation}\label{eqn:implication via trajectory}
	\ev{a}{\varphi}\at{t'}\leq\ev{b}{\varphi}\at{t'}
\end{equation}
holding for all $t'\in\mytime$. In other words, our agents assume that relations among sensors do not change over time\footnote{This is not to say that our agents are not allowed to change their minds regarding which relations hold true and which do not: the purpose of keeping a dynamic record of relations is to eventually uncover the `correct' relations.}.

For example, if $a,b\in\sens$ where $a$ is a state sensor and $b$ is a transition sensor, consider the statements:
\begin{equation*}
	(\dagger)\;\ev{a}{\varphi}\at{t}\leq\ev{b}{\varphi}\at{t}\,,\quad
	(\ddagger)\;\ev{b}{\varphi}\at{t}\leq\ev{a}{\varphi}\at{t}
\end{equation*}
treated as identities over both $\varphi$ and $t$. We see that $(\ddagger)$ states all transitions of type $b$ must terminate in a state of type $a$, while $(\dagger)$ means that \emph{only} transitions of type $b$ could produce state $a$. It is clear that both kinds of statement are essential for planning.

\subsection{The Model Spaces}\label{subsection:intro to model spaces and coherence} 
\subsubsection{A Record of Implications}\label{poc sets for the first time} Informally, a ``record of implications in $\sens$'' is a partial ordering on $\sens$ reflecting the standard interactions between Boolean complementation and Boolean implication. Formally, our DBA will maintain, at any time $t\geq 0$, a {\it weak poc set structure} $\ppoc\at{t}$ on $\sens$ consisting of a partial order relation $\leq$ satisfying, for all $a,b\in\sens$:
\begin{enumerate}
	\item $\minP\leq a$;
	\item $a\leq b\THEN b^\ast\leq a^\ast$.
\end{enumerate}
Note that $a^\ast\neq a$ by the construction of $\sens$ and compare with the definition in appendix \ref{defn:poc set}.

\subsubsection{Observations as vertices of a cube}\label{current observation} From the agent's viewpoint, the current state of the experiment at time $t$ is completely characterized by the measurements $\ev{a}{\varphi}\at{t}$, where $\varphi$ is the agent's trajectory. Equivalently, the state may be encoded in a subset $\observe{t}\subset\sens$ satisfying $\card{\observe{t}\cap\{a,a^\ast\}}=1$ for all $a\in\sens$. Such subsets of $\sens$ are called {\it complete $\ast$-selections}. An incomplete measurement of the state would then correspond to a subset $O\subset\sens$ satisfying $\card{O\cap\{a,a^\ast\}}\leq 1$ for all $a\in\sens$, called an (incomplete) $\ast$-selection -- see definition in appendix\ref{defn:selection}, along with remarks on the notation to follow.

Thus, one thinks of the collection $S(\sens)^0$ of all complete $\ast$-selections on $\sens$ as enumerating the possible sensory equivalence classes in the sensed space. However, some of the elements in this collection are redundant given the record $\ppoc\at{t}$: an implication $a\leq b$ means that no $O\in S(\sens)^0$ containing $\{a,b^\ast\}$ is expected by the agent to be witnessed by any observation (see fig.\ref{fig:duality for a simple relation}). Formally, a set $O\subset\sens$ is {\it coherent} (definition \ref{defn:coherence}), if no pair of elements $a,b\in O$ satisfies $a\leq b^\ast$.

\subsubsection{The model spaces}\label{intro to model spaces} The model space $\model\at{t}$ corresponding to the record $\ppoc\at{t}$ takes the form of a cubical complex~ ---~ a topological space constructed from a collection of vertices (the $0$-skeleton), a set of edges (the $1$-skeleton), and successively higher dimensional connecting cells in the form of cubes \cite{Wise-riches_to_raags}. We choose the vertex set of $\model\at{t}$ to coincide with the set of coherent $\ast$-selections in $S(\sens)^0$. Edges are inserted to join any pair of vertices $A,B$ satisfying $|A\minus B|=1$ (this condition turns out to be symmetric). The hop-distance on the resulting graph may be seen as a variant of the crude, `information motivated', Hamming distance on $\{0,1\}^\sens$. The $1$-dimensional skeleton of $\model\at{t}$ is further enriched with higher dimensional cubes to yield the space $\cube{\ppoc\at{t}}$, as described in appendix \ref{appendix:prelim} (definition \ref{defn:dual}) for the interested reader. While a fairly detailed knowledge of the geometry and topology of spaces obtained in this way is essential for following our formal arguments regarding the modeling capabilities of this class, much of it is unnecessary for this section's account of how the agent obtains its representation of $\model\at{t}$, the record $\ppoc\at{t}$.

\subsubsection{Maintaining a record of the current state}\label{current state} Returning to the problem of representing the current state, observe that $\ppoc\at{t}$ is expected to change as time progresses, possibly giving rise to observations $\observe{t}$ that are incoherent with respect to $\ppoc\at{t}$, and therefore represent points `outside' the model space. While the raw observation $\observe{t}$ {\it must} be applied to the agent's data structure in hopes of improving $\ppoc\at{t}$, the agent must resolve the contradiction within the framework of its current model, replacing the incoherent complete observation $\observe{t}$ in its role as the record of the current state kept by the agent with a coherent but incomplete observation \eqref{eqn:coherent projection}, $\current{t}:=\coh{\observe{t}}$, satisfying certain naturality requirements~ ---~ see appendix \ref{coherent projection} for the complete technical discussion.

This means the agent resolves the contradiction at the price of introducing ambiguity into its record of the current state: instead of having a single vertex of $\model\at{t}$ representing the current state (``complete knowledge''), any vertex containing the set $\current{t}$ may turn out to be the correct current state.

The complexity of coherent projection (lemma \ref{lemma:propagation gives coherent projection}) and its role in the agent's reasoning processes, its interplay with the convexity theory of the model space $\model\at{t}$ and its interpretation as the basis for viewing our architecture as a connectionist model (albeit a very limited one) of cognition will all be discussed in section \ref{subsection:planning}.

\subsection{Snapshots}\label{subsection:snapshot preliminaries}
In \cite{Allerton_2012} we have introduced the rather loose notion of a {\it snapshot}, aiming to outline a class of database structures for dynamically maintaining weak poc-set structures from a sequence of observations made by an agent along a trajectory $\varphi$ through $\spc$. A rigorous treatment of this tool requires some careful definitions.
\begin{defn}\label{defn:pocg} Denote by $\pocg$ the graph obtained from the complete graph over the vertex set $\sens$ by removing all edges of the form $aa^\ast$, $a\in\sens$. Edges of $\pocg$ will be referred to as {\it proper pairs} in $\sens$. We will abuse notation and write $ab\in\pocg$ for the edge $\{a,b\}$ of $\pocg$.\defstop
\end{defn}
The graph $\pocg$ is the scaffolding for snapshots:
\begin{defn}[Snapshot]\label{defn:snapshot} A snapshot $\snap{S}$ over $\sens$ consists of the following:
\begin{itemize}
	\item[(a)] {\bf State. } Each vertex $a\in\sens$ of $\pocg$ is assigned a binary state $\state{a}\snap{S}\in\{0,1\}$. The set 
	\begin{equation}
		\state{}\snap{S}=\set{a\in\sens}{\state{a}=1}
	\end{equation}
	is called the \emph{state vector of the snapshot $\snap{S}$} and is required to be a $\ast$-selection on $\sens$ (definition \ref{defn:selection}).
	\item[(b)] {\bf Edge weights. } Each edge $ab\in\pocg$ is assigned a non-negative real number denoted $\witness{ab}=\witness{ab}(\snap{S})$.
	\item[(c)] {\bf Learning Thresholds.} Each edge $ab\in\pocg$ carries a non-negative real number $\threshold{ab}=\threshold{ab}(\snap{S})$ satisfying $$\threshold{ab}=\threshold{a^\ast b}=\threshold{ab^\ast}=\threshold{a^\ast b^\ast}\leq\tfrac{1}{4}\,.$$
\end{itemize}

For every $ab\in\pocg$, the restriction of $\snap{S}$ to the subgraph induced by the vertices $a,a^\ast,b$ and $b^\ast$ will be denoted by $\snap{S}\res{ab}$ and referred to as a {\it square} in $\snap{S}$.\defstop
\end{defn}

The original motivation of \cite{Allerton_2012} for the notion of a snapshot is twofold:
\begin{enumerate}
	\item {\bf Maintaining a consistent representation of the current state. } For this purpose we will generally assign the coherent projection of the current state measurement to be stored in $\state{}\snap{S}$. 
	\item {\bf Learning implications in the sensorium. } To learn an estimate of the implication order on $\sens$ inherited from its realization in $\spc$ it should suffice to maintain a system of weights $\witness{ab}^t$ on $\snap{S}\at{t}$ quantifying the relevance (e.g. frequency) of the event $a\wedge b$, allowing one to partially orient the snapshot according to the rule of thumb illustrated in figure \ref{fig:relational square}. 
\end{enumerate}

\begin{figure}[t]
	\begin{center}
		\includegraphics[width=.7\columnwidth]{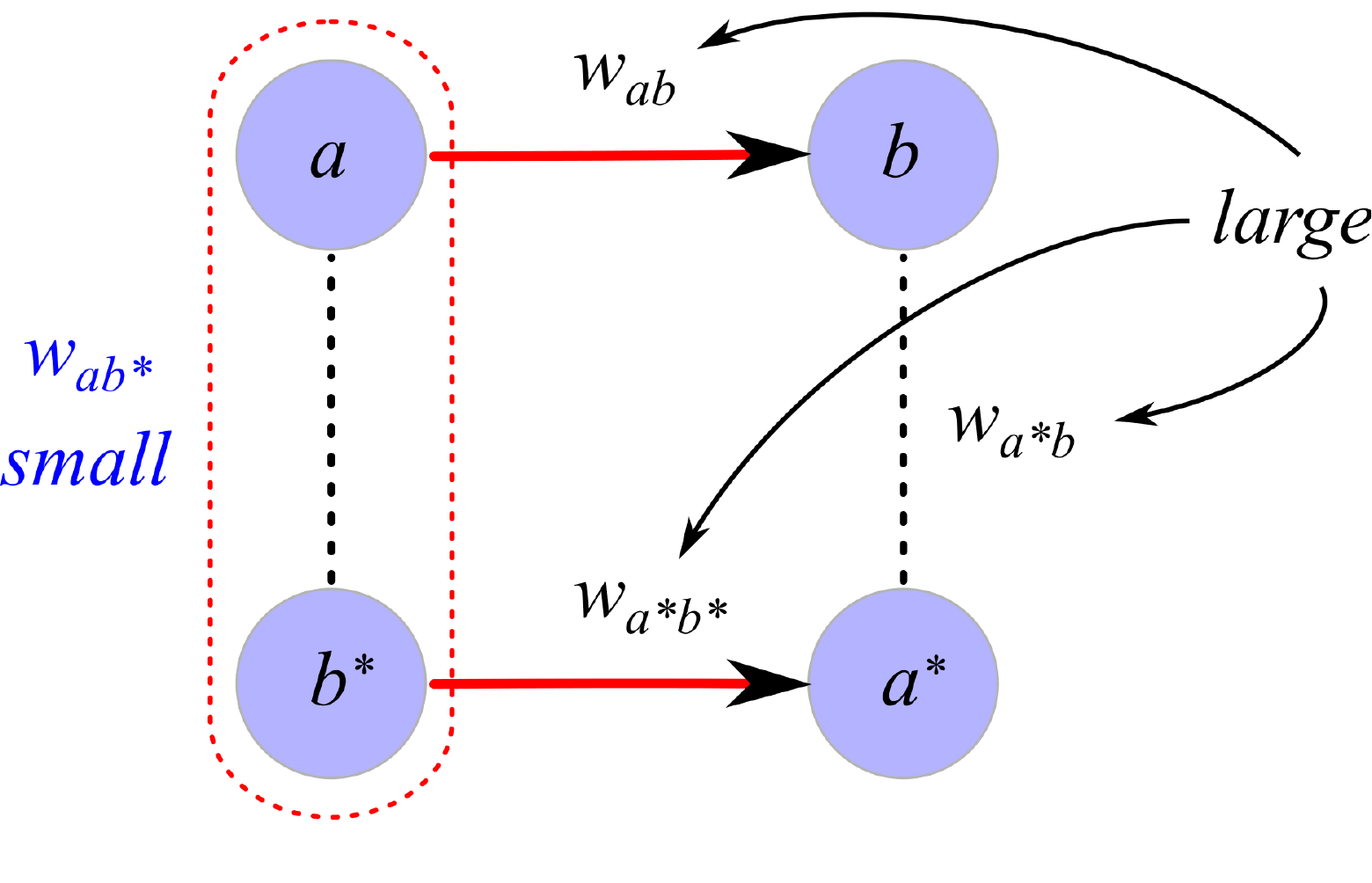}
		\caption{determining edge orientations in a snapshot by restricting attention to $\snap{S}\res{ab}$. \label{fig:relational square}}
	\end{center}
\end{figure}

The graphical representation derived from a snapshot in this manner does not automatically define a weak poc set, but is nearly there:
\begin{defn}[poc graph] A \emph{poc graph} $\pog$ over $\sens$ is a subgraph of $\pocg$ endowed with an orientation which satisfies, for every $ab\in\pocg$:
\begin{itemize}
	\item[-] If $ab\in\pog$ then $b^\ast a^\ast\in\pog$;
	\item[-] If $ab\in\pog$ then $a^\ast b,ba^\ast,b^\ast a,ab^\ast\notin\pog$.
\end{itemize}
By abuse of notation, we use the symbol $ab$ to mean the directed edge emanating from $a$ and pointing to $b$ (if it exists in $\pog$).\defstop
\end{defn}
In order for a poc graph to represent a weak poc set structure on $\sens$ one needs:
\begin{lemma}[derived poc set]\label{lemma:derived poc set} The transitive closure of the orientation relation on a poc graph $\pog$ over $\sens$ is a weak poc set structure on $\sens$ if and only if $\pog$ has no directed cycles.
\end{lemma}
\begin{proof} This follows directly from the discussion in example \ref{example:generators and relations}.
\end{proof}

The rest of this section mainly deals with characterizing a large class of snapshots encoding acyclic directed poc graphs and with means of evolving snapshot representations of $\spc$ from trajectories. Given a trajectory $\varphi$ of our agent through $\spc$, the collection of {\it coincidence indicators}
\begin{equation}\label{eqn:coincidence indicators}
	c_{ab}^t:=\ev{a}{\varphi}\at{t}\cdot\ev{b}{\varphi}\at{t}
\end{equation}
may be used to evolve a sequence of snapshots $\snap{S}\at{t}$ representing, at any time $t>0$, the cumulative influence of the agent's observations on its perception of implications in the sensorium.

\subsection{Probabilistic Snapshots and Acyclicity}\label{subsection:probabilistic snapshots} The following set-theoretic identities among the coincidence indicators are easily verified for all $a,b,c\in\sens$:
\begin{equation}\label{indicator identities}
	\begin{array}{l}
	c_{aa^\ast}^t=0\\
	c_{ab}^t=c_{ba}^t\\
	c_{aa}^t=c_{ab}^t+c_{ab^\ast}^t	\\
	c_{ab}^t+c_{a^\ast b}^t+c_{a^\ast b^\ast}^t+c_{ab^\ast}^t=1\\
	c_{ab^\ast}^t+c_{bc^\ast}^t+c_{ca^\ast}^t=c_{a^\ast b}^t+c_{b^\ast c}^t+c_{c^\ast a}^t
	\end{array}
\end{equation}
These identities motivate considering snapshots with weights obeying analogous constraints:
\begin{defn}[Probabilistic Snapshot]\label{defn:probabilistic snapshot} We say that a snapshot $\snap{S}$ is {\it probabilistic}, if $\state{}\snap{S}$ is a coherent $\ast$-selection and the edge weights satisfy the following:
\begin{itemize}
	\item{\bf Consistency constraint. } if $ab,ac\in\pocg$ then:
	\begin{equation}\label{eqn:consistency identity}
		\witness{ab}+\witness{ab^\ast}=\witness{ac}+\witness{ac^\ast}
	\end{equation}
	\item{\bf Normalization constraint. } for any $ab\in\pocg$:
	\begin{equation}\label{eqn:normalization identity}
		\witness{ab}+\witness{a^\ast b}+\witness{a^\ast b^\ast}+\witness{ab^\ast}=1
	\end{equation}
	\item{\bf Orientation constraint. } if $ab,\,bc,\,ac\in\pocg$ then:
	\begin{equation}\label{eqn:cocycle identity}
		\witness{a^\ast b}+\witness{b^\ast c}+\witness{c^\ast a}=				
		\witness{ab^\ast}+\witness{bc^\ast}+\witness{ca^\ast}
	\end{equation}	
\end{itemize}
We denote the set of all probabilistic snapshots over $\sens$ by $\probabilistic{\sens}$, or simply $\probabilistic{}$ when there is no danger of confusion.\defstop
\end{defn}
A fundamental observation regarding probabilistic snapshots is the following
\begin{prop}[Acyclicity Lemma]\label{prop:acyclicity lemma} Suppose $\snap{S}$ is a probabilistic snapshot over $\sens$ and $\pog$ is a poc graph satisfying the {\emph orientation cocycle condition}:
\begin{equation}\label{eqn:orientation cocycle}
	ab\in\pog\THEN\ori{ab}:=\witness{a^\ast b}-\witness{ab^\ast}>0
\end{equation}
Then $\pog$ contains no directed cycles.
\end{prop}
\begin{proof} See appendix \ref{proofs:acyclicity lemma}.
\end{proof}
This proposition puts the vague notion from figure \ref{fig:relational square} on how to derive implications from a snapshot on a firm footing:
\begin{prop}[Poc graphs from snapshots]\label{poc sets from snapshots} Let $\snap{S}$ be a probabilistic snapshot. Construct a poc graph $\snapdir{\snap{S}}{}$ by setting 
\begin{equation}\label{eqn:virtual implication}
	ab\in\snapdir{\snap{S}}{}\IFF\witness{ab^\ast}<\min\left\{\begin{array}{c}
		\threshold{ab},\,\\
		\witness{ab},\,
		\witness{a^\ast b},\,
		\witness{a^\ast b^\ast}
	\end{array}\right\}
\end{equation}
Then $\snapdir{\snap{S}}{}$ is an acyclic poc graph.\defstop
\end{prop}
\begin{proof} The symmetries of $\threshold{\wild}$ and $\witness{\wild}$ immediately imply $ab\in\snapdir{\snap{S}}{}$ iff $b^\ast a^\ast\in\snapdir{S}{}$. The strict inequality in \eqref{eqn:virtual implication} implies the second condition of a poc graph holds as well. Since the orientation cocycle is positive on every edge of $\snapdir{\snap{S}}{}$ by definition, the acyclicity lemma applies.
\end{proof}

The element of thresholding present in \eqref{eqn:virtual implication} may also be used as a part of the updating procedure of a probabilistic snapshot, without affecting the derived poc set:
\begin{prop}[Snapshot Truncation]\label{snapshot truncation} Let $\snap{S}$ be a probabilistic snapshot. Define a new snapshot $\snapfloor{\snap{S}}$ to have the same state as $\snap{S}$ while for every $ab\in\pocg$ satisfying \eqref{eqn:virtual implication} the weights are updated as follows:
\begin{equation}
	\left\{\begin{array}{rcl}
		\witness{ab^\ast}&\mapsto& 0\\
		\witness{ab}&\mapsto&\witness{ab}+\witness{ab^\ast}\\
		\witness{a^\ast b^\ast}&\mapsto&\witness{a^\ast b^\ast}+\witness{ab^\ast}\\
		\witness{a^\ast b}&\mapsto&\witness{a^\ast b}-\witness{ab^\ast}
	\end{array}\right.
\end{equation}
Then $\snapdir{\snap{S}}=\snapdir{\snapfloor{\snap{S}}}{}$.
\end{prop}
\begin{proof} The proof amounts to a direct verification that the $\snapfloor{\snap{S}}$ is probabilistic, and that $ab\in\pocg$ satisfies \eqref{eqn:virtual implication} in $\snapfloor{\snap{S}}$ if and only if the same condition is satisfied by $ab$ in $\snap{S}$.
\end{proof}

Following lemma \ref{lemma:derived poc set} we may now safely define:
\begin{defn}[Derived Poc Set]\label{defn:derived poc set} Let $\snap{S}$ be a probabilistic snapshot. Denote by $\poc{\snap{S}}$ the weak poc set structure obtained by setting $a\leq b$ iff there exists a directed path in $\snapdir{\snap{S}}$ from $a$ to $b$.\defstop
\end{defn}

We now proceed to introduce and study two possible snapshot constructions.

\subsection{Empirical Snapshots}\label{subsection:empirical snapshot} The empirical snapshot structure maintains an empirical approximation of the relative frequencies of observations of the form $a\wedge b$, $a,b\in\sens$. For any trajectory $\varphi$ of the agent through $\spc$ we could try setting
\begin{equation}\label{weights in empirical snapshot}
	\witness{ab}^t=\sum_{k=1}^t c_{ab}^k\,,
\end{equation}
with $\snap{S}\at{t}$ a trivial snapshot for all $t\leq 0$.

\begin{defn} A snapshot $\snap{S}$ with $\witness{\wild}=0$ is said to be {\it trivial} and denoted $\nullsnap$.\defstop
\end{defn}
Properties \eqref{indicator identities} then imply that $\snap{S}\at{t}$ satisfies the consistency and cocycle constraints (defn. \ref{defn:probabilistic snapshot}) for all $t>0$, and would satisfy the normalization constraint if we replace the weights $\witness{ab}^t$ by $\frac{1}{t}\witness{ab}^t$ throughout. 

\medskip
\subsubsection{Construction and Properties}\label{empirical:formalism} The formal construction is as follows:
\begin{defn}[Empirical Snapshot]\label{defn:empirical snapshot} A snapshot $\snap{S}$ over $\sens$ is an \emph{empirical snapshot} if the following conditions are satisfied:
\begin{itemize}
	\item For all $ab\in\pocg$, $\witness{ab}\in\ZZ_{_{\geq 0}}$; 
	\item For all $ab\in\pocg$, the expression
	\begin{equation}\label{eqn:snapshot consistency}
		\witness{a}:=\witness{ab}+\witness{ab^\ast}
	\end{equation}
	is independent of the choice of $b$, and vanishes only if $\state{a}=0$;
	\item The following expression does not depend on $a\in\sens$:
	\begin{equation}\label{eqn:snapshot synchrony}
		\clock{\snap{S}}:=\witness{a}+\witness{a^\ast}
	\end{equation}
\end{itemize}
Denote the set of empirical snapshots over $\sens$ by $\empirical{\sens}$ (or just $\empirical{}$ when justified).\defstop
\end{defn}
The evolution of an empirical snapshot under a sequence of observations is then defined through:
\begin{defn}[Empirical Update] Let $\snap{S}$ be an empirical snapshot and let $O\subset\sens$ be complete $\ast$-selection. The snapshot $O\ast\snap{S}$ is the empirical snapshot obtained from $\snap{S}$ by setting
\begin{equation}
	\witness{ab}(O\ast\snap{S}):=\witness{ab}(\snap{S})+\ev{\indicator{O}}{a}\cdot\ev{\indicator{O}}{b}
\end{equation}
for all $ab\in\pocg$. The state of $O\ast\snap{S}$ is set to $\coh{O}$ (where, recall, this is the the coherent projection \eqref{eqn:coherent projection} computed with respect to the weak poc set structure
derived from the new weights).\defstop
\end{defn}
\begin{defn}[Evolution] We say that a snapshot $\snap{T}$ over $\sens$ is an evolution of a snapshot $\snap{S}$, either if $\snap{S}=\snap{T}$ or if there is a sequence $\left(O_k\right)_{k=1}^m$ of complete $\ast$-selections in $\sens$ such that $\snap{T}=O_k\ast\cdots\ast O_1\ast\snap{S}$.\defstop
\end{defn}
Empirical snapshots are characterized by their ancestry:
\begin{lemma}\label{lemma:empirical evolves from trivial} A non-trivial snapshot $\snap{S}$ over $\sens$ is empirical if and only if it is an evolution of the trivial snapshot.
\end{lemma}
\begin{proof} See appendix \ref{proofs:empirical evolves from trivial}.
\end{proof}
Having characterized empirical snapshots as evolutions of the trivial one, we return to the observation that the weight $\witness{\wild}(\snap{S})/\clock{\snap{S}}$ on $\pocg$~ ---~ see \eqref{eqn:snapshot synchrony}~ ---~ defines a probabilistic snapshot. We may thus define $\snapdir{\snap{S}}$ accordingly, by setting
\begin{equation}\label{eqn:empirical virtual implication}
	ab\in\snapdir{\snap{S}}{}\IFF\witness{ab^\ast}<\min\left\{\begin{array}{c}
		\threshold{ab}\cdot\clock{\snap{S}},\\
		\witness{ab},\,
		\witness{a^\ast b},\,
		\witness{a^\ast b^\ast}
	\end{array}\right\}
\end{equation}
and conclude that:
\begin{prop}[empirical implies acyclic] If $\snap{S}$ is an empirical snapshot, then $\snapdir{\snap{S}}$ as defined in \eqref{eqn:empirical virtual implication} is an acyclic poc graph, and $\poc{\snap{S}}$ as defined in defn. \ref{defn:derived poc set} is a weak poc set structure on $\sens$.\defstop
\end{prop}
We will henceforth refer to DBAs endowed with empirical snapshots and utilizing the empirical update as {\it empirical agents}.

\medskip
\subsubsection{Performance of Empirical Agents}\label{empirical:performance}
In this paper we restrict attention to agents endowed with a fixed finite set of actions. An agent starting out at time $t=0$ with a trivial snapshot $\snap{S}\at{0}$ has no knowledge of its environment, and is therefore assumed to engage in random exploration for some time, until actionable information becomes available. This motivates the question as to how well the memory structure of an empirical agent performs in this initial stage.

\medskip
In the case where $\spc$ is finite and the agent's actions are deterministic it is easy to formulate this: Let $\actions$ be the set of available actions, and consider the graph with vertex set $\spc$, where a vertex $x$ is joined to a vertex $y$ labeled by an action $\alpha\in\actions$ if applying $\alpha$ at $x$ results in $y$. Thus, $\spc$ becomes endowed with the structure of a Markov chain, where we draw actions uniformly at random in every state. Focus on the case when all the actions available to the agent are reversible in the sense that there is an edge from $x$ to $y$ if and only if there is an edge from $y$ to $x$ (loops are allowed as well). Then the corresponding Markov chain is a random walk and its stationary distribution over $\spc$, denoted by $\pi$, is uniform \cite{Lovasz-random_walks_survey} over each connected component of the resulting transition system. Thus, each normalized weight $\witness{ab}^t/t$ is nothing but the empirical estimate of the joint probability, given by $\pi$, for $a,b\in\sens$ to fire synchronously.

Restricting to a reachability component, we may assume $\spc$ is connected. By abuse of notation, for $ab\in\pocg$ denote
\begin{equation}
	\pi(ab)=\pi(\rho(a)\cap\rho(b))
\end{equation}
Let $\mathtt{Dir}^\infty$ denote the matrix\footnote{Recall that $\mathtt{Dir}(\snap{S})$ introduced in Prop. \ref{poc sets from snapshots} is a directed graph. The new notation is intended to connote a matrix representation of such a graph.} with entries $\mathtt{Dir}^\infty_{ab}=1$ whenever $\rho(a)$ is contained in $\rho(b)$ up to a precision of $\threshold{ab}$, that is:
\begin{equation}\label{eqn:thresholded dir matrix}
	\pi(\rho(ab^\ast))<\min\left(
			\threshold{ab},
			\pi(ab),\pi(a^\ast b^\ast),\pi(a^\ast b)	
	\right)
\end{equation}
and set $\mathtt{Dir}^\infty_{ab}=0$ otherwise. This matrix represents the true poc set structure to be learned by the agent, as determined by the fixed learning thresholds. Analogously we let $\mathtt{Dir}^t$ be the matrix with $\mathtt{Dir}^t_{ab}=1$ iff the directed edge $ab\in\pocg$ is contained in the derived poc graph of $\snapdir{\snap{S}\at{t}}$ (and $0$ otherwise). A good measure of the agent's performance would be the behavior of the total error
\begin{equation}
	Err(t):=\norm{\mathtt{Dir}^t-\mathtt{Dir}^\infty}_1
\end{equation}
over time (the matrices viewed simply as vectors in $\ell^1$ of the appropriate dimension). By Theorem 5.1 in \cite{Lovasz-random_walks_survey}, the agent's random walk converges to $\pi$ at an exponential rate depending only on the transition structure of $\spc$ determined by the actions $\actions$. We conclude:
\begin{prop}\label{prop:convergence} Suppose a DBA performs a random walk on a connected $\spc$, at each moment in time performing one of a fixed finite set of reversible deterministic actions. Then $Err(t)$ converges to zero at an exponential rate. 
\end{prop} 

\medskip
With such strong performance guarantees for a broad class of empirical agents we are left to examine the variation in performance as a function of the geometry/topology of the environment (beyond the guarantees given by the preceding discussion) we have run simulations in the following settings:
\begin{itemize}
	\item[(a)] The agent performs a random walk along a path with $20$ edges (example \ref{example:N-path}), choosing between one step forward and one step back uniformly at random for every $t\in\mytime$, learning the poc structure of a sensorium consisting of $20$ `GPS' sensors, as described in example \ref{example:N-path};
	\item[(b)] The agent performs a random walk along a cycle with $20$ edges, choosing between a clockwise and a counter-clockwise step uniformly at random for every $t\in\mytime$, and learning a sensorium consisting of $20$ beacon sensors as described in example \ref{example:N-cycle};
	\item[(c)] The agent performs a random walk (up/down/left/right) on a square grid with $10$ `GPS' sensors along each of the $x-$ and $y-$ axes;
	\item[(d)] The agent performs a random walk along (forward/back) a path with $20$ edges, but the sensors are chosen to have random activation fields (randomly chosen subsets of the set of vertices along the path); the sensor fields have been drawn anew prior to each separate run.
\end{itemize}
The number of sensors is the same for each setting, and each agent carries out $50$ runs of a length that is cubic in the number of sensors, starting at a random position with an empty snapshot. We have tested $10$ different agents for each setting, corresponding to $10$ different values of the learning threshold, spread linearly in the interval from $1/(20)^3$ to the maximal meaningful value of $1/4$ (where one should not expect much useful learning to occur). 

The results are summarized in figure \ref{fig:empirical learners} plotting $Err(t)$, where we have replaced the matrix $\mathtt{Dir}^\infty$ as defined in \eqref{eqn:thresholded dir matrix} by the $\{0,1\}$-valued matrix
\begin{equation}\label{eqn:true dir matrix}
	\mathtt{Dir}^\infty_{ab}=1\IFF\rho(a)\subseteq\rho(b)
\end{equation}
to render the effect of choosing different values for the learning threshold more visible in the graph of $Err(t)$.
\begin{figure}
	\begin{center}
		\includegraphics[width=\columnwidth]{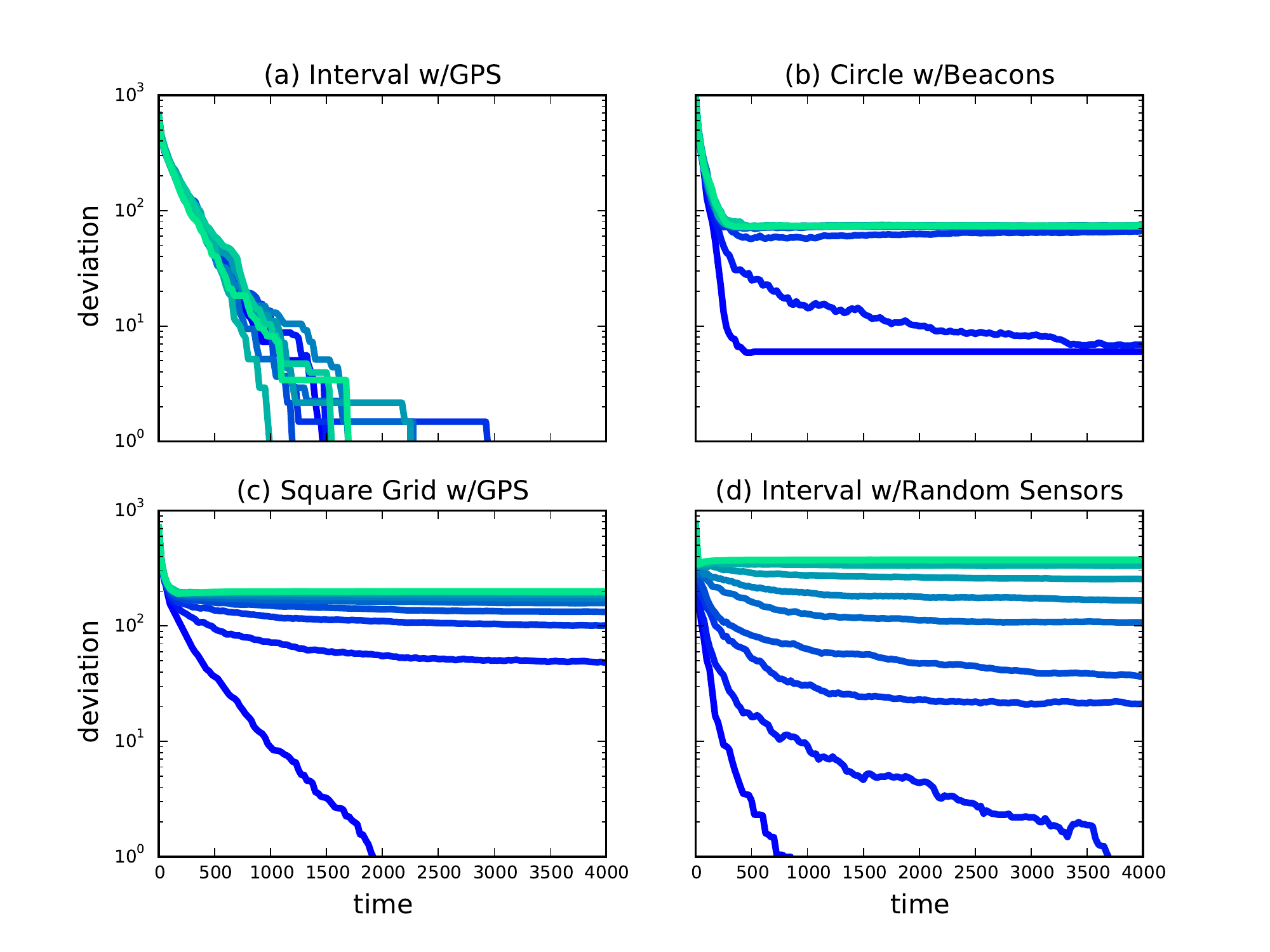}
		\caption{\small Logarithmic plot of the mean number of incorrect edges in the derived poc graph of an empirical snapshot (20 sensors), for learning thresholds varying linearly between $\tfrac{1}{4}$ (cyan/light) and $\tfrac{1}{20^3}$ (blue/dark), averaged over 50 runs of random walks each.
		\normalsize	\label{fig:empirical learners}}
	\end{center}
\end{figure}

The resulting plots show significant, though subtle, differences in performance between the four settings, illustrating the similarities and differences in the weak poc set structures being approximated, most notably:
\begin{itemize}
	\item[-] The sharper initial decline in the mean deviation for (b) and (c) in comparison to (a) is expected due to the relative abundance of crossing in the former, as opposed to complete nesting (see definition \ref{defn:nesting}) in the latter.
	\item[-] Performance in the random setting (d) seems to lag significantly behind performance in any of the structured settings.
	\item[-] Performance in the completely nested setting (a) seems to provide exponentially fast learning no matter what; by contrast, the other settings seem to experience a transition between two modes, depending on how small the learning threshold $\tau$ is:
	\begin{enumerate}
		\item For large $\tau$, the deviation plateaus.
		\item For small enough $\tau$, the deviation decreases to zero in finite time.
	\end{enumerate}
	We expect the critical value of $\tau$ in any finite setting to be somewhere around the minimum probability of a state (under the stationary distribution of the random walk): in order for a relation $a<b$ to be put on record, it is necessary for the agent to have visited $\rho(a)\cap\rho(b^\ast)$ at a frequency below $\tau$; the smaller $\tau$ is, the fewer false relations will be recorded for posterity.
	\item[-] Recalling that the poc set representing the ground truth in (c) is the direct sum (see \ref{example:direct sum}) of two smaller copies of (a) having 10 sensors each, we see that the crossing relations between the $x$-axis sensors and their $y$-axis counterparts account for 800 of the 1600 entries (two 20$\times$20 null sub-matrices) in the adjacency matrix of the derived graph. Thus, the two experiments are not that different: loosely speaking, the 10$\times$10 square grid experiment projects onto a Cartesian product of two 10-path experiments where the random walk on the 10-path becomes a lazy walk with probability $\tfrac{1}{2}$ to stay put. In other words, the behavior of (c) may be inferred from the behavior of (a).
	\item[-] Not so when comparing (a) and (c) with (b): note how the sub-critical values of the learning parameters in (a),(c) and (d) force the deviation plot to `plunge' into the $x$-axis {\it versus} the horizontal asymptote behavior of (b). In view of theorem \ref{thm:homotopy type}), our guess is that the environment (the circle) not being contractible has something to do with this qualitative change in behavior, but this requires further investigation.
\end{itemize}

\subsection{Discounted Snapshots}\label{subsection:discounted snapshot}
A notable weakness of empirical snapshots as a data structure is their potential high cost in space, due to the need for indefinitely maintaining integer-valued counters. In some sense, the {\it entire} history of the agent matters, and, in some sense, matters too much. This motivates the search for an alternative, more quantized, updating mechanism whose dependence on any given past observation weakens at a fixed rate.

\medskip
\subsubsection{Construction and Properties}\label{discounted:formalism}
\begin{defn}(discounted update)\label{defn:discounted update} Let $q\in[0,1]$ and let $\snap{S}$ be a probabilistic snapshot over $\sens$. For any complete $\ast$-selection $O$ on $\sens$ we define the snapshot $O\ast_{q}\snap{S}$ to be the snapshot with weights determined by
\begin{equation}
	\witness{ab}(O\ast_{q}\snap{S}):=q\witness{ab}(\snap{S})+(1-q)\ev{\indicator{O}}{a}\cdot\ev{\indicator{O}}{b}
\end{equation}
The state of $O\ast_q\snap{S}$ is set to $\coh{O}$, the reduction being computed with respect to the weak poc set structure derived from the new weights. Finally, define the \emph{$q$-discounted update} of $\snap{S}$ to be the snapshot $\snapfloor{O\ast_q\snap{S}}$ and we refer to $q$ as the \emph{decay parameter}.\defstop
\end{defn}
A significant advantage of the discounted update is its applicability to arbitrary probabilistic snapshots:
\begin{lemma}\label{lemma:discounted update} The $q$-discounted update of a probabilistic snapshot by a complete $\ast$-selection is probabilistic.\defstop
\end{lemma}
\begin{proof} It is clear that the discounted update preserves probabilisticity. Proposition \ref{snapshot truncation} finishes the proof.
\end{proof}

Consider the length of time (or the amount of evidence) it takes a discounted snapshot to acquire an implication, compared to the amount of evidence required for giving up an implication already on record.

Assuming a fixed value of the decay parameter $q$ over a considerable length of time, a lower bound on the amount of time $\Delta t$ required for $\witness{ab^\ast}$ to become small enough for a relation $a\leq b$ to be put on record is given by the situation when a long enough sequence of consecutive observations with $a\wedge b^\ast$ {\it not} occurring is made:
\begin{equation}
	q^{\Delta t}<\threshold{ab}\IFF \Delta t>\frac{\log_2\threshold{ab}}{\log_2 q}
\end{equation}
On the other hand, once the relation $a\leq b$ has been put on record, the number $\Delta t$ of successive observations of $a\wedge b^\ast$ required for replacing this relation with $a\pitchfork b$ must satisfy:
\begin{equation}
	\Delta t(1-q)>\threshold{ab}\IFF\Delta t>\frac{\threshold{ab}}{1-q}
\end{equation}
--~ this much is guaranteed by the truncation mechanism. Overall, it seems that choosing a value of $q$ with $1-q$ sufficiently small should produce meaningful learning: lower values of $\threshold{ab}$ make it both harder to learn and easier to unlearn a false relation, while maintaining a qualitative difference between the necessary requirements for either process.
 
\medskip
Keeping $q$ fixed over long periods of time places an emphasis on the values of the learning thresholds $\threshold{ab}$. As these values do not have to be chosen uniformly over the snapshot, one might want to vary the values of the learning thresholds individually with the aim of altering the flexibility of the learning process in the corresponding square. This opens up a doorway to employing methods for varying the learning thresholds and the decay parameter in ways analogous to \cite{Martius_et_al-infotaxis_driven_self_organization} and \cite{Cobo_Isbell_Thomasz-object_focused_Qlearning} as a means of improving the quality/dependability of the model space. The simulation results below emphasize the need for this kind of control, showing that a discounted agent is much more susceptible to changes in geometry and topology/combinatorics of the sensor fields than an empirical one.

\medskip
\subsubsection{Performance Analysis}\label{discounted:performance}
Figure \ref{fig:discounted learners} compares the mean performance of time-discounted snapshot learning from a random walk in the four settings described earlier in \ref{empirical:performance}, for the values of the decay parameter $q$ given by $q=1-\tfrac{1}{2^{k+2}}$, $k=\{0,\ldots,9\}$.

One immediately notices, in comparison with the empirical case, that the dependence of the learning process on the discount parameter is not monotone: it would seem that a choice of $k=5$ works best for all settings in terms of optimizing the eventual deviation,~ ---~ though it is hard to say what `best' would even mean for (d)~ ---~ while a choice of $k=4$ is more reasonable given the observed waiting time until meaningful learning occurs in the structured environments (a)-(c). 

\begin{figure}[t]
	\begin{center}
		\includegraphics[width=\columnwidth]{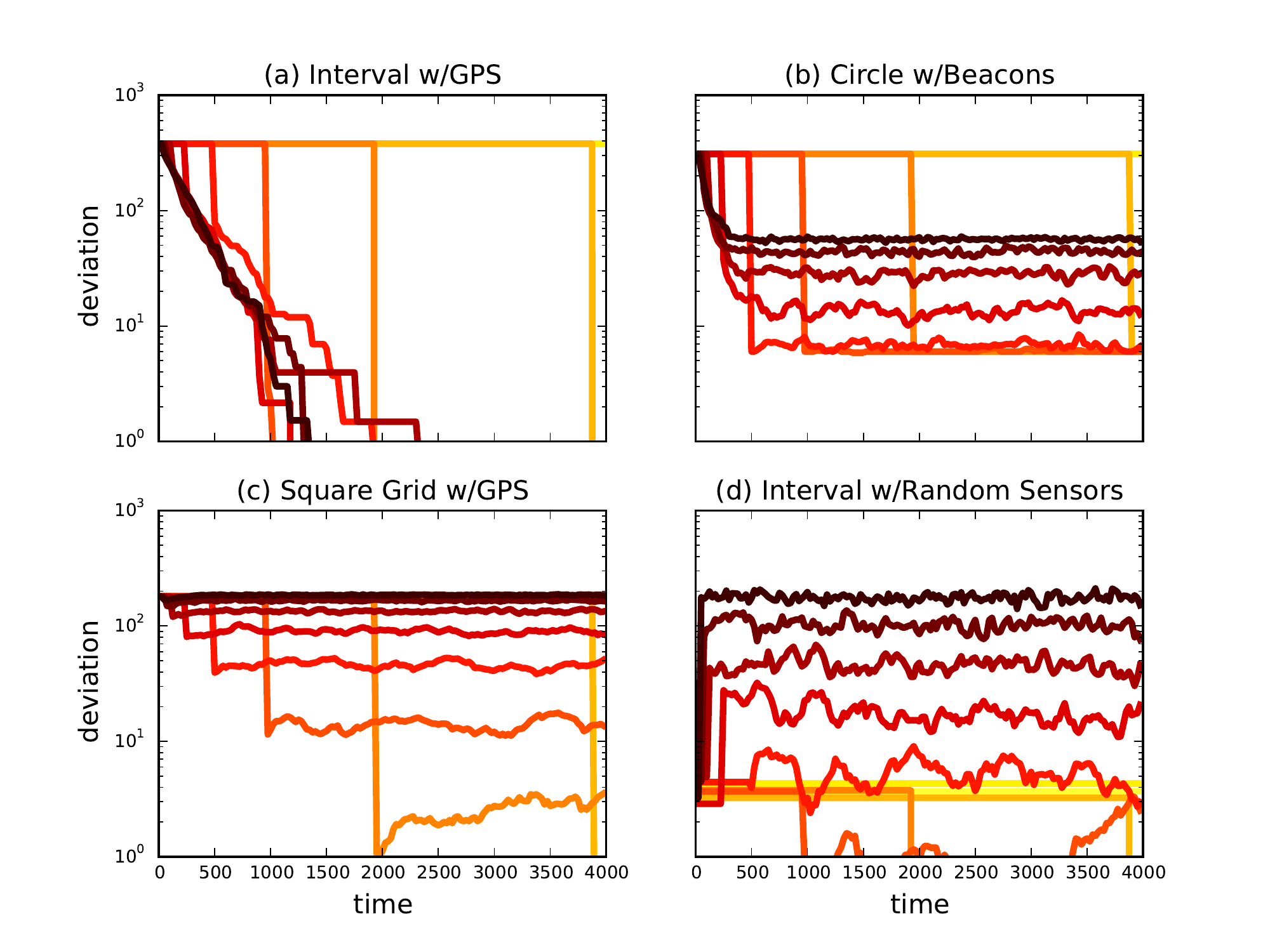}
		\caption{\small
		Mean number of incorrect edges in the derived poc graph of a discounted snapshot in 4 environments (20 sensors each) for varying values of the decay parameter, $q=1-\tfrac{1}{2^{k+2}}$, $k$ from $0$ (red/dark) to $9$ (yellow/light), averaged over 50 runs of a random walk.
		\normalsize\label{fig:discounted learners}}
	\end{center}
\end{figure}

Similar observations to those made for the empirical case (figure \ref{fig:empirical learners}) regarding the interplay between `learning modes' and geometry/topology can be made here as well, but are more subtle, as the comparison in figure \ref{fig:learning quality} shows.

\begin{figure}[t]
	\begin{center}
		\includegraphics[width=\columnwidth]{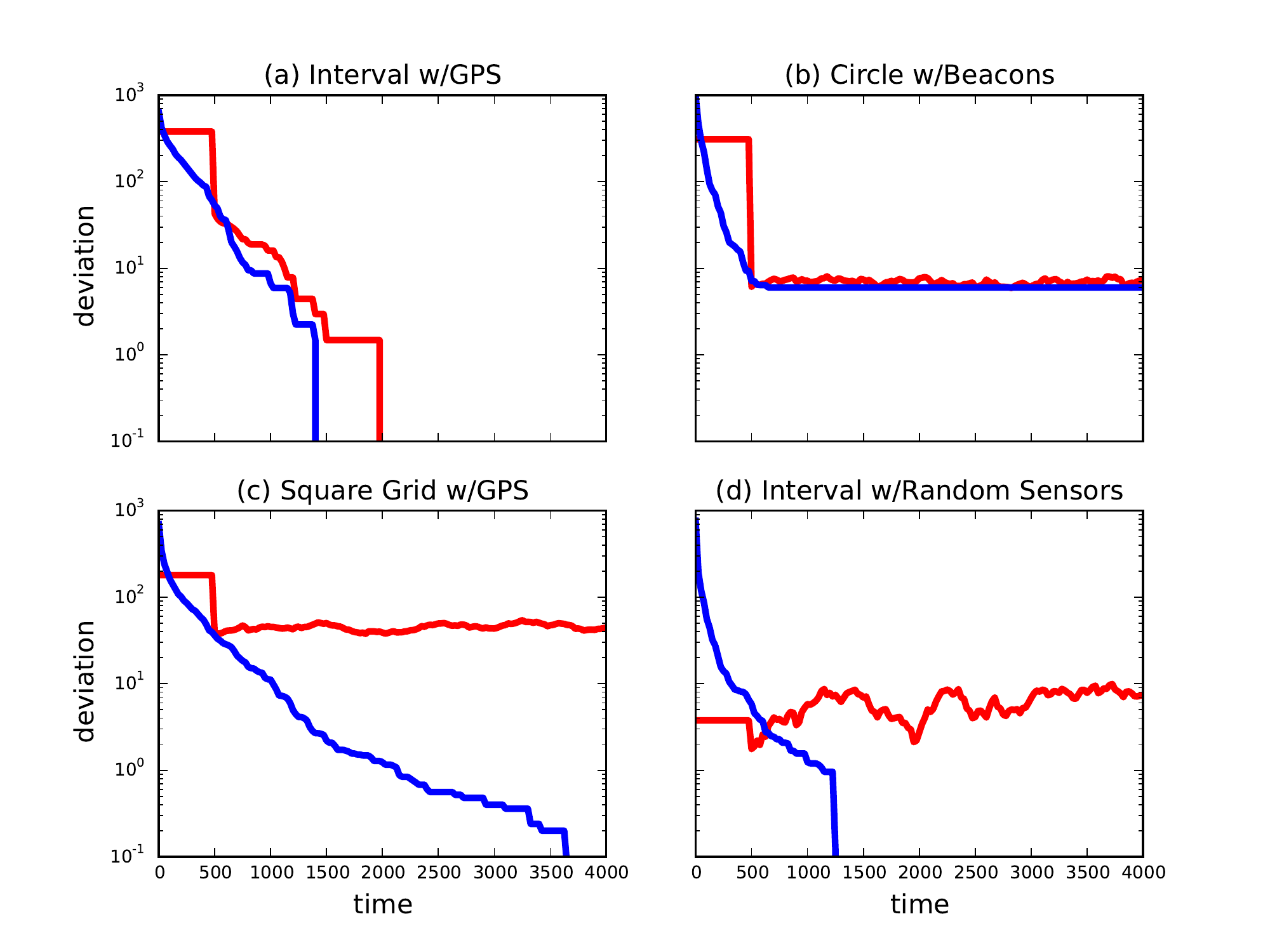}
		\caption{\small A comparison of the mean number of incorrect edges in the derived poc graph as a function of time, for an empirical snapshot (blue) and a discounted one (red). Here $\tau=1/20^3$ and $q=1-\tfrac{1}{2^6}$. \normalsize\label{fig:learning quality}}
	\end{center}
\end{figure}

\subsection{Further Adjustments to the Weak Poc Structure}\label{subsection:equivalences} The implication record constructed from a probabilistic snapshot in the preceding section does not recognize possible equivalences among sensations: if, for whatever reason, a relation of the form 
\begin{equation}\label{eqn:equivalence}
	\witness{ab^\ast}=\witness{a^\ast b}=0
\end{equation}
takes place in a snapshot $\snap{S}$, it becomes reasonable to interpret it as the logical equivalence $a\IFF b$, yet $\snapdir{\snap{S}}$ will not register any relations in the square $\snap{S}\res{ab}$, barring an agent equipped with $\snap{S}$ from utilizing the currently observed equivalence. 

Thus, an adjustment of $\snapdir{\snap{S}}$ is required if we are to allow our agents the advantage of reasoning about equivalences. The following extension of $\snapdir{\snap{S}}$ turns out to serve our purposes for a restricted class of probabilistic snapshots:
\begin{defn}\label{defn:poc graph with equivalences} Let $\snap{S}$ be a probabilistic snapshot. The poc graph $\snapdir{\snap{S}}_0$ is defined to be the poc graph obtained from $\snapdir{\snap{S}}$ by adding the directed edges $ab,ba,a^\ast b^\ast,b^\ast a^\ast$ for each $ab\in\pocg$ satisfying \eqref{eqn:equivalence}.\defstop
\end{defn}

It turns out that $\snapdir{\snap{S}}_0$ gives rise to an adequate weak poc set structure and model space, provided $\snap{S}$ satisfies the additional requirement:
\begin{defn} A snapshot $\snap{S}$ is said to satisfy the {\it triangle inequality}, if
\begin{equation}\label{eqn:triangle inequality for snapshots}
	\witness{a^\ast b}+\witness{ab^\ast}+\witness{b^\ast c}+\witness{bc^\ast}\geq\witness{a^\ast c}+\witness{ac^\ast}
\end{equation}
holds whenever $ab,bc,ac\in\pocg$.\defstop
\end{defn}

A class of examples of special significance in this work is that of snapshots $\snap{S}$ whose edge weights are derived from a measure $\mu$ on a space $Z$ by pulling back along a realization $\rho:\sens\to Z$ as follows:
\begin{equation}
	\witness{ab}=\mu\left(\rho(a)\cap\rho(b)\right)
\end{equation}
The triangle inequality for $\snap{S}$ is then an immediate consequence of the well-known (e.g. \cite{Deza_Laurent-cutbook}, chapter 3) triangle inequality for measures:
\begin{equation}
	\mu(A\symplus C)\leq\mu(A\symplus B)+\mu(B\symplus C)\,,
\end{equation}
where $A,B,C\subset Z$ are arbitrary measurable sets and $A\symplus B$ denotes the symmetric difference $(A\minus B)\cup(B\minus A)$.

The coincidence indicators $c_{ab}^t$ of \eqref{eqn:coincidence indicators} are a special case of this example (where $\mu$ is an atomic measure), and so are empirical snapshots (as their weights are sums of coincidence indicators). Discounted snapshots fall into this class, too, as their weights are convex combinations of coincidence indicators.

Due to the technical nature of the interactions between $\snapdir{\snap{S}}$ and the extension $\snapdir{\snap{S}}_0$, we postpone the formal discussion of these interactions to appendix \ref{proofs:adding equivalences}. The bottom line, however, is that for any probabilistic snapshot structure satisfying the triangle inequality our agent may safely apply the control protocols of the next section to the {\it extended} poc graph derived from the agent's current snapshot to arrive at action choices while taking advantage of the perceived equivalences within the sensorium. 

Although technically we are obliged to distinguish between $\snapdir{\snap{S}}$ and $\snapdir{\snap{S}}_0$, as well as between the weak poc set structures they correspond to, we will treat these objects as identical for the sake of simplifying the rest of the exposition.

\section{Control with Snapshots}\label{section:control}
This section introduces the basic control function of a snapshot. We begin with introducing a formalism designed to treat discrete actions as a sub-structure of the binary sensorium, and discuss the effect of this formalism on shaping the model space (\ref{subsection:actions}). We next turn to a discussion of the snapshot $\snap{S}\at{t}$ as a highly efficient computational mechanism for coherent state-updating and for decision making based on the geometry of the model space $\model\at{t}$ (\ref{subsection:planning}). 

At the technical level, this section requires an understanding of the convexity theory of cubings: The classical results are covered in appendix \ref{subsection:convexity}, while our new technical results underlying the use of snapshots for greedy navigation in cubings are covered in appendix \ref{subsection:greedy navigation}.

Building on these results, section \ref{propagation} introduces the mechanism of {\it signal propagation over a snapshot} which realizes the computation both of coherent projection and of closest point projections to prescribed convex subsets of the associated model space. This mechanism suggests a view of snapshot architecture as highly simplified connectionist architectures, and some related work in the literature is discussed.

At the heart of our proposed decision-making algorithm is an assumption that the sensorium is rich enough to detect direct causal relations between actions and other sensations. We provide a fairly broad formalization of this assumption in \ref{evaluating an action} (with an example in \ref{example:N-path with propagation}), and prove the ability of an agent to correctly `halucinate' the immediate consequences of taking an action, provided sufficient exposure.

An algorithm using this tool to attempt greedy navigation over $\model\at{t}$ to a specified target state is proposed in \ref{general planning}, and some of its failure modes are discussed as a motivation for future research on judicious dynamical expansions of the sensorium which would allow the agent to overcome the navigational obstructions formed by states in $\model\at{t}$ having no witness in the situation space $\spc$.

Finally, in \ref{gradient descent}, we explore the performance of some {\it excitation-driven DBAs}: agents endowed with an excitation level that changes depending on their distance from a target in the environment $\env$; the agents are capable of sensing an increase or a decrease in excitation, and seek instant gratification in the sense of operating on the mandate to always pick an action guaranteeing an increase in excitation (or else act randomly). We compare the performance of such agents in the domains considered in \ref{empirical:performance} and \ref{discounted:performance}; in these domains it is easy to guarantee arrival to the target given {\it a-priori} knowledge of the correct snapshot structure, but we are interested in the agent's performance as they learn the problem "from scratch".

\subsection{Defining Actions}\label{subsection:actions}
We will now restrict attention to DBAs with a sensorium $\sens$ endowed only with state (degree $0$) and transition (degree $1$) sensors. As before, we denote the realization of a sensor $a\in\sens$ by $\rho(a)$, where $\rho(a)\subseteq X$ for a state sensor and $\rho(a)\subseteq X\times X$ for a transition sensor. Thus, state sensors and transition sensors may be viewed as Boolean and situational fluents over the situation space $\spc$, which is sufficient for setting up a discussion of actions and competencies according to McCarthy and Hayes \cite{McCarthy_Hayes-Philosophical_probs_of_AI}.

For our agents, we posit a set $\actions\subset\sens$ of transition sensors, each of which may be switched on and off {\it at will}, earning them the name of {\it actions}. To be precise, our requirements are:
\begin{itemize}
	\item{\bf Actions are binary. } We assume $\actions\cap\actions^\ast=\varnothing$, and we denote the poc subset $\actions\cup\actions^\ast\cup\{\minP,\maxP\}$ by $\pact$.
	\item{\bf Every action has outcomes. } For any $\alpha\in\actions$ and $x\in\spc$, the sets
	\begin{equation}
		\alpha(x)=\set{y\in\spc}{x\times y\in \rho(\alpha)}
	\end{equation}
	are non-empty subsets of $\spc$.
\end{itemize}
In this we depart slightly from the accepted notion of actions in the literature on transition systems of various flavors (e.g. \cite{Erdmann-IJRR'09-StrategySpace},\cite{Sutton_Precup_Singh-SMDPs_temporal_abstraction_in_RL}), where actions are attached to states and the collection of actions available at each state may differ, depending on that state. Instead, we consider actions as nothing more than control signals, sent by the agent's `mind' to the agent's `body' in order to invoke (or not) one or more of a fixed set of available behaviors. It is the purpose of the `mind' to identify whether or not a control signal produces meaningful outcomes as those outcomes are being sensed.

\medskip
\subsubsection{Invoking Actions Synchronously}\label{synchronous actions} Our sensor-centric approach to actions reflects the viewpoint that (1) an action $\alpha\in\actions$ taken at a state $x\in\spc$ imposes a time-independent restriction on the set of states the system may enter in the following moment, and (2) the agent is capable~ ---~ at least in principle~ ---~ of observing its own decisions as they are being invoked. We must now discuss the precise extent to which these principles may or may not restrict our initial suggestion that the sensations in $\pact$ are controllable.

For example, consider the situation where the agent is not engaging in an action $\alpha\in\actions$ during a transition from state $x$ to state $y$. This implies $\alpha^\ast$ is on during this transition, which restricts the possible values of $y$ to $\spc\minus\alpha(x)$. 
Hence, not invoking any of the available $\alpha\in\actions$ must then restrict $y$ to the intersection $\bigcap_{\alpha\in\actions}\spc\minus\alpha(x)$, the set of possible outcomes of the "no-action".

More generally, allowing a number of actions to be taken at the same time (while not engaging in the rest) forces the following interpretation by our sensing model:
\begin{enumerate}
	\item A {\it generalized action} by the agent is a complete $\ast$-selection $A$ on $\pact$ (recall definition \ref{defn:selection});
	\item The realization of a (generalized) action $A\in S(\pact)^0$ is defined to be
	\begin{equation}
		\rho(A)=\bigcap_{\alpha\in A}\rho(\alpha)\,,
	\end{equation}
	or, equivalently, for every $x\in\spc$, the set of possible outcomes of an action $A$ equals
	\begin{equation}
		A(x)=\bigcap_{\alpha\in A}\alpha(x)
	\end{equation}
\end{enumerate}

For this extended collection of actions one notices that the second requirement of an action~ ---~ $A(x)\neq\varnothing$ for all $x\in\spc$~ ---~ does not necessarily hold: for example, moving forward along a rail contradicts any motion in the opposite direction. We will say that a generalized action $A\in S(\pact)^0$ is {\it admissible at $x\in\spc$} if $A(x)\neq\varnothing$, and that $A$ is {\it admissible}, if it is admissible at $x$ for all $x\in\spc$.

Aside from setting natural bounds on the meaning of the initial statement that actions are available to the agent at will, the notion of admissibility of a generalized action explains how to interpret the poc set structure induced on $\pact$ from the realization $\rho$: if
\begin{equation}
	\alpha<\beta\IFF\rho(\alpha)\subset \rho(\beta)
\end{equation}
happens to hold for $\ppoc\at{t}$, then every generalized action admissible at a point $x\in\spc$ defines a vertex of $\cube{\pact}$ no matter the choice of $x\in\spc$. Similarly, generalized actions not showing up as vertices $\cube{\pact\at{t}}$, where $\pact\at{t}$ denotes the restriction of the poc set structure $\ppoc\at{t}$ to $\pact$, represent the agent's belief at time $t$ regarding combinations of elementary actions it cannot achieve at that moment.

In the simple examples considered in this paper all agents will be endowed with a collection of mutually exclusive atomic actions. By this we mean that $\alpha<\beta^\ast$ holds for all $\alpha,\beta\in\actions$ ($\alpha\neq\beta$). Equivalently, only the "no-action", $\set{\alpha^\ast}{\alpha\in\actions}$, and the "pure" actions $\{\alpha\}\cup\set{\beta^\ast}{\beta\in\actions\minus\{\alpha\}}$ are admissible, and the resulting cubing $\cube{\pact}$ takes the form of a {\it starfish}: a tree with only one vertex of degree$\geq 2$ given by the "no-action" and with a set of leaves in one-to-one correspondence with the set of "pure" actions (see example \ref{example:N-path} and figure \ref{fig:N-path}b).

\medskip
\subsubsection{Observations}
The following set is the set of {\it observations} in $\ppoc$ (note that it is closed under the $\ast$-operator):
\begin{equation}\label{defn:observations}
	\pobs:=(\sens\minus\pact)\cup\{\minP,\maxP\}\,,
\end{equation}
and stands for the set of "passive" sensations, as opposed to actions. Sections \ref{subsection:snapshot preliminaries}-\ref{subsection:equivalences} explain how a trajectory $\varphi\at{t}$, $t\geq 0$ may be used to form an evolving sequence of weak poc-set structures $(\ppoc\at{t})_{t\geq 0}$ over $\sens$, with each $\ppoc\at{t}$ representing the agent's belief at time $t$ regarding which implications among the sensors in $\sens$ hold true throughout time. Two poc subsets of $\ppoc\at{t}$ are formed by restricting its poc structure:
\begin{itemize}
	\item $\pact\at{t}$ is the induced poc structure on $\pact$; 
	\item $\pobs\at{t}$ is the induced poc structure on $\pobs$.
\end{itemize}
We are interested in the interaction between these smaller poc sets and the full model space, $\cube{\ppoc\at{t}}$. One has two surjections
\begin{equation}
	\begin{array}{c}
		proj_{act}:\cube{\ppoc\at{t}}\to\cube{\pact\at{t}}\\
		proj_{obs}:\cube{\ppoc\at{t}}\to\cube{\pobs\at{t}}
	\end{array}
\end{equation}
defined, at the level of $0$-skeleta, as follows: $proj_{act}$ sends a coherent $\ast$-selection $A$ on $\ppoc\at{t}$ to the $\ast$-selection $A\cap\pact$, and similarly for $proj_{obs}$. Hence, at the level of $0$-skeleta, there is a map:
\begin{equation}\label{eqn:representing in product}
	\cube{\ppoc\at{t}}\to\cube{\pact\at{t}}\times\cube{\pobs\at{t}}
\end{equation}
In fact, Sageev-Roller duality \cite{Roller-duality} implies a much more precise statement:
\begin{prop}\label{prop:representing in product} The map above \eqref{eqn:representing in product} is a median-preserving embedding of $\cube{\ppoc\at{t}}$ in the Cartesian product $\cube{\pact\at{t}}\times\cube{\pobs\at{t}}$.
\end{prop}
\begin{proof} See proof in appendix \ref{proofs:representing in product}.
\end{proof}

\medskip
\subsubsection{Example: discrete path with motion}\label{example:path with motion} To illustrate the description of the model space provided by proposition \ref{prop:representing in product}, consider an agent moving in steps of unit length along a path of integer length $L>1$. Formally, the environment is given by $\env=\{0,\ldots,L\}$ and the agent has the actions defined by:
\begin{equation}\label{eqn:move_example}
	\begin{array}{rcl}
	y\in\wait(x)
		&\IFF&	\pos(y)=\pos(x)\\
	y\in\fwd(x)
		&\IFF&	\pos(y)=\min\left\{L,\pos(x)+1\right\}\\
	y\in\back(x)
		&\IFF&	\pos(y)=\max\left\{0,\pos(x)-1\right\}
	\end{array}
\end{equation}
enabling motion from any vertex $k\in\env$ to the adjacent $k+1$ and $k-1$, when they exist. We also endow the agent with sensors $a_1,\ldots,a_L\in\sens$ realized as:
\begin{equation}\label{eqn:sensors_on_interval}
	\ev{a_k}{x}=1\IFF\pos(x)<k
\end{equation}

Up to symmetry, the only relations holding in the existing scheme are
\begin{equation}
	a_1<a_2<\ldots<a_L\,,
\end{equation}
the `starfish' relations for $\pact$: 
\begin{equation}
	\fwd<\back^\ast\,,\quad
	\back<\wait^\ast\,,\quad
	\wait<\fwd^\ast\,,
\end{equation}
for the actions $\{\fwd,\back,\wait\}$, and the two relations
\begin{equation}\label{eqn:boundary relations}
	\fwd<a_1^\ast\,,\quad
	\back<a_L\,,
\end{equation}
indicating $\pos(x)=0$ may not be reached by applying $\fwd$, while $\pos(x)=L$ may not be reached by applying $\back$. No other relations hold universally. Let $\ppoc$ denote the poc set structure over $\sens$ recording these relations.  

\begin{figure}[ht]
	\begin{center}
		\includegraphics[width=\columnwidth]{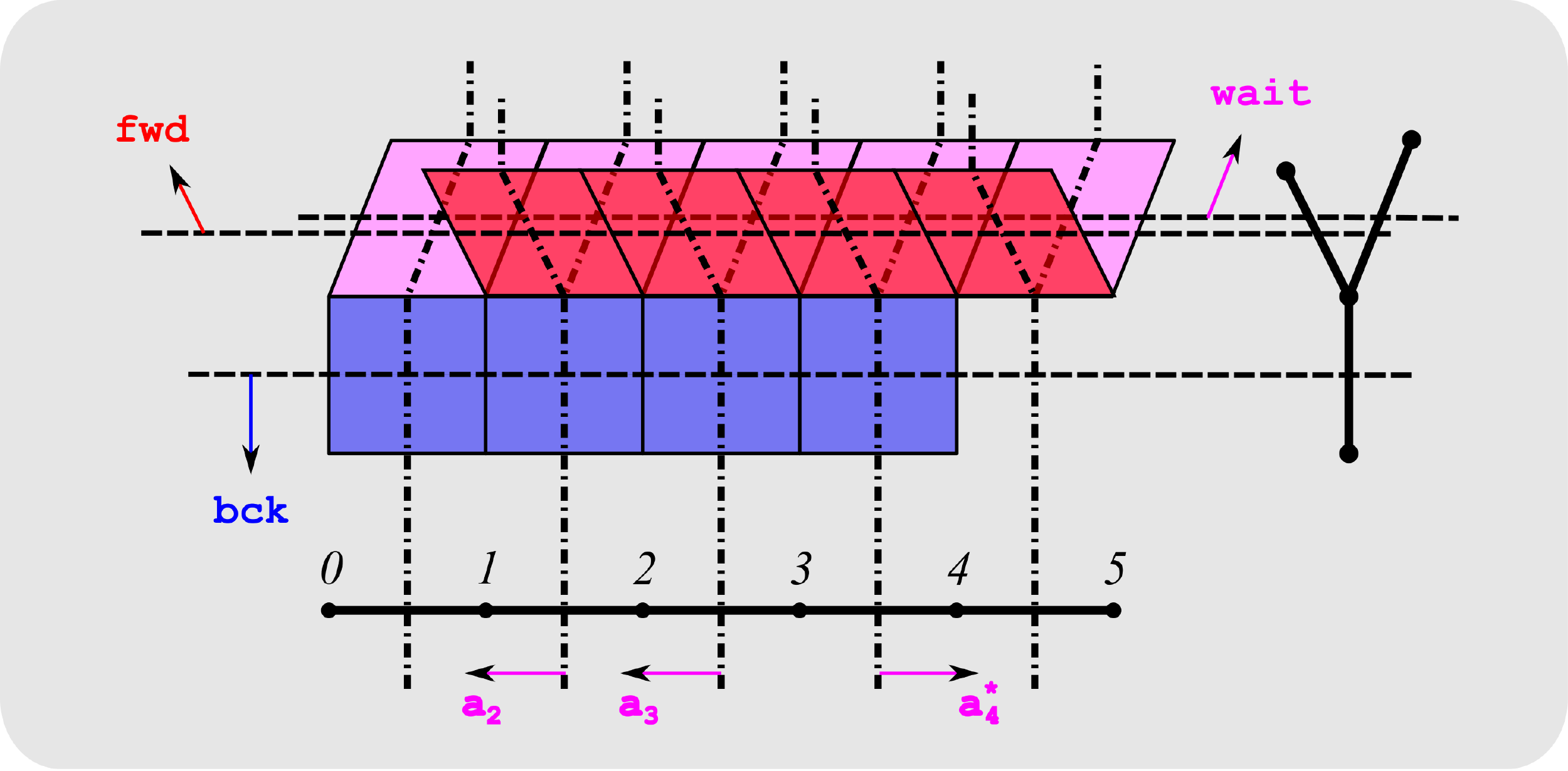}
		\caption{\small Model space for a DBA placed in a discrete path, depicted together with its projections to $\cube{\pact}$ (right) and to $\cube{\pobs}$ (below). This is the case $L=5$ of example \ref{example:path with motion}. \normalsize}\label{fig:interval with motion, basic}
	\end{center}
\end{figure}

$\cube{\ppoc}$ is the result of forming the Cartesian product of a $3$-pronged starfish $\cube{\pact}$ with the path of length $L$ obtained\footnote{Compare with example \ref{example:N-path} and figure \ref{fig:N-path}.} as $\cube{\pobs}$, and then removing two squares as shown in figure \ref{fig:interval with motion, basic}, due to the relation in \eqref{eqn:boundary relations}.

\subsection{Reactive Planning}\label{subsection:planning} 
\subsubsection{Statement of the planning problem}\label{planning problem statement} In this section we consider a DBA at time $t>0$, equipped with a snapshot $\snap{S}\at{t}$ with a derived poc graph $\pog\at{t}=\snapdir{\snap{S}\at{t}}$ and associated weak poc set $\ppoc\at{t}$ (but keep in mind the notational simplifications at the end of \ref{subsection:equivalences}). The agent's tasks at hand are:
\begin{description}
	\item[\bf{(T1)}] Predict the immediate outcome of any $\alpha\in\pact\at{t}$ (or, more generally, of any $A\in\cube{\pact\at{t}}$);
	\item[\bf{(T2)}] Given a set $T\subset\sens$ of target sensations to be achieved {\it jointly}, decide on a (generalized) action $\decide{t}$ for the agent to invoke in the next transition.
\end{description}
Both tasks need to be achieved based on the agent's record of the current state, $\current{t}=\state{}\snap{S}\at{t}$, which is a coherent (though not necessarily complete) $\ast$-selection on $\ppoc\at{t}$ obtained from the complete $\ast$-selection $\observe{t}$ representing the agent's raw observation of the current state at time $t$, by coherent projection \eqref{eqn:coherent projection}. We will keep all the above notation fixed through the rest of this section.

It is crucial that we interpret these tasks in terms of the model space $\model\at{t}=\cube{\ppoc\at{t}}$. For any subset $B\subset\sens$, define the set:
\begin{equation}\label{eqn:easy defn of convex set}
	\half{B}:=\set{V\subset\sens}{V\text{ is a vertex of }\model\at{t}\text{ containing } B}
\end{equation}
These are known to be {\it precisely} the convex subsets of the $1$-skeleton of $\model\at{t}$ (see appendix \ref{subsection:convexity}). Thus, the agent is assigned the problem of reaching the convex set $\half{T}$ from a (possibly unknown) position in the convex set $\half{\current{t}}$.

\medskip
\subsubsection{Signal Propagation over a Snapshot}\label{propagation} The purpose of $\pog\at{t}$ is to serve as an inference tool for the agent. Recall that $\pog\at{t}$ is formed from the weight structure of the snapshot $\snap{S}\at{t}$, which, in turn, is a result of updating the weights on $\snap{S}\at{t-1}$ with the {\it raw} observation $\observe{t}$. The last step of the update is the `loading' of $\pog\at{t}$ with the current state $\current{t}=\coh{\observe{t}}$ of the agent.

\begin{defn} Let $B\subset\sens$. Denote by $[\pog\at{t},B]$ the weighted graph obtained from $\pog\at{t}$ by attaching the Boolean weight $\ev{\indicator{B}}{v}$ to each vertex $v\in\sens$, and refer to it as {\it $\pog\at{t}$ being loaded with $B$}.\defstop
\end{defn}
\begin{defn}\label{defn:propagation algorithm} A \emph{propagation algorithm along $\pog\at{t}$} is any algorithm which, for any coherent load $B\subset\sens$ and \emph{any} $T\subseteq\sens$ accepts $[\pog\at{t},B]$ and $T$ as input and produces as its output the loaded graph $[\pog\at{t},R]$ where $a\in R$ if and only if: 
\begin{enumerate}
	\item there is a directed path in $\pog\at{t}$ from $B\cup T$ to $a$, or --
	\item there is no directed path in $\pog\at{t}$ from $a$ into $T^\ast$.
\end{enumerate}
The set $R\subset\sens$ is said to be the {\it result of propagating the signal $T$ along $[\pog\at{t},B]$}.\defstop
\end{defn}
Applying the convexity theory of the model space $\model\at{t}$~ ---~ specifically corollary \ref{cor:projection by propagation} and proposition \ref{prop:adding equivalences}~ ---~ we find the following applications for propagation:
\begin{lemma}[Implementing the State Update]\label{lemma:propagation gives coherent projection} For any propagation algorithm, propagating the signal $\observe{t}$ along $[\pog\at{t},\varnothing]$ produces $\current{t}=\coh{\observe{t}}$.\defstop
\end{lemma}
\begin{lemma}[Reasoning in Snapshots]\label{lemma:propagation gives closest point projection} Let $T\subset\sens$ be any set. For any propagation algorithm, propagating the signal $T$ along $[\pog\at{t},\current{t}]$ produces the projection in $\model\at{t}$ of the current state $\half{\current{t}}$ onto the reduced target $\half{\coh{T}}\subset\model\at{t}$.
\end{lemma}
The first lemma explains how to implement the snapshot update, given a propagation algorithm: 
\begin{enumerate}
	\item use the raw observation $\observe{t}$ and $\snap{S}\at{t-1}$ to recalculate the edge weights for $\snap{S}\at{t}$;
	\item compute the derived graph $\pog\at{t}$;
	\item propagate the signal $\observe{t}$ over $[\pog\at{t},\varnothing]$ to compute $\current{t}=\coh{\observe{t}}$.
\end{enumerate}
The second lemma is the key tool for turning a propagation algorithm into the planning algorithms we discuss in the rest of this section.

\medskip
In practice, one can implement propagation using a variant of depth-first search (DFS) on $\pog\at{t}$ \cite{Cormen_et_al-intro_to_algorithms}, while maintaining an expanding record of vertices visited~ ---~ see algorithm 1. This algorithm clearly has time complexity that is at most quadratic in the number of sensors, and we conclude:
\begin{cor}[Quadratic Snapshot Maintenance] Both the time and space complexity of updating the snapshot $\snap{S}\at{t-1}$ with an observation $\observe{t}$ to form $\snap{S}\at{t}$ are at most quadratic in $\card{\sens}$.\defstop
\end{cor}

\begin{algorithm}[t]
\caption{A simple implementation of propagation of a signal $T$ over a poc graph $\pog$ loaded with $B$, based on depth-first search.}
\label{alg:propagation_by_DFS}
\begin{algorithmic}
\Function{main}{$\pog,B,T$}\Comment{Propagating $T$ over $[\pog,B]$}
	\State $\mathtt{visited}\gets\varnothing$ \Comment{A global variable}
	\State $U\gets$\Call{closure}{$\pog,T$}
	\State \Return $(B\cup U)\minus U^\ast$
\EndFunction

\Function{closure}{$\pog,T$}\Comment{Forward closure of $T$ in $\pog$}
	\ForAll{$a\in T$}
		\State\Call{explore}{$\pog,a$}
	\EndFor
	\State \Return $\mathtt{visited}$
\EndFunction

\Procedure{explore}{$\pog,v$}\Comment{Recursive step}
	\State $\mathtt{visited}\gets\mathtt{visited}\cup\{v\}$
	\ForAll{$w\in$\Call{children}{$\pog,v$}$\minus\mathtt{visited}$}
		\State\Call{explore}{$\pog,w$}
	\EndFor
\EndProcedure

\Function{children}{$\pog,v$}\Comment{Children of $v$ in $\pog$}
	\State \Return $\set{w\in\sens}{vw\in\pog}$
\EndFunction 
\end{algorithmic} 
\end{algorithm}

A far more efficient implementation is possible provided sufficient parallel processing power, by realizing propagation directly on $\pog\at{t}$ in a distributed fashion, using corollary \ref{cor:projection by propagation}: given $[\pog\at{t},\current{t}]$ and a target $T$ one first follows all directed paths in $\pog\at{t}$ emanating from $(\current{t}\cup T)$ loading the traversed vertices with $1$, and then follows all {\it reverse} paths emanating from $T^\ast$ and loads their vertices with zeros. Implementing this algorithm in practice is problematic for large $\card{\sens}$ in view of the high plasticity of the graph $\pog\at{t}$ and the potentially prohibitive requirement for the DBA to maintain up to $O(\card{\sens})$ processes, all active at the same time. Despite its current impracticality, such an implementation seems evocative of the notion of neuronal networks. We discuss this tentative connection in section \ref{discussion:connectionist}.

\medskip
\subsubsection{Algorithm: Computing the consequences of an action}\label{evaluating an action}
Planning of any kind requires an ability to sense the context of an action. This ability may be imparted to the agent by introducing sensors of the form
\begin{equation}\label{eqn:context}
	\ev{\alpha\wedge s}{\varphi}\at{t}=\ev{\alpha}{\varphi}\at{t}\cdot\state{s}{\snap{S}\at{t-1}}
\end{equation}
where $\alpha$ is an action and $s\in\sens$ is any sensor. The idea behind constructing $\alpha\wedge s$ in this way is for snapshots to be able to detect implications of the form "invoking $\alpha$ when $s$ is on leads to $s'$" simply as directed paths from $\alpha\wedge s$ to $s'$ in the derived graph. From a formal point of view, allowing this kind of sensor requires the observability of the snapshot in $\spc$, that is: in order for the values of sensors in $\sens$ to be allowed as input to (possibly other) sensors in $\sens$, it is necessary by our formalism for $\spc$ to carry information regarding these values in its state. 


The problem of constructing a judicious process of enriching the sensorium with an effective collection of introspective sensors is set aside for future research. Instead, in the present model and simulations used to illustrate these ideas for the purposes of this paper we have committed to a sensorium containing an over-abundance of such sensors:
\begin{itemize}
	\item{\bf Position Sensors. } We assume the environment $\env$ is given as the union of a collection $\hsm{U}$ of its subsets, the agent being given a state sensor $\loc{U}$ for each $U\in\hsm{U}$ satisfying $\ev{\loc{U}}{x}=\ev{\indicator{U}}{\pos(x)}$.
	\item{\bf Actions. } A collection of actions (in the form of $1$-sensors) is provided.
	\item{\bf Contextualized actions. } For each $U\in\hsm{U}$ and $\alpha\in\pact$ the agent is given the sensors $\alpha\wedge\loc{U}$ and $\alpha\wedge\loc{U}^\ast$.
\end{itemize}
Under these assumptions, the following result yields a mechanism allowing the agent to `hallucinate' the broadest consequences of an action within the context of its current model space $\model\at{t}$:
\begin{cor}[Computing the Consequences of an Action]\label{consequences of an action} For any generalized action $A\subset\pact$, the result of applying $A$ in the transition from time $t>0$ to time $t+1$ is the result of propagating the collection $\left\{\alpha\wedge\loc{U}\right\}_{\loc{U}\in\current{t}, \alpha\in A}$ along $[\pog\at{t},\current{t}]$.\defstop
\end{cor}
Thus, propagation provides a provably correct and computationally efficient mechanism for predicting the immediate outcomes of an action, provided a sensorium of the above form and a snapshot faithfully recording the nesting relations among the sensors.

Combined with the results of section \ref{section:snapshots} demonstrating that learning the correct relations within a fixed sensorium with a high degree of fidelity is possible even for an agent performing a random walk, the last corollary suggests that effective (and efficient) planning and closed loop control~ ---~ bundled together with life-long learning features~ ---~ are entirely feasible for DBAs carrying a snapshot architecture. We discuss both problems in the following paragraphs.

\begin{figure}[ht]
	\begin{center}
		\includegraphics[width=.8\columnwidth]{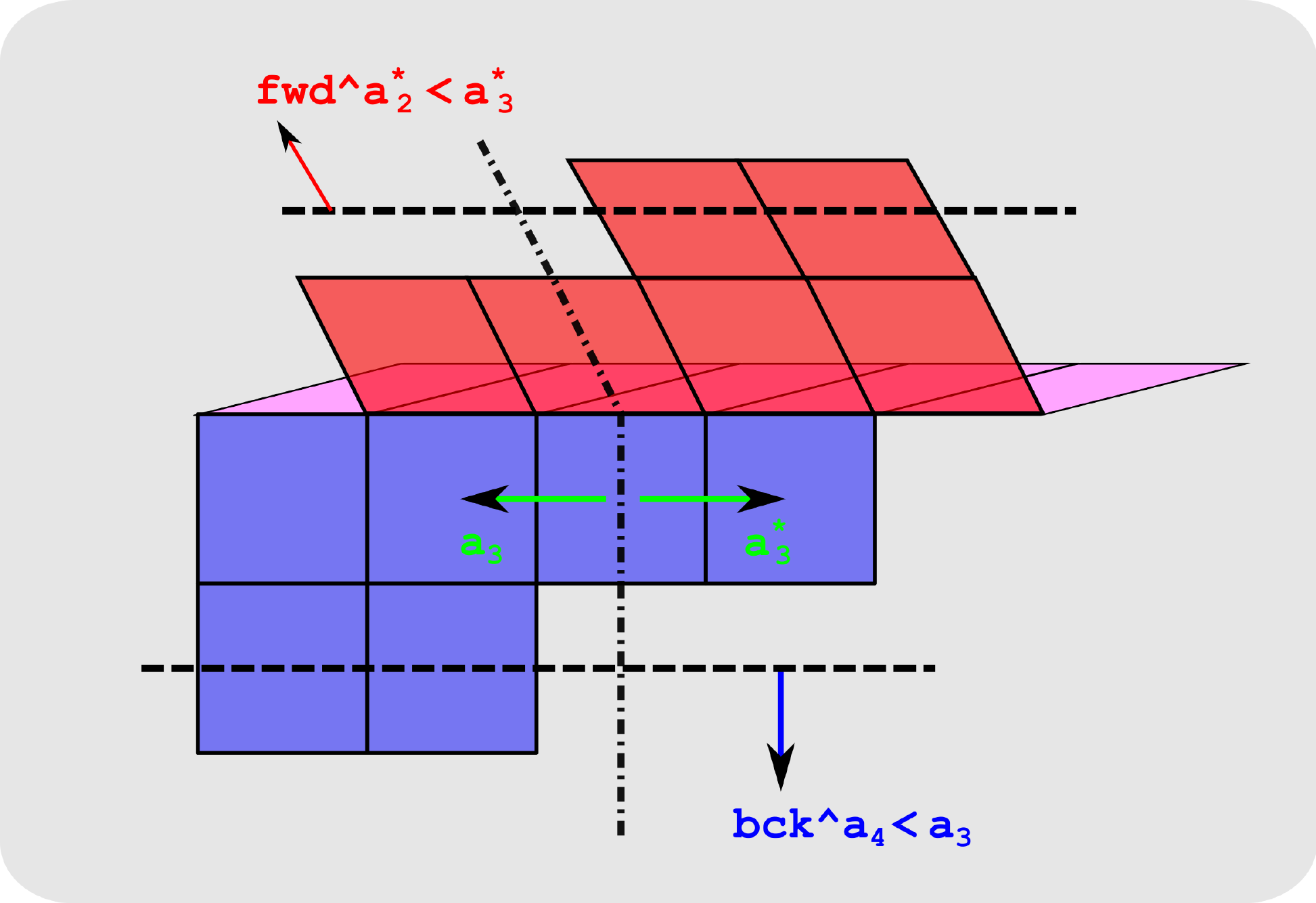}
		\caption{\small Model space for an agent on a discrete path, with two added contextual action sensors. \normalsize\label{fig:interval with motion, slight expansion}}
	\end{center}
\end{figure}

\medskip
\subsubsection{Example: discrete path with motion, revisited}\label{example:N-path with propagation} To illustrate the above, we continue example \ref{example:path with motion}. Recalling $\env=\{0,\ldots,L\}$ we see that the sensors $a_k$ defined in \eqref{eqn:sensors_on_interval} may be rewritten as:
\begin{equation}\label{eqn:sensors on interval redefined}
	a_k=\loc{U_k}\,,\quad
	U_k=\set{i\in\env}{0\leq i<k}
\end{equation}
Thus, for example, adjoining the two sensors $\fwd\wedge a_2^\ast$ and $\back\wedge a_4$ to $\sens$ implies the relations 
\begin{equation}
	\fwd\wedge a_2^\ast<a_3^\ast\,,\quad
	\back\wedge a_4<a_3\,,
\end{equation}
whose effect on $\cube{\ppoc}$, once they are learned by the agent, is shown in figure \ref{fig:interval with motion, slight expansion}.

\medskip
Further expanding $\sens$ to include all the sensors
\begin{equation}\label{eqn:contextual fwd and bck}
	\begin{array}{rl}
	\fwd\wedge a_k^\ast\,,& k=1,\ldots,L-1\\
	\back\wedge a_k\,,& k=2,\ldots,L
	\end{array}
\end{equation}
turns $\cube{\ppoc}$ into the complex illustrated in figure \ref{fig:interval with motion, full expansion}. As shown in the figure, the order structure on $\ppoc$ encodes both large-scale geometry (the agent may use propagation to conclude "in order to reach $\half{a_5^\ast}$, I need to to reach $\half{a_2^\ast}$"), and the actions required to negotiate this geometry ("I know that $\fwd\wedge a_1^\ast$ implies $a_2^\ast$, and I am currently in $\half{a_1^\ast}$").

\begin{figure}[ht]
	\begin{center}
		\includegraphics[width=\columnwidth]{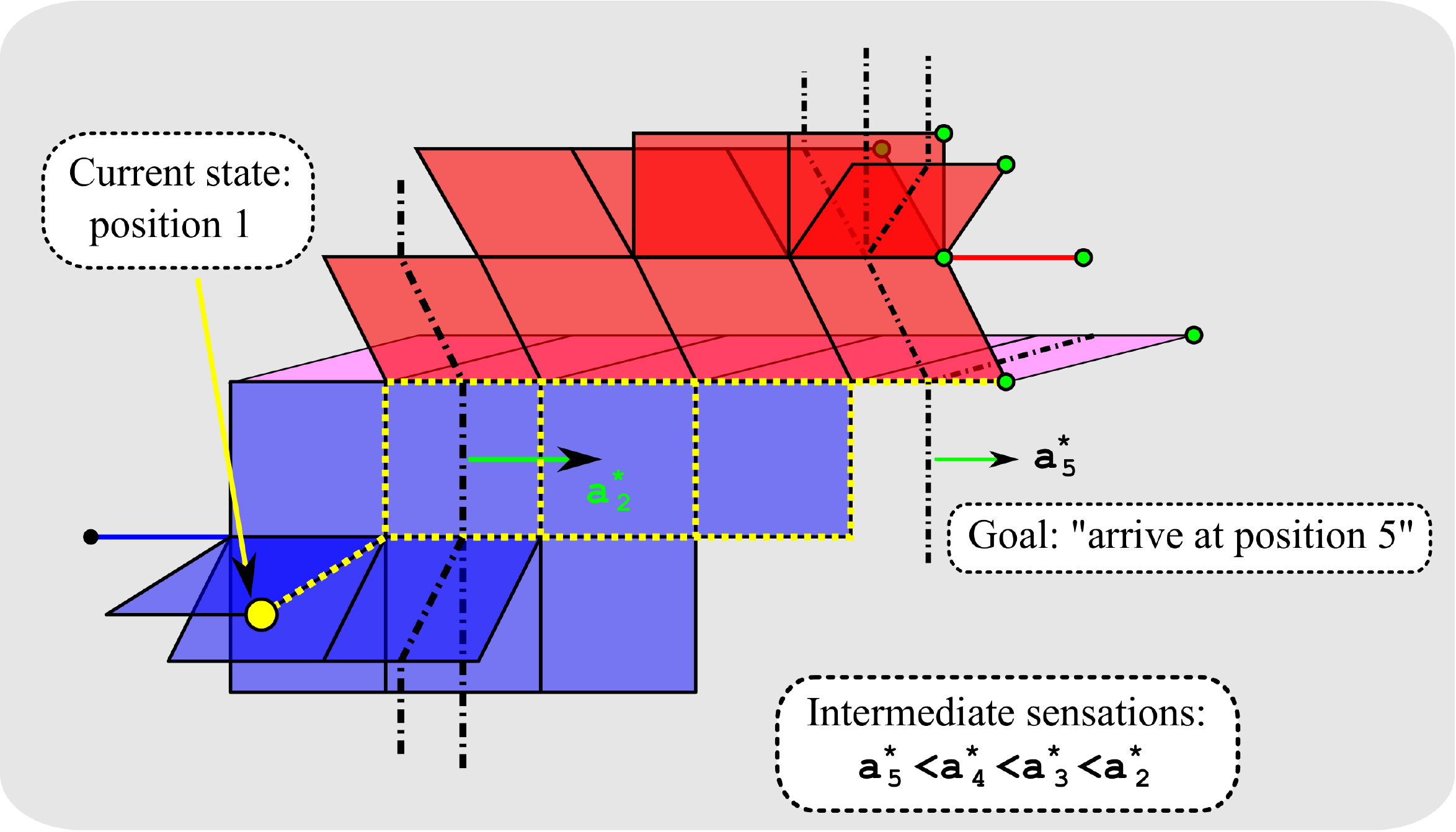}
		\caption{\small Model space for an agent on a discrete path, enriched with a full complement of contextualized action sensors \eqref{eqn:contextual fwd and bck}, and illustrating the geometry underlying planning by propagation \ref{general planning}.\normalsize\label{fig:interval with motion, full expansion}}
	\end{center}
\end{figure}

\medskip
\subsubsection{Algorithm: Greedy Reactive Planning (GRP) of Motion Towards a Specified Target}\label{general planning}
The ability to compute the immediate consequences of any available action and the convexity theory of $\model\at{t}$ underlie the following greedy algorithm used to decide on an action to be taken for the purpose of achieving a {\it long-term} goal:
\begin{enumerate}
	\item Given a set $T$ of target sensations, propagating $T$ over $[\pog\at{t},\current{t}]$ yields a list $R$ of sensations characterizing the projection of the region $\half{\current{t}}$ representing current state in $\model\at{t}$ to the desired region $\half{T}$.
	\item Each of the elements of $R$ may be considered as a sub-goal, and a generalized action guaranteed to achieve as many of these subgoals as possible may be selected based on the corollary \ref{consequences of an action}; any ties are broken arbitrarily. 
	\item Once an action is invoked, the same target $T$ is presented to the agent for an additional iteration of this procedure, until completion.
\end{enumerate}
By lemma \ref{lemma:propagation gives closest point projection}, this algorithm is directly analogous to motion planning in the Euclidean plane in the absence of obstacles: the agent selects an action which, to the best of its knowledge, best approximates the greedy path towards the closest point of the indicated target. The next section will consider the kinds of problems arising in the presence of obstacles in the model space and some early numerical study undertaken to explore overcoming some of these problems.

\subsection{Obstructions to Greedy Reactive Planning. } Where do obstructions to GRP in $\model\at{t}$ come from? Recall that every transition $x\times y\in X\times X$ capable of occurring in the given experiment determines a complete $\ast$-selection $A$ on $\sens$, by our observation model, via:
\begin{equation}\label{eqn:real vertices}
	a\in A\iff\left\{\begin{array}{rl}
		y\in\rho(a)	&\text{if }a\text{ is a state sensor}\\
		x\times y\in\rho(a)	&\text{if }a\text{ is a transition sensor}
	\end{array}\right.
\end{equation}
Thus, although $\model\at{t}$ does provide a universal model space for {\it any} realization of the weak poc set structure $\ppoc\at{t}$, the agent is only capable of witnessing $\ast$-selections of the above form, no matter the choice of action. This observation motivates the following definition:
\begin{defn}[Punctured Model Space]\label{defn:punctured model spc} By the {\it punctured model space at time $t$} we mean the sub-complex $\punct\at{t}$ of $\model\at{t}$ induced\footnote{Recall that a sub-complex $L$ of a cell complex $K$ is induced by a set of vertices $V\subset K^0$, if $L$ contains every cell of $K$ all of whose vertices belong to $V$.} by the set of vertices of $\model\at{t}$ of the form \eqref{eqn:real vertices} (compare with the discussion in appendix \ref{realization}).
\end{defn}

Thus, in addition to the possibility that an agent will have a false implication on record (causing some sensory equivalence classes to be deemed incoherent until they are sufficiently sampled), it is also possible that $\model\at{t}$ contains obstacles to GRP in the form of vertices in $\model\at{t}\minus\punct\at{t}\neq\varnothing$. In fact, the presence of obstacles of this kind is {\it guaranteed} by the main result of \cite{Allerton_2012}~ ---~ also reviewed in appendix \ref{homotopy type result}, Theorem \ref{thm:homotopy type}~ ---~ at least in cases when $\env$ does not have the homotopy type of a point and the covering $\hsm{U}$ of $\env$ by location place fields satisfies the richness requirements placed on it by that theorem. We consider two of examples of this kind.

\subsubsection{Example: A Punctured Grid}

Consider the example of an agent navigating an $N\times N$ square grid (realized as a subset $G_N$ of the integer grid) and equipped with a collection of position sensors identical to that of \ref{empirical:performance}(d) and \ref{discounted:performance}(d). Denoting $\pos(x):=\xi\times\eta\in\ZZ\times\ZZ$ we have the sensors
\begin{equation}
	\ev{\mathtt{x}_i}{x}=1\IFF\xi<i\,,\quad
	\ev{\mathtt{y}_i}{x}=1\IFF\eta<i\,,
\end{equation}
for $i\in\{1,\ldots,N-1\}$. This time, however, suppose that one interior vertex $v_0$ of the grid has been removed, so that $\env=G_N\minus\{v_0\}$. As in the above simulations, we assume the agent is equipped with actions labelled $\mathtt{up},\mathtt{down},\mathtt{left}$ and $\mathtt{right}$ whose effect at each vertex is to move to the appropriate adjacent vertex of the integer grid if that vertex belongs to $\env$, and to remain in place otherwise. Suppose, for simplicity, that the snapshot structure for this agent is empirical.

For $N$ sufficiently large, the statistical nature of the learning algorithm will cause the agent to learn the same weak poc set structure as in the case of $v_0$ being present: implications of the form 
\begin{equation}
	\mathtt{up}\wedge\mathtt{y}_i^\ast<\mathtt{y}_{i+1}^\ast\,,\quad
	\mathtt{up}\wedge\mathtt{x}_i<\mathtt{x_i}\,,\quad
	\mathtt{x}_i<\mathtt{x}_{i+1}
\end{equation}
and their respective variations will be learned upon sufficient exposure, giving rise to the {\it same poc set structure as the one representing the complete grid $G_N$}. One could view this as a manifestation of the fact that our model spaces are always contractible (corollary \ref{cor:cubings are contractible}).

As a consequence, the agent is bound to attempt moving to the unavailable vertex $v_0$ at any time $t$ when its position, $p^t$, is adjacent to $v_0$ and $v_0$ belongs to a shortest path in $G_N$ joining $p^t$ with the prescribed target. In such a situation, the agent is guaranteed to attempt motion in the direction of $v_0$ and fail (remain in place). Moreover, after sufficiently many such attempts the agent is bound to unlearn the implication responsible for this particular choice of action; this will have an overall negative impact on planning.

\subsubsection{Example: Agent on a Circular Rail}
A subtly different example is that of \ref{empirical:performance}(b) and \ref{discounted:performance}(b). Here, motion along a circular rail is modeled by setting $\env$ to be the set of integers modulo $N$, with two available actions $\fwd$ and $\back$ corresponding to the operations of adding and subtracting a unit, respectively (all arithmetic relating to the environment in this example is done modulo $N$). Position sensors have the form $\loc{U_i}$ where $U_i=\{i-1,i,i+1\}$.

For simplicity consider a situation with $N$ big and even, and assume the agent has complete knowledge of the correct poc set structure, which is the one generated by the relations (appendix \ref{example:generators and relations}):
\begin{equation}
		\loc{U_i}<\loc{U_j}^\ast\IFF\dist{i}{j}>2\,,
\end{equation}
where $\dist{i}{j}$ denotes the distance (modulo $N$) between the positions $i$ and $j$, as well as:
\begin{equation}
\begin{array}{rcl}
	\fwd\wedge\loc{U_i}&<&\loc{U_{i+1}}\,,\\
	\back\wedge\loc{U_i}&<&\loc{U_{i-1}}
\end{array}	
\end{equation}
for all $i\in\{0,\ldots,N-1\}$. Without loss of generality, the current state $x$ of the system satisfies $\pos(x)=0$.

Let the specified target be $T=\{U_p\}$, where $p\in\{0,\ldots,N-1\}$ is sufficiently far from the origin (the current position of the agent) to accommodate a pair $U_i,U_j$ such that:
\begin{enumerate}
	\item $U_i\cup U_j$ does not intersect $U_p\cup U_0$;
	\item $U_i\cup U_j$ separates $p$ from $0$ (on the circle).
\end{enumerate}
Thus both the current state and the target region satisfy the constraints $\loc{U_i}^\ast$ and $\loc{U_j}^\ast$, which implies that {\it any} geodesic in the model space joining the current model state with the model target set passes through $\half{\loc{U_i}^\ast,\loc{U_j}^\ast}$, yet it is clearly impossible to guarantee these constraints by {\it any} of the available actions.

It is, never the less, possible to extend this sensorium in a way that enables the effective learning of a target, as the numerical studies below demonstrate, by introducing into the environment a graded signal whose strength encodes a measure of distance to the target, while endowing the agent with sensors responding to the gradient of this signal.

\subsubsection{Closing the Loop: Excitation-Driven Navigation}\label{gradient descent}\label{subsection:closing the loop}

From the preceding examples it is clear that additional sensing~ ---~ most probably involving information regarding transitions~ ---~ is absolutely necessary for overcoming the geometric and topological obstructions to the GRP algorithm: while the GRP algorithm may be considered as providing a reasonable reference dynamics for reactive planning, one must consider possible means for {\it replanning in the face of failure}. We conjecture that the notion of a snapshot is sufficiently simple and agile for such purposes:
\begin{itemize}
	\item[-]{\bf Control of learning thresholds. } At this stage of our research, no attempt is being made to control the learning thresholds $\tau_{ab}$; it seems plausible, however, that having a high level of "frustration" cause the lowering of a relevant threshold may be used as a tool for locating exceptions to poc relations.
	\item[-]{\bf Introduction of new sensors. } A principled mechanism is required for the introduction of {\it combinations of existing sensors}, such as Boolean functions thereof, or delayed conjunctions such as \eqref{eqn:context}, to {\it become additional members of the sensorium}. In particular, such a mechanism must be capable of responding to exceptions, or failures of GRP, as we had already discussed above.
\end{itemize}
The need for self-adjustment in the sensorium opens the door to the introduction of auxiliary internal mechanisms of evaluating the position of the agent in the environment, or, more generally, the state of the system consisting of the pairing of the agent with the environment. For example, the settings described in the preceding paragraph suggests the introduction of an internal variable evaluating success (and failure) of invoking a planned action, while the need for closed-loop control suggests simple local control mechanisms based on an internally-defined `navigation function' \cite{Koditschek_Rimon-artificial_potential_functions}. Many other ideas of internal behavior modulation ranging from varying notions of novelty, surprise and dependability \cite{Barto_Mirolli_Baldassarre-novelty_or_surprise,Schmidhuber-theory_intrinsic_motivation,Cobo_Isbell_Thomasz-object_focused_Qlearning,Mugan_Kuipers-autonomous_learning_high_level_states} all the way to a multivariate model of human neuro-modulation mechanisms \cite{Cox_Krichmar-neuromodulation_as_controller} become relevant in this context.

\medskip
\noindent{\bf Excitation-Driven Agents. } In the absence of tools for reactive replanning (our current situation), we have chosen to study a simplified notion of target, allowing us to close the control loop with a "motion" command generated with the aim to guarantee an immediate decrease in the value of an internal excitation signal.

The simplest instance of such a controller, applied to the navigation setting, seems to be the following. In addition to a sensorium of the form described above in \ref{evaluating an action} and the examples that followed, assume that the agent possesses a pair of sensors $\better$ and $\worse$, responding to the decrease and increase, respectively, in a fixed measure of distance to a target point in the environment $\env$, over a single transition (think of this as a radically simplified sense of smell). 

Starting out as a lazy random-walking agent\footnote{An action $\wait$ is available now, to let the agent stay put when it has found the target.}, the agent uses the algorithm of \ref{general planning} at each step to obtain an action resulting with $\better$ (target specification $T=\{\better\}$) as its first priority. In the case of failure to produce such an action, the agent attempts to guarantee $\worse^\ast$, periodically attempting a random action so as not to get stuck in place (upon having figured out that $\wait<\worse^\ast$). Figure \ref{fig:sniffy} presents a comparison between the mean behaviors of four different such agents simulated in the same settings as those analyzed in section \ref{section:snapshots}.

\begin{figure}[t]
	\begin{center}
		\includegraphics[width=\columnwidth]{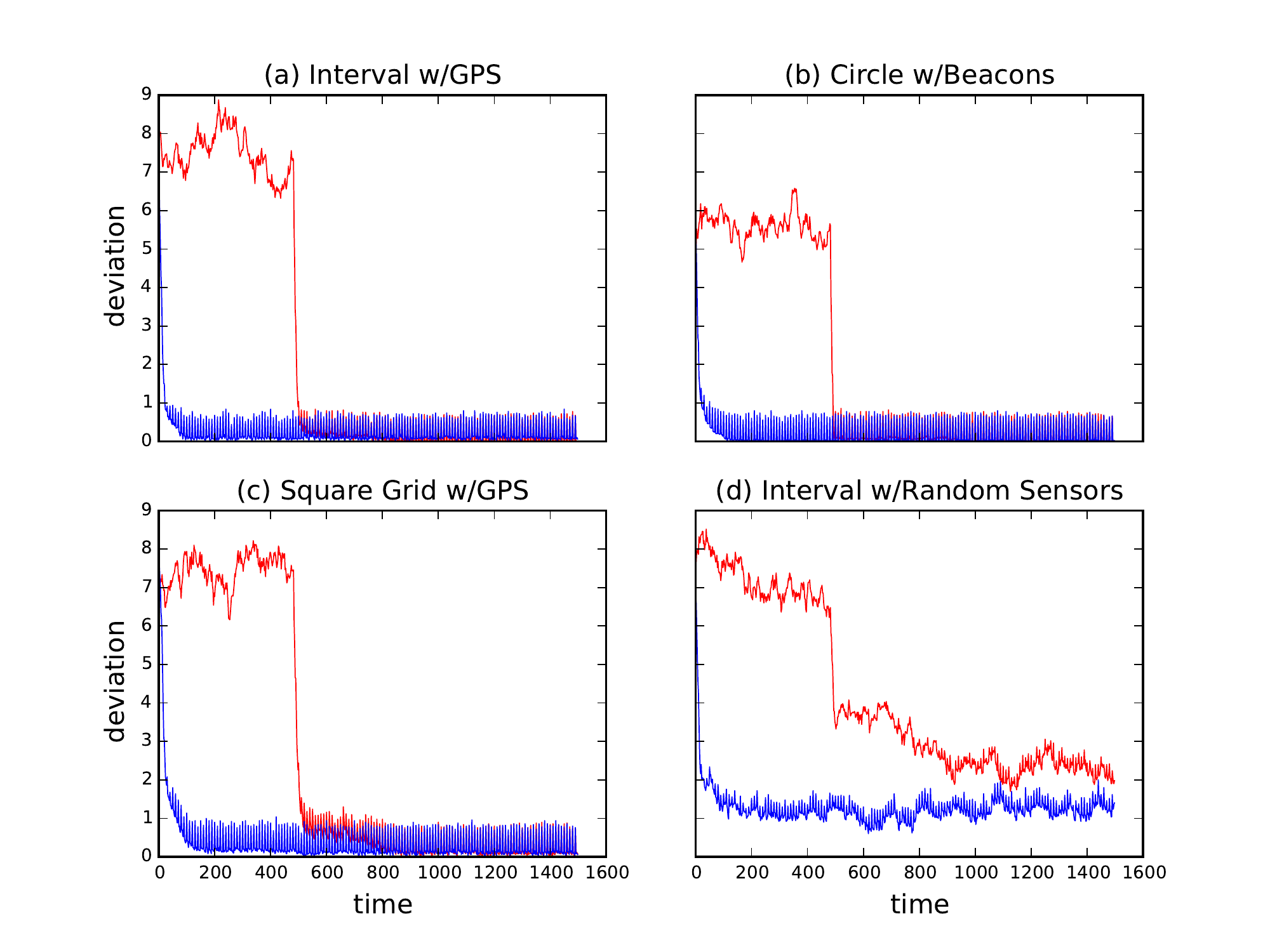}
		\caption{\small Mean deviation from target for empirical (blue) and a discounted (red) agents (20 sensors each), as a function of time in four different settings, averaging over 50 runs.\normalsize\label{fig:sniffy}}
	\end{center}
\end{figure}

It is important to stress that, following the discussion in \ref{example:path with motion} and the formal results of the preceding paragraphs, the guarantee of the agents in figure \ref{fig:sniffy}(a)-(c) finding their targets {\it given sufficient knowledge} of the correct poc structure on $\sens$ is absolute. To see this, it suffices to verify for the true poc set structure on $\sens$ that any position other than the target has associated with it a location sensor $a=\loc{U}$ and an action $\alpha$ such that every state $x$ with $\ev{a}{x}=1$ satisfies the requirement that $\alpha(x)$ is closer to the target than $x$ is. 

The only remaining question is whether or not the control strategy we propose guarantees sufficient exposure for the agent to recover the poc relations necessary for it to capture the target. Figure \ref{fig:sniffy} offers some evidence in favor of an affirmative answer, based on simulations.

\section{Discussion}\label{section:discussion}
\subsection{Topological Mapping and Planning}
Mapping methods vary from probabilistic representations \cite{Leonard_Durrant-Whyte-SLAM,Thrun-probabilistic_robotics,Montemerlo_Thrun-fast_SLAM} to hybrid representations in which precisely mapped contractible local metric patches are integrated into a global map through `gluing instructions' recorded in an annotated graph that is learned by the agent in the course of its travels (\cite{Tomatis_et_al-hybrid_SLAM,Kuipers-topological_SLAM,Dellaert-probabilistic_topological_mapping,Choset-bayesian_topological_SLAM}). The latter are known as {\it topological SLAM} (simultaneous localization and mapping) methods. 

In the presence of significant obstacles in large scale environments, the lightweight `topological skelteon' of the environment recorded by a topological SLAM method provides valuable information on loop closure, which would otherwise be costly to obtain using global metric mapping techniques. This observation led to an extensive effort \cite{Kuipers-SSH,Kuipers_Remolina-topological_maps} to produce a general notion of a {\it topological map}, based on the notion of {\it spacial semantic hierarchy} (SSH), that would allow for motion planning at varying scales. Other ways of leveraging planar topological descriptors to obtain simple and efficient data structures supporting localization and motion planning in simply connected planar domains have been explored as well, e.g.  \cite{Tovar_et_al-BITBOTS,Tovar_et_al-Information_spaces_BITBOTS} and \cite{Mair_et_al-bioinspired_SLAM}.

In view of our stated need for guaranteed universality properties, the strong dependence of SLAM methods on Euclidean geometry made it necessary for us to adopt a significantly more abstract approach to localization arising from the point of view of sensor fields. A precursor of our approach is a family of SLAM algorithms, known as RatSLAM, employing a neural network to simulate the function of place/pose cells in rats, e.g.  \cite{Milford_Wyeth_Prasser-RatSLAM,Milford_Wyeth-persistent_navigation_with_RatSLAM,Milford_Jacobson-brain_inspired_sensor_fusion_for_SLAM}. In a broad generalization of the neural computational engine underlying RatSLAM, \cite{Curto_Itskov-cell_groups} tests the hypothesis that place cells with sensor fields in {\it any} sufficiently dense configuration should make it possible [for a rat] not only to localize well, but also to accurately represent the topology of the environment by estimating the nerve\footnote{See {\it nerve of a covering} in \cite{Hatcher-alg_top_textbook}, section 3.3.} of the system of place fields from observations of near-synchrony in place-cell firing. It is shown through simulations that recovering the topological invariants of a connected planar arena, as well as some approximation of its geometry, is possible with a sufficiently dense network of convex place fields. 
Further evidence in support of this idea is the recently introduced \cite{Ghrist-homological_SLAM} method for localization in an urban canyon, as well as the ELM architecture \cite{Stachowicz_Kruijff-episodic_like_memory}, using nerve-like information (nesting among convex polygons in the plane) as a means of encoding spatio-temporal context in an `episodic memory' for a planar agent.

A significant drawback of nerve estimation is that, by definition, computing nerve of a covering requires exponential space in the cardinality of the covering\footnote{As one must keep track of the intersections of all possible sub-collections of the covering}. At the same time, restricting attention to pairwise intersections only, and focusing on those of them that are empty turns out to guarantee a universal model space for each combinatorial type of this "reduced nerve", by Sageev-Roller duality \cite{Sageev-thesis,Roller-duality}. By construction, snapshots are estimators of this reduced nerve, turned into a computational tool for navigating the corresponding model space after converting all relations of the form "$A\cap B$ negligible/unimportant" into relations of the form "$A$ implies $B$".

The idea of leveraging nesting relations among geometric descriptors of events for the purpose of properly representing context is a well-recognized and widely used tool in the literature, for example: nesting of planar domains is used in \cite{Kuipers_Remolina-topological_maps}, as well as the more recent \cite{Stachowicz_Kruijff-episodic_like_memory}, and a notion of nesting for actions is used in \cite{Mugan_Kuipers-autonomous_learning_high_level_states}. What is new about the snapshot architecture is its application of this principle to the {\it entire} sensorium, including the set of available actions.

Our numerical experiments with closed loop control suggest that a snapshot-driven agent with sufficient actuation and sufficiently rich sensorium is capable of learning a good approximation of the gradient field of a (discrete) Morse function despite been given no prior semantic information and starting out with random `motor babble' for its navigation strategy.

In fact, the snapshot architecture is flexible enough to trivialize the task of introducing discrete variants of complex motivational systems \cite{Cox_Krichmar-neuromodulation_as_controller} based on introspective sensing of signals quantifying (a) internally available resources (e.g. battery charge), (b) repulsion or fear of punishment, (c) attraction or anticipation of a reward signal (e.g., in the sense of navigation functions \cite{Koditschek_Rimon-artificial_potential_functions} or in the broad sense of Reinforcement Learning \cite{Barto_Sutton-reinforcement_learning}), and even (d) frustration over the failure of a plan \cite{Revzen_Ilhan_Koditschek-frustration} and (e) measures of innovation \cite{Barto_Mirolli_Baldassarre-novelty_or_surprise,Schmidhuber-theory_intrinsic_motivation}.

We expect systems such as (a)-(c) to contribute to an agent's planning capabilities from the point of view of the variety of tasks they would enable. Even more significantly, we expect (d) and (e) to contribute to the agent's ability to improve the quality of the topological representation encoded in its snaphot. Namely, (d) could be used to facilitate chunking by serving as a signal driving the creation of new conditional sensors detecting inconsistent states of the model space, while (e) could drive the learning of useful complex actions, as has already been proposed for many other architectures \cite{Schmidhuber-Godel_machines_optimality,Mugan_Kuipers-autonomous_learning_high_level_states,Cobo_Isbell_Thomasz-object_focused_Qlearning,Martius_et_al-infotaxis_driven_self_organization}, resulting in improving the connectivity of the model space. Endowing the snapshot architecture with these capabilities is the most immediate goal for follow-up research. 

\subsection{Connectionist Architectures}\label{discussion:connectionist} 
From the earliest days of the field, even extremely simple neural networks with a very small number of neurons have demonstrated the ability to perform complex learning and control tasks \cite{Barto_Sutton-adaptive_network_with_internal_rep,Barto_Sutton_Anderson-neuronlike_solve_difficult_learning_probs}, including complex symbolic structures such as context-sensitive grammars \cite{Elman-finding_structure_in_time,CMU_team-encoding_structure_in_recurrent_networks,Elman-recurrent_networks_and_grammatical_structure}, and complex hierarchical schemata whose structural features vastly outnumber the physical components the actual network \cite{Botvinick_Plaut-without_schema_hierarchies}. Even simple feed-forward networks have been shown to be capable of exercising `intelligent' control through stigmergy \cite{Chung_Choe-memory_emergence_from_markers}, the depositing of tokens in the environment. Furthermore, the merging of reinforcement learning methods with the computational power of `deep learning' networks \cite{Hinton_Salakhutdinov-reducing_dimensionality,Lee_Grosse_Ranganath_Ng-convolutional_deep_belief_networks} has yielded architectures capable of matching and outperforming humans in complex tasks such as learning and playing video games \cite{Mnih_et_al-ATARI_deep_RL} based only on the raw RGB output provided by the game console.

At the same time, the power of connectionist architectures comes with very limited formal understanding of how the internal representations they maintain encode symbolic representations in terms of problem spaces.

Though the connection demonstrated in this work falls short solving this problem (due to the over-simplification of the connectionist structure and due to realization of plasticity in the network by a non-neural controller), direct analogies with current studies of plasticity and neural coding give hope that a suitable generalization of the snapshot architecture could yield {\it both} strength of performance {\it and} provable guarantees of symbolic reasoning processes.

Of the classes of neural networks that are well understood, most relevant for our purposes is that of competitive attractor networks, whose strong stability properties \cite{Cohen_Grossberg-absolute_stability_competitive} were applied to modeling the navigation mechanisms in rats, also inspiring the RatSLAM mapper \cite{Milford_Wyeth_Prasser-RatSLAM}.

Expanding on these results, \cite{Xie_Hahnloser_Seung-selectively_grouping_neurons,Hahnloser_Seung_Slotine-permitted_and_forbidden_sets} sparked the discussion of the structure of the set of binary codewords corresponding to stable activity patterns of threshold-linear neural networks. This line of inquiry was picked up in \cite{Curto_Degratu_Itskov-encoding}, producing a combinatorial characterization of the possible codes in terms of the network's organization; and in \cite{Curto_Degratu_Itskov-flexible_memory_networks}, initiating a rigorous study of the way codes vary as the synaptic weights are perturbed while subject to structural constraints, exposing interesting connections with topological invariants associated with these constraints.

The analogy with our work is straightforward. The process of obtaining the coherent projection of a binary observation by propagating it through the derived graph of a snapshot is completely analogous to the process of propagating a signal through a threshold-linear neural network and waiting for the network state to stabilize at a code word~ ---~ especially when taking into account the excitatory nature of relations of the form $a\leq b$ and inhibitory role of relations of the form $a\leq b^\ast$ under propagation.

Chasing this analogy, it could be worthwhile investigating the degree to which the collection of codewords of a threshold-linear network conforms to the strict demands of median geometry (represented by coherent snapshot states), to establish a rigorous formal connection (if it exists) between the two architectures. A more general study of which neural learning methods \cite{Hinton-connectionist_learning_procedures} could be transferred into a snapshot architecture may, on one hand, expand the range of applications for snapshots, while, on the other hand, provide some existing architectures with a rigorous symbolic interpretation.

\section{Conclusion}
We introduce a new computationally efficient architecture
intended to endow a generic discrete binary agent with the capacity to build over time an actionable model of its operations within a completely unknown  dynamic environment, $\env$. The proposed architecture has a dual nature. On one hand, the agent maintains an evolving data structure,~ ---~ the snapshot $\snap{S}\at{t}$~ ---~ of size quadratic in the number of sensors, encoding a planning mechanism based on propagation of excitation and inhibition signals through the highly plastic directed network $\snapdir{\snap{S}\at{t}}$, and is, thus, in a very crude sense, a connectionist learning and control architecture. On the other hand, the rather specific ordering properties of networks arising in this way (the derived `weak poc set structure' $\ppoc\at{t}$) also characterize any such network as encoding a system of `half-spaces' in a geometric internal representation $\model\at{t}$ that is {\it just rich enough} to account for all sensory equivalence classes provided to the agent by its sensorium $\sens$. This duality affords the re-interpretation of snapshots as encoding a high-level symbolic representation of the problem space (i.e., the state space $\spc$ and the transition system induced on it by the agentÕs interaction with its dynamic environment), through a mathematical formalism that rigorously supports symbolic planning with the efficiency and economy of a connectionist architecture.

Clearly, our current snapshot architectures (section II) lack certain key features found in existing AGI architectures. First and foremost among these is a mechanism for enriching the agent's sensorium with sensors representing general Boolean predicates (or, even better,  some limited LTL predicates), composed of the original atomic sensations. Of course, the problem lies not in proposing intuitively attractive approaches (there are many) but rather doing so in a principled, economical way that maintains the present combination of analytical and computational tractability. These `compound' sensors are required for facilitating chunking and the learning of useful motor primitives. Another required feature is a capacity for symbolic abstraction (relating problem spaces via substitutions). While the duality theory of weak poc sets and their model spaces (appendix \ref{subsection:duality theory}) affords a rigorous formulation of enriched predicates and consequent symbolic abstraction, it is not yet clear how to engineer an enlarged snapshot-like architecture realizing such meta-extensions.

\section*{Acknowledgements}
This work was funded in part by the Air Force Office of Science Research under the MURI FA9550-10-1-0567 and in part by the National Science Foundation under CDI-II-1028237.

\bibliographystyle{IEEEtran}
\bibliography{IEEEabrv,memory}

\appendices
\section{A Primer on Sageev-Roller Duality}\label{appendix:prelim}
The duality between poc sets and median algebras, going back to \cite{Isbell-duality}, was thoroughly studied by Martin Roller in \cite{Roller-duality}, in a very successful attempt of pushing the envelope on a theory of actions of discrete groups on simply connected non-positively curved cubical complexes -- henceforth reffered to as \emph{cubings} -- pioneered by Michah Sageev in \cite{Sageev-thesis} and by Victor Chepoi \cite{Chepoi-median}, who characterized such complexes in terms of the convexity theory of their $1$-dimensional skeleta. 

This appendix provides a detailed review of the elements of this theory supporting the memory architecture proposed in this paper. This overview of the preliminary meterial is meant to extend the initial discussion provided in \cite{Allerton_2012} as well as to illustrate it with examples, intended as bridges to our current application. In the end, the duality theory of poc sets will be called upon to provide the necessary formal guarantees that the proposed memory and control architectures actually do their job. We will mainly rely on \cite{Roller-duality} as a source of theoretical results, though sometimes it will be easier to use results from the elegant exposition in \cite{Nica-spaces_with_walls}.

\subsection{Basic Notions} We introduced the extension of the duality theory of poc sets to so-called {\it weak poc sets} in \cite{Allerton_2012} out of the necessity to maintain poc sets as dynamical data structures.
\begin{defn}[Weak Poc Set]\label{defn:poc set} A partially-ordered set $(P,\leq)$ endowed with an order-reversing fixpoint-free involution $a\mapsto a^\ast$ and having a minimum element $\minP\in P$ is called a \emph{weak poc set}. Note that $\maxP$ is a maximum for $P$. Thus, for all $a,b\in P$ one has:
\begin{itemize}
	\item $0\leq a=a^{\ast\ast}$ and $a^\ast\neq a$;
	\item $a\leq b\THEN b^\ast\leq a^\ast$.
\end{itemize}
An element $a\in P$ is said to be \emph{negligible} if $a\leq a^\ast$, and \emph{ubiquitous} if $a^\ast$ is negligible. A \emph{poc set} is a weak poc set in which $\minP$ is the only negligible element. An element that is neither negligible nor ubiquitous is said to be \emph{proper}.\defstop
\end{defn}
Weak poc sets form a category:
\begin{defn}[Poc Morphism]\label{defn:poc morphism}
A function $f:P\to Q$ between two weak poc sets is a \emph{poc morphism}, if $f$ is order-preserving, $\ast$-equivariant and $f(\minP)=\minP$. The set of all poc morphisms as above will be denoted $\morph{P}{Q}$.\defstop
\end{defn}

\medskip
\subsubsection{The Minimal Poc Set}\label{example:minimal poc set} The set $\{0,1\}$ with the relations $0<1$ and $1=0^\ast$ is a poc set, and it is denoted by $\power{}$. Clearly, there is only one poc morphism of $\mbf{2}$ into any weak poc set $P$, but then there may be many poc morphisms of a weak poc set $P$ onto $\mbf{2}$. 

\medskip
\subsubsection{Generators and Relations}\label{example:generators and relations} A weak poc set $P=\gen{S}{R}$ may be specified using a set $S$ of generators and a set of relations $R$ of the form $a<b$ or $a^\ast<b$ or $a<b^\ast$ for $a,b\in S$\footnote{One may also use weak inequalities $(\leq)$ to specify relations in $R$.}. 

Formally, $P$ is constructed as follows. Assume that the symbol $\minP$ is not contained in $S$. First, set $S_\pm:=(\{\minP\}\sqcup S)\times\{+,-\}$ and define $(s,+)^\ast=(s,-)$ and $(s,-)^\ast=(s,+)$. Thus, $S_\pm$ has a fix-point free involution $\ast$ defined on it. For simplicity, for each $s\in\{0\}\cup S$ we identify $(s,+)$ with $s$. 

The relation set $R$ is required to be a subset of $S_\pm\times S_\pm$. We then define an extension $R_{_{poc}}$ of $R$ to be the intersection of all relations $W\subseteq S_\pm\times S_\pm$ that are reflexive, transitive and, in addition, satisfy the following:
\begin{itemize}
	\item $(\minP,a)\in W$ holds for all $a\in S_\pm$;
	\item For all $a,b\in S_\pm$, if $(a,b)\in W$ then $(b^\ast,a^\ast)\in W$.
\end{itemize}

We set $P$ to be the quotient of $S_\pm$ by the equivalence $$x\sim y\IFF (x,y)\in R_{_{poc}}\;\wedge\;(y,x)\in R_{_{poc}}$$
with the induced partial ordering $$[x]\leq[y]\IFF (x,y)\in R_{_{poc}}\,.$$

For example, the notation $$\gen{a,b,c}{a<c,\,b<c}$$ stands for the poc set $$P=\{\minP,\maxP,a,b,c,a^\ast,b^\ast,c^\ast\}$$ having the order relations
\begin{eqnarray*}
	\minP<a<c<\maxP,&& \minP<b<c<\maxP\\
	\minP<c^\ast<a^\ast<\maxP,&& \minP<c^\ast<b^\ast<\maxP
\end{eqnarray*}
as well as the ones derived from these by transitivity. Thus, generators and relations provide a compact way of representing a (weak) poc set explicitly.

As another example, consider the poc sets $$P=\gen{a,b}{a<b}\,,\quad Q=\gen{a,b}{a^\ast<b}$$ The partial assignment $f:P\to Q$ satisfying $$f(a)=a^\ast\,,\quad f(b)=b$$ has one and only one extension to a poc morphism of $P$ into $Q$.

\medskip
\subsubsection{$\sigma$-Algebras as poc sets}\label{example:power set as poc set} Let $\hsm{B}$ be a $\sigma$-algebra on a non-empty (possibly infinite) set $\spc$. Then $(\hsm{B},\subseteq,F\mapsto \spc\minus F)$ is a poc set. In particular, the power set of $\spc$, denoted $\power{\spc}$, obtains the structure of a poc set in this way. It is standard to identify $\power{\spc}$ with the space of functions $f:\spc\to\power{}$: any such $f$ will be identified with the subset $f\inv(1)\in\power{\spc}$. Note that the intersection  and symmetric difference operators translate under this identification into pointwise multiplication and addition modulo $2$, respectively. Recalling our notation \eqref{eqn:evaluations2} for the evaluation of functions, it will be convenient to extend it as follows:
\begin{equation}\label{eqn:evaluation}
	\ev{f}{x}=f(x)\,,\quad \ev{f^\ast}{x}=1+\ev{f}{x}
\end{equation}
The poc set structure on $\power{\spc}$ may then be written in functional form via
\begin{equation}
	f\leq g\;\;in\;\;\power{\spc}\IFF \forall_{x\in\spc} \ev{f}{x}\leq\ev{g}{x}\;\;in\;\;\power{}\IFF fg=f\,,
\end{equation}
where $f,g\in\power{\spc}$ are arbitrary elements.

\begin{defn}[realization] Let $P$ be a weak poc set and let $\spc$ be a non-empty set. A \emph{realization of $P$ in $\spc$} is a poc morphism $f:P\to\power{\spc}$.\defstop
\end{defn}

A realization $f:P\to\power{\spc}$ provides a consistent way of regarding each $a\in P$ as a binary question over $\spc$, so that the set of all $x\in\spc$ with $\ev{f(a)}{x}=1$ is the set of all points where the question is answered affirmatively.

\medskip
\subsubsection{Canonical Quotient} Every weak poc set has a canonical poc set quotient $\hat P$ obtained as the quotient of $P$ by the equivalence relation
\begin{equation}
	a\sim b\IFF\left\{\begin{array}{l}
		a=b\text{ or }\\
		a,b\text{ are both negligible, or }\\
		a,b\text{ are both ubiquitous}
	\end{array}\right.
\end{equation}
and endowed with the obvious ordering and involution.
\begin{defn} Let $P$ be a weak poc set and let $\hat{P}$ denote its canonical poc quotient. For every $a\in P$, we denote the equivalence class of $a$ in $\hat{P}$ with $\hat{a}$. The poc morphism $a\mapsto\hat{a}$ will be denoted by $\pi_P$.\defstop
\end{defn}
The main characteristic of $\hat P$ is the following elementary lemma:
\begin{lemma}\label{lemma:canonical poc quotient} Let $P$ be a weak poc set. Then any poc morphism $f:P\to Q$ of $P$ into a poc set $Q$ factors through $\pi_P$, that is: there exists one and only one poc morphism $\hat{f}:\hat{P}\to Q$ satifying $f=\hat{f}\circ\pi_P$.
\end{lemma}
For example, seeking a realization of a weak poc set structure in any space $\spc$ is only possible after identifying all negligible elements with $\minP$, because the only subset $F$ of $\spc$ satisfying $F\subseteq \spc\minus F$ is $F=\varnothing$.

\subsection{The Dual Graph of a Poc Set} The duality theory for poc sets is an extension of Stone Duality \cite{Stone-duality}. At the base of the construction are binary selections:
\begin{defn}[$\ast$-selections]\label{defn:selection} Let $P$ be a weak poc set. A subset $A\subset P$ is a \emph{$\ast$-selection on $P$} if no $a\in A$ has $a^\ast\in A$. A $\ast$-selection $A$ on $P$ is \emph{complete} if for any $a\in P$ either $a\in A$ or $a^\ast\in A$. The set of all complete $\ast$-selections on $P$ is denoted by $S(P)^0$.\defstop
\end{defn}
The following is a metric on $S(P)^0$:
\begin{equation}\label{eqn:ellone metric}
	\ellone{A}{B}=\left|A\minus B\right|=\left|B\minus A\right|=\tfrac{1}{2}\left|A\vartriangle B\right|
\end{equation}
Indeed, fixing $A_0\in S(P)^0$, an explicit isometry of $\left(S(P)^0,\mbf{\Delta}\right)$ onto $\power{A_0}$ endowed with the Hamming distance is constructed by sending $A\in S(P)^0$ to the [indicator function of the] set $A_0\minus A$. Thus, $S(P)^0$ may be thought of as simply being the vertex set, or $0$-skeleton, of the $\frac{|P|}{2}$-dimensional standard unit cube, viewed as a combinatorial cubical complex, -- we denote this complex by $S(P)$ -- where a cubical face $Q$ of $S(P)$ of dimension $d$ corresponds to a maximal subset of $S(P)^0$ with diameter $d$ as calculated in the metric $\ellone{\cdot}{\cdot}$.

\medskip
\subsubsection{Construction of the Dual Graph}\label{constructing the dual} Some vertices of $S(P)$ cannot be witnessed in any realization of $P$:
\begin{defn}[Coherence]\label{defn:coherence} A pair of elements $a,b\in P$ is said to be \emph{incoherent} if $a\leq b^\ast$. A subset $A$ of a poc set $P$ is said to be \emph{coherent} if it contains no incoherent pair.\defstop
\end{defn}
\begin{defn}[Dual graph, dual Cubing]\label{defn:dual} Given a (finite) poc set $P$, one defines:
\begin{itemize}
	\item[(a)]{\bf The dual cubing of $P$}, denoted $\cube{P}$, is the sub-complex of $S(P)$ induced by the set of coherent vertices;
	\item[(b)]{\bf The dual median algebra of $P$}, denoted $P^\circ$, is the $0$-skeleton of $\cube{P}$;
	\item[(c)]{\bf The dual graph of $P$}, denoted $\dual{P}$, is the $1$-skeleton of $\cube{P}$.\defstop
\end{itemize}
\end{defn}

\begin{figure}[t]
	\begin{center}
		\includegraphics[width=\columnwidth]{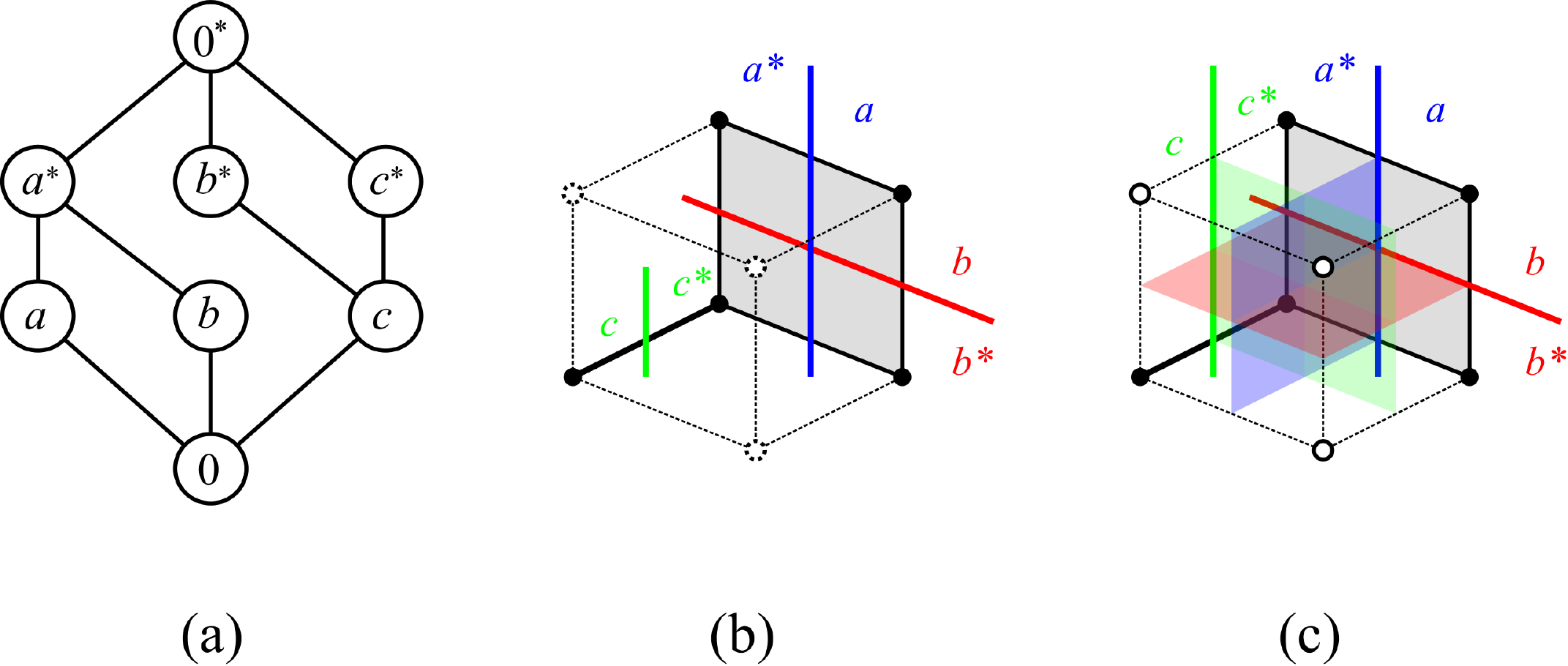}
		\caption{A simple poc set $P$ on three generators (a) and the resulting cube complex (b), obtained by deleting all incoherent vertices from the cube $S(P)$ (c) -- see example \ref{example:simple example of a dual}.\label{fig:duality for a simple relation}}
	\end{center}
\end{figure}

\medskip
\subsubsection{A Simple Example}\label{example:simple example of a dual}
To illustrate the definition, consider the poc set $P$ whose Hasse diagram is given in figure \ref{fig:duality for a simple relation}(a). Given by generations and relations, $P$ takes the form:
\begin{equation}
	P=\gen{a,b,c}{a<c^\ast,b<c^\ast}
\end{equation}
A good way of thinking about a poc set is to pretend that it is realized in a space $\spc$, so that our $P$ is a collection of three binary questions ($a,b$ and $c$) about $\spc$, augmented with the complementary questions ($a^\ast,b^\ast$ and $c^\ast$), together with an additional record of known implication relations between them ($a<c^\ast$ and $b<c^\ast$). In the absence of any implications on record, an observer endowed with $P$ will model the space $\spc$ as the full $3$-cube $S(P)$, where the proper elements of $P$ correspond to the co-dimension one faces of the cube -- fig. \ref{fig:duality for a simple relation}(c). On the other hand, knowledge of the above relations renders some of the vertices of $S(P)$ redundant, resulting in a reduction in the number of binary states necessary for modeling $\spc$ using the same three questions -- fig. \ref{fig:duality for a simple relation}(b).

\begin{remark} We have chosen the term {\it coherent subset} over Roller's {\it filter-base} to better fit the context of our planning/sensing problem. 
\end{remark}

\begin{remark} The standard identification of $\power{P}$ with the space of $\{0,1\}$-valued functions on $P$ also puts the set $P^\circ$ of coherent vertices of $S(P)$ in one-to-one correspondence with the set $\morph{P}{\power{}}$ of poc morphisms of $P$ onto the trivial poc set $\power{}$.
\end{remark}

\subsection{Poc Set Duals are Median Graphs}\label{median graphs} Graphs of the form $\dual{P}$ (for a weak poc set $P$) are completely characterized. We recall:
\begin{defn}[hop-distance, intervals]\label{defn:hop distance} Let $G=(V,E)$ be a connected simple graph and let $u,v\in V$. The \emph{hop distance} $d_G(u,v)$ is defined to be the minimum length of an edge-path in $G$ joining $u$ with $v$. The interval $I(u,v)$ is defined to be the set of all vertices $w\in V$ satisfying the equality $d_G(u,v)=d_G(u,w)+d_G(w,v)$.\defstop
\end{defn}
A fundamental fact about the dual $\dual{P}$ of a poc set $P$ is a quick corollary of the results in \cite{Nica-spaces_with_walls}, section 4:
\begin{lemma}\label{lemma:ellone metric is hop metric} Let $G=\dual{P}$ for a finite poc set $P$. Then the metric $\mbf{\Delta}$ coincides with the hop metric on $G$.
\end{lemma}
An important and well-studied class of graphs is:
\begin{defn}[median graphs \cite{Chepoi-median,median1}] A connected simple graph $G=(V,E)$ is said to be a \emph{median graph}, if the set $I(u,v)\cap I(v,w)\cap I(u,w)$ contains exactly one vertex for each $u,v,w\in V$. This vertex is the \emph{median} of the triple $(u,v,w)$ and denoted by $\med{u}{v}{w}$ -- see figure \ref{fig:median}.\defstop
\end{defn}
Median graphs are a special subfamily of a family of ternary algebras, called {\it median algebras}, \cite{median4,median5,median3,median2}. Some modern generalizations and applications may be found in \cite{CDH-median_spaces_measured_walls}.

\begin{figure}[t]
	\begin{center}
		\includegraphics[width=.6\columnwidth]{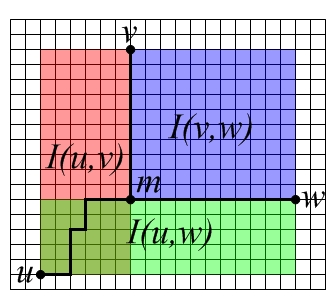}
	\end{center}
	\caption{Computing a median in the integer grid (see example \ref{example:direct sum} for a poc-set presentation).\label{fig:median}}
\end{figure}

Another way of stating the preceding lemma (again, see \cite{Nica-spaces_with_walls}, section 4, where these results are derived in a much more general setting):
\begin{thm} The dual $G=\dual{P}$ of a finite poc set $P$ is a finite median graph, with the median calculated according to the formula
\begin{equation}\label{eqn:median formula}
	\med{u}{v}{w}=(u\cap v)\cup(u\cap w)\cup(v\cap w)
\end{equation}
for all $u,v,w\in P^\circ$.
\end{thm}
In other words, the median of three coherent $\ast$-selections is determined by a majority vote on the values of their observations. 

\subsection{Convexity}\label{subsection:convexity} Median graphs have a very strong convexity theory. We recall:
\begin{defn}[Convexity] Let $G=(V,E)$ be a graph. A subset $K\subseteq V$ is said to be \emph{convex} if $I(u,v)\subseteq K$ holds whenever $u,v\in K$.\defstop
\end{defn}
\begin{defn}[Half-Spaces] Let $G=(V,E)$ be a graph. A subset $H\subseteq V$ is said to be a half-space of $G$, if both $H$ and $V\minus H$ are convex subsets of $G$.\defstop
\end{defn}
For example, lemma \ref{lemma:ellone metric is hop metric} can be used to quickly deduce
\begin{lemma}\label{lemma:ellone halfspaces} Let $P$ be a poc set. Then the half-spaces of $\dual{P}$ are precisely the subsets of $P^\circ$ of the form
\begin{equation}\label{eqn:ellone halfspaces}
	\half{a}:=\set{u\in P^\circ}{a\in u}
\end{equation}
where $a$ ranges over $P$. In particular, subsets of $P^\circ$ of the form
\begin{equation}\label{eqn:typical convex set}
	\half{K}:=\set{u\in P^\circ}{K\subseteq u}=\bigcap_{a\in K}\half{a}
\end{equation}
are convex in $\dual{P}$.
\end{lemma}
\begin{remark} Note also that $\half{a^\ast}=P^\circ\minus\half{a}$ for all $a\in P$.
\end{remark}
Much more can be said in general:
\begin{thm}[Properties of Median Graphs, \cite{Roller-duality}, section 2]\label{thm:convexity theory of median graphs}
Let $G=(V,E)$ be a finite median graph. Then:
\begin{enumerate}
	\item Every convex set is an intersection of halfspaces;
	\item Any family of pairwise-intersecting convex sets has a common vertex ({\it 1-dimensional Helly property});
	\item For any convex subset $K\subset V$, the subgraph of $G$ induced by $K$ is a median graph;
	\item For any convex subset $K\subset V$ and any vertex $v\notin K$ there exists a unique vertex $\proj{K}{v}\in K$ at minimum hop distance from $v$.
	\item For any convex subset $K\subset V$, the closest-point projection $\proj{K}{\wild}$ is a median-preserving, distance non-increasing map of $G$ onto the subgraph of $G$ induced by $K$.
\end{enumerate}
\end{thm}

Any graph $G=(V,E)$ generates a poc set $\poc{G}$: we let the underlying set of $\poc{G}$ be the set of all half-spaces of $G$, then we order it by inclusion and set $H^\ast=V\minus H$ for the complementation operator. Of course, some graphs (e.g. any odd cycle) will generate the trivial poc set in this way, but not so for median graphs:
\begin{thm}[\cite{Roller-duality}, proposition 5.9]\label{thm:every median graph is a poc dual} Let $G$ be a finite median graph. Then $G$ is canonically isomorphic to $\dual{\poc{G}}$ via the median-preserving map which sends each vertex $v$ to the collection of halfspaces of $G$ which contain $v$.
\end{thm}
An important conclusion (special case of proposition 6.11 in \cite{Roller-duality}) is:
\begin{cor}\label{cor:reconstruction} For any finite poc set $P$, the map $a\mapsto\half{a}$ is a poc-isomorphism of $P$ onto $\poc{\dual{P}}$. In other words, the poc-set $P$ may be reconstructed from its dual.
\end{cor}
From a practical standpoint, these two results offer an approach to understanding the geometry of $\cube{P}$ in terms of the order structure of $P$, which is the purpose of this and the following sections. As an aside, let us mention also that these results are best viewed together in categorical terms, as part of a restatement of the duality between the category of finite poc sets (with poc morphisms) and the category of finite median algebras (with median-preserving maps)~ ---~ see appendix \ref{subsection:duality theory} below.

We must consider the possible relations (if any) among elements $a,b\in P$. Those are:
\begin{equation}\label{eqn:nesting relations}
	a\leq b\,,\quad
	a^\ast\leq b\,,\quad
	a^\ast\leq b^\ast\,,\quad
	a\leq b^\ast
\end{equation}
It is easy to see that a pair of distinct {\it proper} elements will never satisfy two of the above conditions at the same time, as $\cube{P}$ provides us with a realization of $P$ inside $\power{P^\circ}$ -- after all, the last theorem tells us that:
\begin{equation}
	a\leq b \IFF \half{a}\subseteq\half{b}
\end{equation}
\begin{defn}[Nesting and Transversality]\label{defn:nesting} Suppose $a,b$ are proper elements of a weak poc set $P$. We say that they \emph{cross} ($a\pitchfork b$), if none of \eqref{eqn:nesting relations} hold. Otherwise, we say they are \emph{nested} ($a\Vert b$). A subset $A$ of $P$ is said to be \emph{nested} if all its elements are pairwise nested, and \emph{transverse} if its elements cross pairwise.\defstop
\end{defn}
Thus, the half-spaces of $\dual{P}$ are nothing more than the restriction to $\dual{P}$ of the half-spaces of the cube $S(P)^1$, with two of them nesting if and only if the corresponding elements of $P$ are nested, that is, if and only if exactly one of the following holds:
\begin{equation}
	\begin{array}{ll}
		\half{a}\cap\half{b}=\varnothing\,,\;&
		\half{a^\ast}\cap\half{b}=\varnothing\,,\;\\
		\half{a^\ast}\cap\half{b^\ast}=\varnothing\,,\;&
		\half{a}\cap\half{b^\ast}=\varnothing
	\end{array}
\end{equation}
We conclude that the more relations are on record in the order structure of $P$ the fewer transverse sets there are to be found there. In other words, nesting relations are an obstruction to high-dimesional cubes in $\cube{P}$: each additional relation in $P$ implies fewer faces of the original cube $S(P)$ survive the culling of incoherent vertices used for obtaining $\cube{P}$. At one extreme one finds $\cube{P}=S(P)$ when $P$ itself (up to removing improper elements) is transverse. At the other extreme (exercise for the reader) $\cube{P}$ forms a tree if and only if $P$ is nested.

\medskip
\subsubsection{Example: the path of length $N$}\label{example:N-path} Consider an idealized point-robot situated on the interval $\env=[0,1]$ and capable of moving about in this environment. Suppose the robot is endowed with binary sensors $a_1,\ldots,a_N$, each responding to the robot's position -- denoted for now by $x$ -- according to the rule, say, that $a_k$ turns on whenver $x<x_k:=k/N$. It would be reasonable for us to wish for the robot to eventually be able to realize that $a_k$ turning on implies $a_{k+1}$ turning on, for all $k<n$. Forming the poc set
\begin{equation}
	P=\gen{a_1,\ldots,a_N}{a_k<a_{k+1},\; k=1,\ldots,N-1}
\end{equation}
it is easy to verify that $\dual{P}$ is the $N$-path -- the path with $N+1$ vertices and $N$ edges -- whose vertices are all of the form
\begin{equation}
	v_k=\{\maxP\}\cup\left\{a_j^\ast\right\}_{j>k}\cup\left\{a_i\right\}_{i\geq k}\,,\quad 0\leq k\leq N
\end{equation}
Please note that the choice of the points $x_k\in\env$ is immaterial -- only their ordering should matter for the correctness of $\cube{P}$ as a discretized model of the `environment' $\env$ of our robot.

\begin{figure}[t]
	\begin{center}
		\includegraphics[width=\columnwidth]{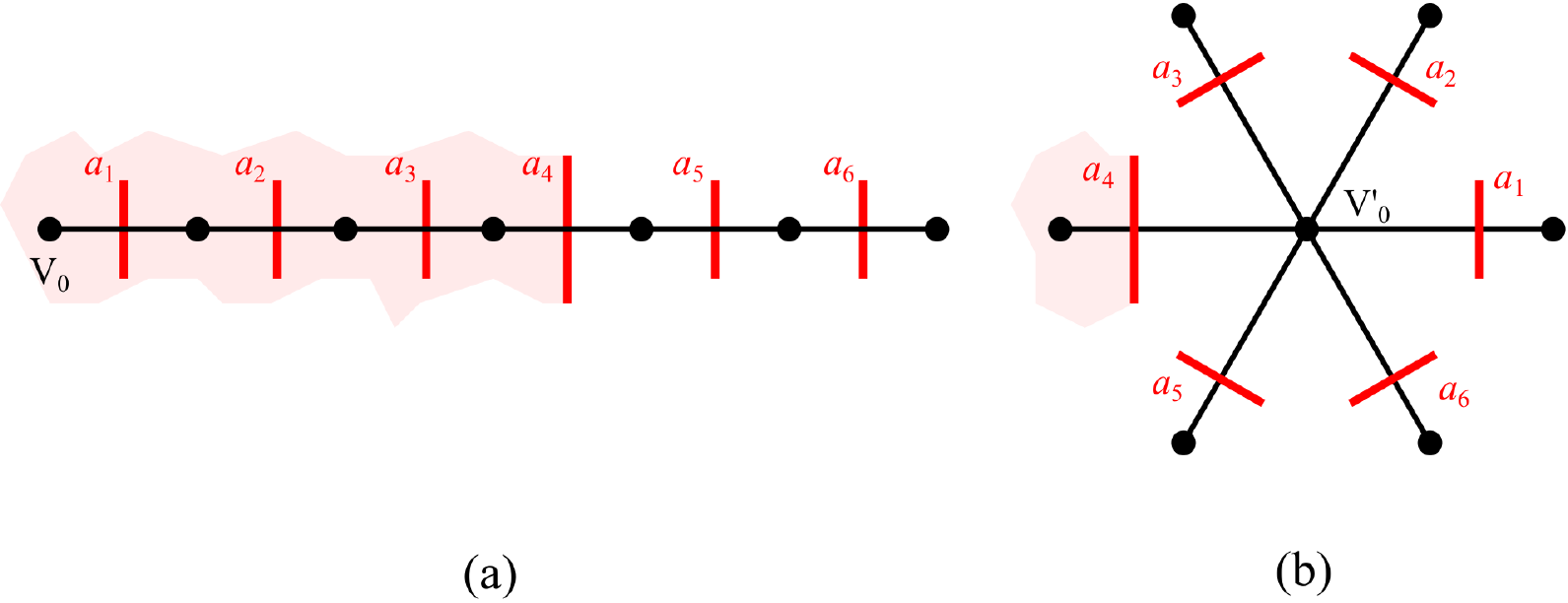}
		\caption{Dual graphs for two arrangements of sensors along the real line (see \ref{example:N-path}): threshold sensors encoding a path (a), and beacon sensors encoding a `starfish' (b).\label{fig:N-path}}
	\end{center}
\end{figure}

\medskip
At the same time, imagine that the sensors $a_k$ corresponded to `beacons', with $a_k$ turning on if and only if $\left|x-\frac{k}{N}\right|<\frac{1}{3N}$. Then a poc set description of the form
\begin{equation}
	P=\gen{a_1,\ldots,a_N}{a_k<a_j^\ast,\;1\leq k<j<N}
\end{equation}
would be more appropriate, indicating that the sensations $a_k$ are mutually exclusive. The resulting dual would still have $N+1$ vertices and $N$ edges, the vertices being:
\begin{equation}
	\begin{array}{rcl}
		v'_k&=&\{\maxP,a_k\}\cup\{a_j^\ast\}_{j\neq k}\text{ for }1\leq k\leq N\\[.5em]
		v'_0&=&\{\maxP\}\cup\left\{a_1^\ast,\ldots,a_N^\ast\right\}
	\end{array}
\end{equation}
In both cases the dual graph is a tree (a {\it path} and a {\it starfish}), and it is hard to ignore the difference in the quality of its representation of the underlying space -- see figure \ref{fig:N-path} 

\medskip
\subsubsection{Example: direct sums of poc sets}\label{example:direct sum} The easiest way to join two poc sets together is to form their direct sum:
\begin{defn}\label{defn:direct sum of poc sets}\label{defn:direct sum} Let $P$ and $Q$ be poc sets. Their \emph{direct sum} $P\vee Q$ is defined to be the quotient of their external disjoint union $P\sqcup Q$ by the identification $\minP_P=\minP_Q$ and $\maxP_P=\maxP_Q$, endowed with the following:
\begin{itemize}
	\item $a<b$ in $P\vee Q$ iff $a,b\in P$ and $a<b$ or $a,b\in Q$ and $a<b$;
	\item $b=a^\ast$ iff both $a,b\in P$ and $b=a^\ast$ or $a,b\in Q$ and $b=a^\ast$.
\end{itemize}
(We abuse notation by identifying each element of $P\cup Q$ with the equivalence class in $P\vee Q$ of its natural representative in $P\sqcup Q$)\defstop
\end{defn}
It is easy to verify, then, that 
\begin{equation}\label{eqn:direct sum}
	\cube{P\vee Q}\equiv\cube{P}\times\cube{Q}
\end{equation}
where the isomorphism is that of cubical complexes. Intuitively, any proper elements $a\in P$ and $b\in Q$ satisfy $a\pitchfork b$, resulting in every cube in $\cube{P}$ and every cube in $\cube{Q}$ to form a product cube in $\cube{P\vee Q}$. For example, the grid in figure \ref{fig:median} may be thought of as the product of an $N$-path with an $M$-path (for the appropriate values of $M$ and $N$) -- hence the dual of the direct sum of two poc sets of the type discussed in the preceding example \ref{example:N-path}.

\begin{figure}[t]
	\begin{center}
		\includegraphics[width=\columnwidth]{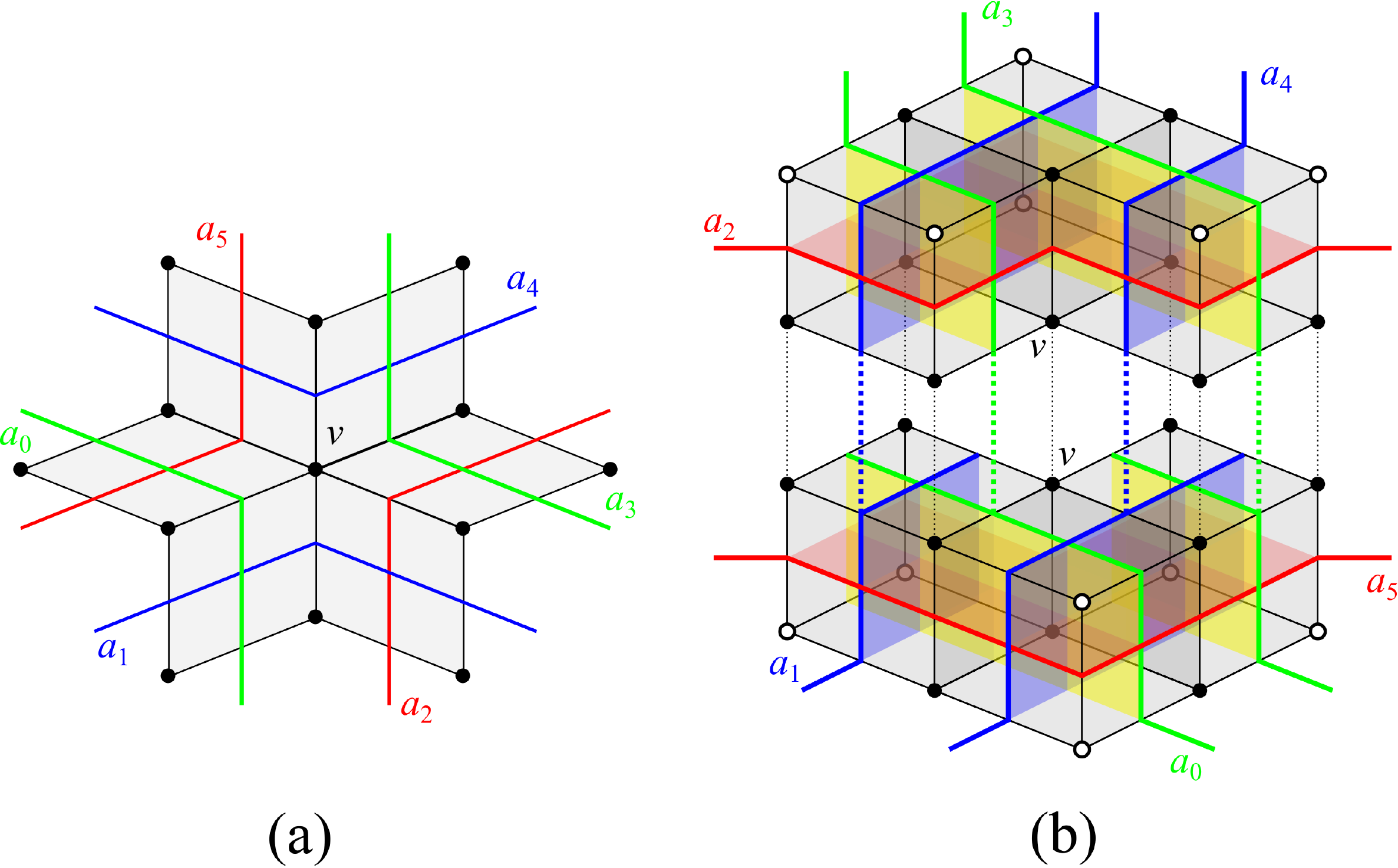}
		\caption{\small Cubical models for example \ref{example:N-cycle} with relations (a) $a_i<a_{i+x}^\ast$ where $x\in\{2,3,4\}$ and addition is modulo $6$, and (b) only the relations $a_i<a_{i+3}^\ast$ are present. Black vertices are those coherent in for both poc set structures. Vertices painted white are coherent vertices for agent $\#2$ that are incoherent for agent $\#1$. The vertex $v$ corresponds to the shared coherent $\ast$-selection $\{a_0^\ast,\ldots,a_5^\ast\}$.\normalsize\label{fig:six way compass}}
	\end{center}
\end{figure}

\subsubsection{Example: a cycle of length $6$}\label{example:N-cycle}
Imagine an agent -- call it $\#1$ -- living on the unit circle $\env=\fat{S}^1$. We mark six vertices, spread uniformly along the circle, with the digits $\{0,\ldots,5\}$. Suppose that agent $\#1$ is capable, for each position it occupies on $\env$, of asking any of the binary questions
\begin{itemize}
	\item $A_j$: {\it Am I positioned at arc length$<\tfrac{\pi}{3}$ from position $j$ along $\env$?}
\end{itemize}
Agent $\#2$ asks a slightly different set of questions:
\begin{itemize}
	\item $B_j$: {\it Am I positioned at arc length$<\tfrac{\pi}{2}$ from position $j$ along $\env$?}
\end{itemize}
The questions available to either agent have sufficient resolution to pinpoint the agent's position wherever it is, but we claim that the collection $\{A_j\}_{j=0}^5$ is, in a sense, more efficient than $\{B_j\}_{j=0}^5$. Let $P=\{\minP,\maxP\}\cup\{a_i,a_i^\ast\}_{i=0}^5$, where the $a_i$ are symbols to represent the sensations corresponding to $A_i$ for agent $\#1$ and to $B_i$ for agent $\#2$. We compare the resulting embeddings $\rho_i:P\into\power{V}$ defined by
\begin{eqnarray*}
	\rho_1(a_j)=A_j\,,\; \rho_1(a_j^\ast)=V\minus A_j\,,\quad\\
	\rho_2(a_j)=B_j\,,\; \rho_2(a_j^\ast)=V\minus B_j\,,\quad
\end{eqnarray*}
and with $\rho_i(\minP)=\varnothing$ and $\rho_i(\maxP)=V$, of course. We observe that both representations of $P$ in $\power{V}$ form injective poc morphisms of $P$ into $\power{V}$ if $P$ is given all relations of the form $a_i<a_{i+3}^\ast$ (addition modulo $6$), yet only agent $\#1$ can afford to also add the relations $a_i<a_{i+2}^\ast$ and $a_i<a_{i+4}^\ast$ to the record without losing the property of $\rho_1$ being a poc morphism. The difference between the duals is significant -- see figure \ref{fig:six way compass} -- clearly showing the advantage of the compact and simple world map that agent $\#1$ could deduce over the cumbersome monstrosity agent $\#2$ must deal with. Note how the complex (a) in the figure may be obtained from (b) through deleting the vertices painted white -- those are precisely the vertices of (b) forming incoherent families for the poc set structure represented in (a).

Two aspects of this example are noteworthy:
\begin{enumerate}
	\item The less nested poc set of the two example poc set structures is capable of accommodating both agents, thus giving us the means for comparing them.
	\item With none of the agents having direct access to the realization maps, they should be looking for efficient means of \emph{evolving their maps} in an adaptive fashion so as to produce a good enough symbolic approximation of the ground truth. 
\end{enumerate}

\subsection{Cubings and the Duality Theory of Weak Poc Sets}\label{subsection:duality theory}
\subsubsection{Sageev-Roller Duality from the Categorical Viewpoint} In the finite case, the duality theory of poc sets has a very clean formulation in category-theoretical terms. For a quick review of the basic notions of Category Theory we refer the reader to \cite{Kozlov-combi_alg_top} chapter 4, while here we will stick to the specific categories of interest:
\begin{itemize}
	\item $\catpoc$, the category of finite poc sets\footnote{One could work with the full category $\mbf{Poc}$ of {\it all} poc sets (rather than just the finite ones) but this introduces major complications that seem unnecessary given the current application. Similarly for the case of median graphs/algebras.}, where each $P,Q\in\catpoc$ have assigned to them the set $\morph{P}{Q}$ of poc morphisms from $P$ to $Q$;
	\item $\catmed$, the category of finite median graphs, where each $G,H\in\catmed$ are assigned the set $\morph{G}{H}$ of median-preserving maps from the vertex set of $G$ to the vertex set of $H$ (such maps are called {\it median morphisms}).
\end{itemize}
What connects the two categories is the assignment of the graph $\dual{P}$ to every poc set $P$. The important bit here is that this assignment is not confined to the level of objects, but, rather, extends over the level of maps as well in a natural way:
\begin{defn}\label{defn:dual map} Let $f:P\to Q$ be a morphism of weak poc sets. The dual map $f^\circ:Q^\circ\to P^\circ$ is defined to be the pullback map $f^\circ(A)=f\inv(A)$.\defstop
\end{defn}
It is easy to verify that $f^\circ:Q^\circ\to P^\circ$ is a median-preserving map, that is:
\begin{equation}\label{eqn:median preserving}
	f^\circ\left(\med{u}{v}{w}\right)=\med{f^\circ(u)}{f^\circ(v)}{f^\circ(w)}
\end{equation}
where the medians are computed in the appropriate duals. Thus, a map $f\in\morph{P}{Q}$ yields a map $f^\circ\in\morph{\dual{Q}}{\dual{P}}$. Moreover, one easily checks that this is done in a way that respects composition, that is:
\begin{equation}
	(g\circ f)^\circ=f^\circ\circ g^\circ
\end{equation}
whenever the composition of the poc morphisms $f,g$ is well-defined. This notion of map between categories is called a {\it functor}. The above constructions (of the dual graph and the dual map), together, are known as the {\it Sageev-Roller duality}.

Applying the results \ref{thm:every median graph is a poc dual} and\ref{cor:reconstruction}, we conclude that the above assignments form a {\it complete duality}, or {\it co-equivalence of categories}, between $\catpoc$ and $\catmed$. That is, there are:
\begin{itemize}
	\item{\bf A correspondence between $\catpoc$ and $\catmed$ at the level of objects: } $P\mapsto\dual{P}$ is a one-to-one correspondence between the collection of finite poc sets and the collection of median graphs;
	\item{\bf A correspondence between $\catpoc$ and $\catmed$ at the level of maps: } $f\mapsto f^\circ$ is a composition-reversing one-to-one correspondence between poc morphisms and median morphisms.
\end{itemize}
Thus, Sageev-Roller duality is a dictionary, translating order-theoretic statements about finite poc sets into graph-theoretic statements about finite median graphs and vice-versa. Loosely speaking, the aspects of Boolean Algebra covered by poc sets may be conveniently interpreted in terms of the convex geometry of median graphs, reasoned about within this framework, and the conclusions may then be translated back into the Boolean Algebra setting for the purpose of dealing with applications. We will now proceed to survey some of the contributions of this category-theoretic point of view to our application.

\medskip
\subsubsection{Extending Sageev-Roller Duality to Weak Poc Sets} Recalling the fact that the category $\catpoc$ of proper finite poc sets is too restrictive for the purposes of our current application \cite{Allerton_2012}, the first order of business is to verify that Sageev-Roller duality extends to weak poc sets.

The first observation regarding dual maps is a consequence of the fact that no coherent subset of a weak poc set may contain a negligible element:
\begin{lemma}\label{lemma:reduction to poc sets} Let $P$ be a weak poc set and let $\pi:P\to\hat{P}$ denote the canonical projection. Then $p^\circ:\hat{P}^\circ\to P^\circ$ is a median isomorphism, making $\cube{P}$ and $\cube{\hat{P}}$ isomorphic cubical complexes.
\end{lemma}
Thus, weak poc sets are indistinguishable from poc sets, as far as dual graphs are concerned: applying Sageev-Roller duality one simply obtains
\begin{cor} For any weak poc set $P$, $\hat{P}$ is naturally isomorphic to $\poc{\dual{P}}$.\defstop
\end{cor}
At the same time, weak poc sets form a more flexible class of objects. In particular, weak poc set structures are easier to represent and evolve dynamically using snapshots \ref{section:snapshots}.

\medskip
\subsubsection{Example: Realizations}\label{realization}
Suppose $\spc$ is the state space (possibly infinite) of some system, and $\sens$ is a collection of binary sensors sensitive to the state of the system, such as in examples \ref{example:N-path} and \ref{example:N-cycle}. Since the sensors are binary we may assume that each sensor $a\in\sens$ comes paired with a sensor $a^\ast\in\sens$ corresponding to the negation of $a$. In other words, the sensorium $\sens$ comes equipped with a fixpoint-free involution $\ast$ and with a {\it realization map} $\rho:\sens\to\hsm{B}\subseteq\power{\spc}$ satisfying $\rho(a^\ast)=\spc\minus\rho(a)$ for all $a\in\sens$, where $\hsm{B}$ is a prescribed $\sigma$-algebra of measurable events in $\spc$. It also costs us nothing to assume there is a special sensor $\minP\in\sens$ that is never on, that is: $\rho(\minP)=\varnothing$.

Suppose now that, having spent some time observing state transitions in $\spc$, we are able to write down some implication relations among the elements of $\sens$. These will be recorded in the form of a partial ordering $(\leq)$. We would like to use our {\it a-priori} knowledge that $\rho(a^\ast)=\spc\minus\rho(a)$ and, naturally, we would like to believe that $a\leq b$ implies $\rho(a)\subseteq\rho(b)$, in which case $\rho$ becomes a morphism of the weak poc set $\ppoc=(\sens,\leq,\ast)$ into the poc set $\hsm{B}$. Assuming this is correct, what can we say?

For any observed state $x\in\spc$ the poc set $\hsm{B}$ supplies us with a coherent $\ast$-selection $\hsm{B}_x=\set{A\in\hsm{B}}{x\in A}$. The dual $\rho^\circ:\hsm{B}^\circ\to\ppoc^\circ$ then produces a vertex $v_\rho(x):=\rho^\circ(\hsm{B}_x)$ in $\cube{\ppoc}$.
\begin{defn}[Consistent Vertices]\label{defn:consistent vertices} Let $\rho$ be a realization of a poc set $P$. Then the vertices of the form $v_\rho(x)$ as above are called {\it $\rho$-consistent vertices}.
\end{defn}
For example, the vertex $v$ in figures \ref{fig:six way compass} (a) and (b) is inconsistent for either realization.

\medskip
It is possible that $v_\rho:\spc\to\ppoc^\circ$ is not surjective, motivating the definition:
\begin{defn}[Punctured Dual]\label{defn:punctured dual} Let $\rho$ be a realization of a poc set $P$. The {\it punctured dual} (with respect to $\rho$), denoted $\punc{\ppoc}$, or $\cube{\ppoc,\rho}$, is the cubical sub-complex of $\cube{\ppoc}$ induced by the set of $\rho$-consistent vertices of $\ppoc{}^\circ$ which are contained in $\cube{\ppoc}$.
\end{defn}
As a corollary of this discussion we obtain:
\begin{cor} $\cube{\ppoc}$ is the smallest cubical sub-complex of $S(P)$ with the property that, for {\it any} realization $\rho$ of $\ppoc$, if $\rho$ is a poc morphism, then $\cube{\ppoc}$ contains all $\rho$-consistent vertices of $S(P)$.
\end{cor}
In other words, any realization that is also a poc morphism gives rise to a discrete representation of $\spc$ in $\cube{\ppoc}$ via the mapping $v_\rho:\spc\to\ppoc^\circ$. The benefit of maintaining a record of the order in $\sens$ is our ability to discard the incoherent vertices of $S(P)$ (viewed as possible states of the observed system) without the risk of losing any information about $\spc$, while possibly gaining some insight into the organization of $\spc$, as stated in the introduction, contribution {\bf (i)}.

\medskip
\subsubsection{Realizations, Cubings and Topology}\label{homotopy type result} We recall a definition from \cite{Sageev-thesis}:
\begin{defn} A \emph{cubing} is a simply connected, non-positively curved cubical complex.\defstop
\end{defn}
We point the reader to \cite{Bridson_Haefliger-textbook} for a detailed account of non-positively curved metric spaces. 
For the purpose of this paper it will suffice to quote a corollary of the well-known Cartan-Hadamard theorem (\cite{Bridson_Haefliger-textbook}, II.4.1):
\begin{cor}\label{cor:cubings are contractible} Cubings are contractible.
\end{cor}
We owe the following theorem in its full generality (finite and infinite cases) to the collective efforts of Michah Sageev \cite{Sageev-thesis}, Martin Roller \cite{Roller-duality} and Victor Chepoi \cite{Chepoi-median}.
\begin{thm} The following are equivalent for a finite simple graph $G$:
\begin{enumerate}
	\item $G$ is the $1$-dimensional skeleton of a cubing;
	\item $G$ is a median graph;
	\item $G$ is isomorphic to $\dual{P}$ for some poc set $P$;
	\item $G$ is the $1$-dimensional skeleton of $\cube{P}$ for some poc set $P$.
\end{enumerate}
\end{thm}
Further developing the discussion in the preceding paragraph, we recall one of the main theorem of \cite{Allerton_2012}, applied to that setting:
\begin{thm}\label{thm:homotopy type} Let $\spc$ be a topological space and let $\ppoc$ be a poc set structure on a finite set $\sens$ with realization $\rho:\ppoc\to\power{\spc}$. Let $\hsm{C}$ denote the collection of cubes in the cubical complex $\punc{\ppoc}=\cube{\ppoc,\rho}$. For each $C\in\hsm{C}$ let $\spc_C=(r^\circ)^\inv(C)$ be the set of all points in $\spc$ witnessing $C$. If $\rho$ is a poc morphism, and every $\spc_C$, $C\in\hsm{C}$ has an open neighbourhood $\hsm{N}_C\subset\spc$ such that:
\begin{enumerate}
	\item each $\hsm{N}_C$ is contractible;
	\item $\{\spc_C\}_{C\in\hsm{C}}$ and $\{\hsm{N}_C\}_{C\in\hsm{C}}$ have isomorphic nerves,
\end{enumerate}
then $\punc{\ppoc}$ is homotopy-equivalent to $\spc$.
\end{thm}
To illustrate the theorem, consider figure \ref{fig:six way compass} again: note how deleting the vertex $v$ from either discrete model of the circle results in a space with the correct homotopy type (that of the circle). On the other hand, going back to example \ref{example:N-path} and figure \ref{fig:N-path}, the last theorem explains the qualitative difference in representations of the interval between the two provided models: while the `thresholds' model satisfies the requirements of the theorem, the `beacons' model possesses a vertex -- the one marked $V'_0$ -- whose set of witnesses is not connected.

This result could be interpreted as stating a condition on the {\it richness} of the sensorium $(\sens,\rho)$: the complex $\cube{\ppoc}$ provides an observer with a discretized contractible model of the state space $\spc$ of the observed system, while $\punc{\ppoc}$ is a more realistic model of $\spc$ taking the standard form of a contractible space minus a set of obstacles. 

\medskip
\subsubsection{Example: A `bad' Poc Morphism}\label{example:bad dual} It is not true in general that the dual of a poc morphism $f:P\to Q$ extends to a morphism of graphs. For example, consider the situation
\begin{equation}
	P=\gen{a,b,c}{a<b,\,b<c},\quad
	Q=\gen{x,y}{x<y}
\end{equation}
and $f:P\to Q$ is defined by $f(a)=f(b)=x$ and $f(c)=y$. The duals and dual map are illustrated in figure \ref{fig:bad dual}.

\begin{figure}[ht]
	\begin{center}
		\includegraphics[width=.9\columnwidth]{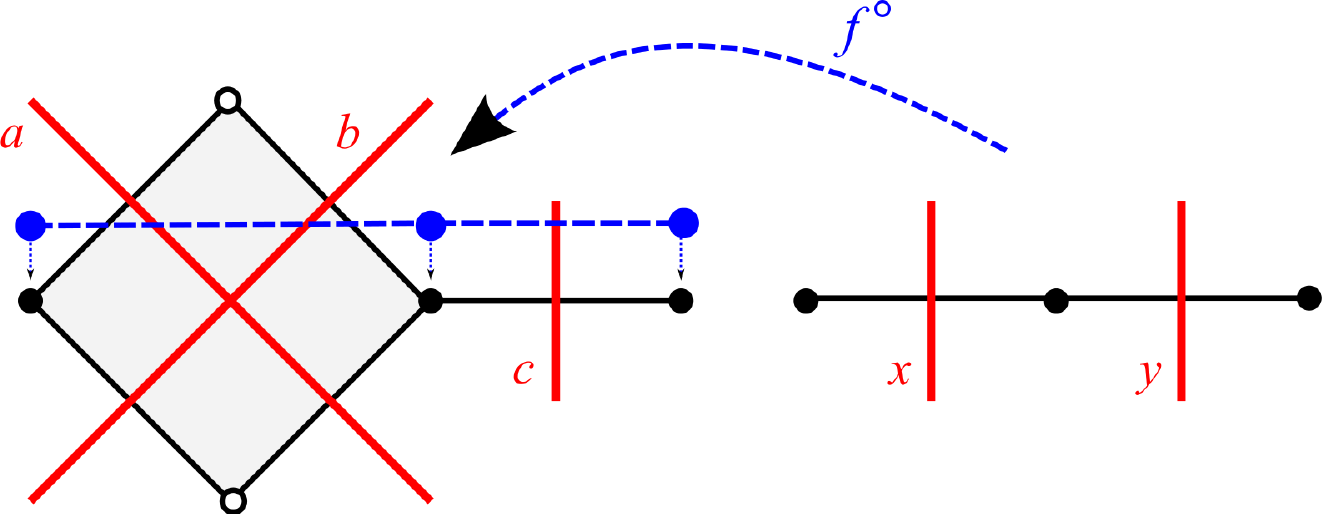}
		\caption{The dual of a poc morphism is not necessarily a graph morphism (details in \ref{example:bad dual}).\label{fig:bad dual}}
	\end{center}
\end{figure}

The absence of a canonical choice of extension for $f^\circ$ to a graph morphism of $\dual{Q}$ into $\dual{P}$ hints at a solution directly involving cubings: if one were to extend the range of $f^\circ$ to include the $2$-dimensional cube shown in the figure, it would have been possible to extend $f^\circ$ to a cellular map taking the edge of $\dual{Q}$ crossed by $x$ to an appropriately chosen diagonal of that cube. More generally, it {\it is} possible to extend $f^\circ$ to a continuous embedding of $\cube{Q}$ into $\cube{P}$ for {\it any} poc morphism $f:P\to Q$ by applying convexity properties of the canonical piecewise-Euclidean metrics on $\cube{P}$ and $\cube{Q}$ (\cite{Bridson_Haefliger-textbook}, II.2.7). Thus, although median graphs are sufficient for describing the dual graphs of poc sets, describing the {\it dual morphisms} requires the higher dimensional geometry of cubings.

\subsubsection{Example: Degeneration}\label{example:degeneration} Recall our promise to maintain weak poc set structures on a sensorium $\sens$ {\it dynamically}, updating the ordering on $\sens$ in real time. The duality theory of poc sets provides us with a hint as to how such maintenance should be done. The learning methods of section \ref{section:snapshots} are motivated by the an analogy between the following observations and the ideas underlying Hebbian learning:

The kind of update we expect to see in an instance of learning is captured in the following simple example.
\begin{eqnarray*}
	P_1&=&\gen{a,b,c}{a<c,\,b<c}\,,\\
	P_2&=&\gen{a,b,c}{a<b<c}
\end{eqnarray*}

\begin{figure}[ht]
	\begin{center}
		\includegraphics[width=\columnwidth]{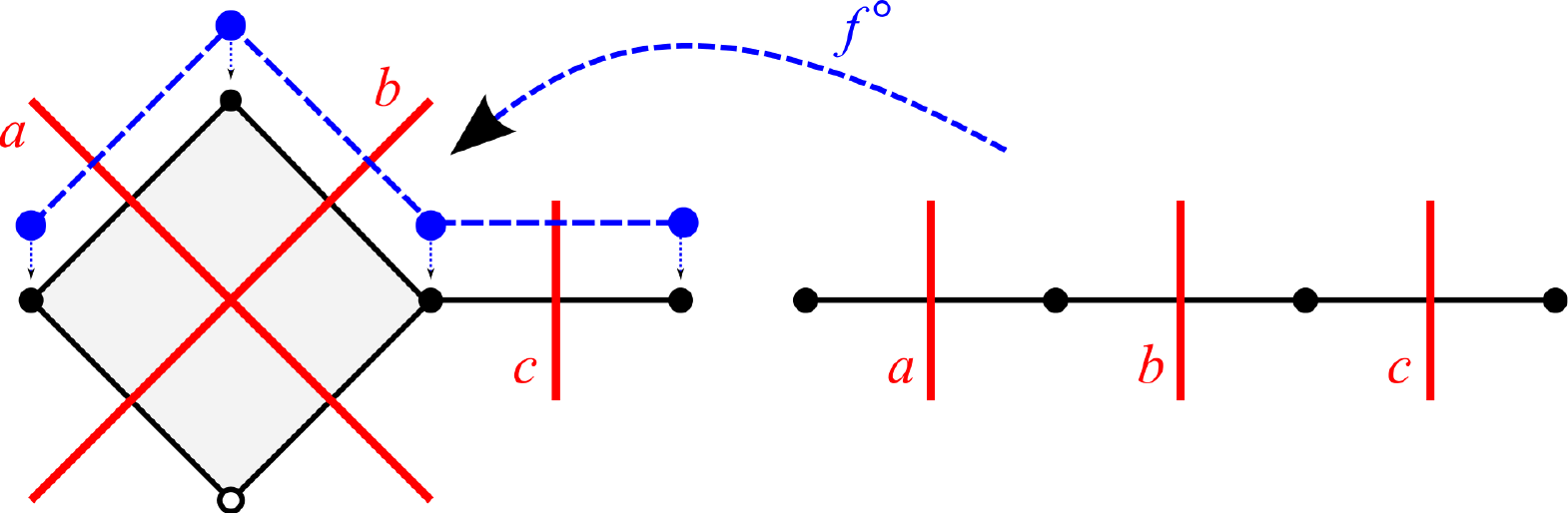}
		\caption{The dual of a degeneration is an embedding of median graphs (details in \ref{example:degeneration}).\label{fig:degeneration}}
	\end{center}
\end{figure}

The two poc set structures have the same underlying set (denote it by $P$) and the identity map $f=\id{P}:P_1\to P_2$ is a morphism, while the inverse map $g=\id{P}:P_2\to P_1$ is not. Thinking of $P_1$ as representing an observer yet undecided regarding the nature of nesting (if any) of the pair $\{a,b\}$ and therefore maintaining $a\pitchfork b$ in $P_1$, we see poc set $P_2$ as representing an observer with an identical set of beliefs except for the additional relation $a<b$. Figure \ref{fig:degeneration} visualizes the dual map $f^\circ$.

In general, if $P_1$ and $P_2$ are poc sets with the same underlying set $P$ and $f=\id{P}:P_1\to P_2$ is a poc morphism, then the dual $f^\circ$ has the following properties (see e.g. \cite{Roller-duality}):
\begin{prop}\label{prop:some dual properties} Suppose $f:P_1\to P_2$ is a bijective poc morphism. Then:
\begin{enumerate}
	\item $f^\circ$ is injective (\cite{Roller-duality}, proposition 7.8);
	\item $f^\circ$ extends to an injective cellular embedding of $\cube{P_2}$ in $\cube{P_1}$;
	\item The image of $\cube{P_2}$ under this embedding is a strong deformation retract of $\cube{P_1}$. 
\end{enumerate}
\end{prop}
A more complicated instance of this setting is very nicely visualized in figure \ref{fig:six way compass}.

\section{Proofs of Technical Results}
\subsection{Proof of proposition \ref{prop:acyclicity lemma}}\label{proofs:acyclicity lemma} We first extend the weight function $ab\mapsto\witness{ab}$ to a symmetric function of $\sens\times\sens$ by setting
\begin{equation}\label{eqn:weight extension}
	\witness{aa}=\witness{ab}+\witness{ab^\ast}\,,\quad
	\witness{aa^\ast}=0
\end{equation}
for any $a\in\sens$ and for any $b\in\sens$ with $ab\in\pocg$. The consistency constraint \eqref{eqn:consistency identity} implies that this extension is well-defined. This allows us to extend the definition of $\ori{\wild}$ in \eqref{eqn:orientation cocycle} to the whole of $\sens\times\sens$ while satisfying the following identities
\begin{equation}\label{eqn:extended cocycle symmetries}
	\ori{ab}+\ori{ba}=0\,,\quad
	\ori{aa}=0
\end{equation}
for all $a,b\in\sens$. Note that $\ori{\wild}$, by definition, takes directed edges and loops for its input.

Further, the consistency identity allows us to write, for any $ab\in\pocg$:
\begin{eqnarray}
	\ori{aa^\ast}&=&\witness{a^\ast a^\ast}-\witness{aa}\\
		&=&	\witness{a^\ast b}+\witness{a^\ast b^\ast}-\witness{ab}-\witness{ab^\ast}\\
		&=&	\ori{ab}+\ori{ab^\ast}\label{eqn:degenerate cocycle identity 2}
\end{eqnarray}
We next observe that the cocycle constraint \eqref{eqn:cocycle identity} may be rewritten in the form
\begin{equation}\label{eqn:cocycle identity 2}
	\ori{ab}+\ori{bc}=\ori{ac}
\end{equation}
Let us verify that this identity is, in fact, an identity over all $a,b,c\in\sens$. Due to the symmetries of $\ori{\wild}$ in \eqref{eqn:extended cocycle symmetries} it suffices to verify \eqref{eqn:cocycle identity 2} only in the following cases:
\begin{enumerate}
	\item The pair $ab\in\pocg$ and $c=a$: this case is taken care of by the anti-symmetry identity of $\ori{\wild}$.
	\item The pair $ab\in\pocg$ and $c=a^\ast$: this is precisely the statement of \eqref{eqn:degenerate cocycle identity 2}.
	\item The pair $ac\in\pocg$ and $b,c\in\{a,a^\ast\}$: without loss of generality, either $b=c=a$ and \eqref{eqn:cocycle identity 2} ends up claiming that $0+0=0$, or $b=a$ and $c=a^\ast$ -- in which case the statement turns into the trivial identity $\ori{aa^\ast}=\ori{aa^\ast}$.
\end{enumerate} 
(cases 1-2 correspond to exactly two of the pairs being proper; case 3 accounts for all situations when none of the pairs is proper; having exactly one of the pairs proper is impossible).

Now suppose that $p=(a_0,\ldots,a_m)$ is any directed vertex path in the given poc graph $\pog$. Then, applying the identity \eqref{eqn:cocycle identity 2} repeatedly we obtain:
\begin{equation}
	\ori{a_0a_m}=\ori{a_0a_1}+\ldots+\ori{a_{m-1}a_m}
\end{equation}
By the assumption on $\pog$, all the summands on the right hand side are positive. In particular, if $p$ were a cycle with $a_m=a_0$ we would have obtained
\begin{equation}
	0=\ori{a_0a_0}=\ori{a_0a_m}>0
\end{equation}
-- a contradiction. We must therefore conclude that directed cycles in $\pog$ are impossible, as desired.

\subsection{Proof of lemma \ref{lemma:empirical evolves from trivial}}\label{proofs:empirical evolves from trivial} An evolution of the trivial snapshot is an empirical snapshot, for if $\snap{S}=\snap{S}\at{t}$ can be written in the form $\snap{S}=O_k\ast\cdots\ast O_1\ast\nullsnap$ then defining indicators
\begin{equation}
	c_{ab}^k=\ev{\indicator{O_k}}{a}\cdot\ev{\indicator{O_k}}{b}
\end{equation}
-- compare with \eqref{eqn:coincidence indicators} -- one would have the following identity holding for $\snap{S}$:
\begin{equation}\label{eqn:empirical decomposes into indicators}
	\witness{ab}=\sum_{k=1}^t c_{ab}^k\in\ZZ_{_{\geq 0}}
\end{equation}
Since the functions $c_{\cdot}^k$ satisfy the consistency constraint, so does their sum $\witness{\cdot}$. The clock requirement is satisfied, too, since for any proper pair $\{a,b\}$ one has:
\begin{equation}
	\witness{a}+\witness{a^\ast}=\sum_{k=1}^t\underbrace{(c_{ab}+c_{ab^\ast}+c_{a^\ast b^\ast}+c_{a^\ast b})}_{=1}=t
\end{equation}
-- compare with equation \eqref{indicator identities}.

\medskip
Conversely, due to the presence of an integer clock, it suffices to show that any empirical snapshot $\snap{S}$ can be written in the form $\snap{S}=O\ast\snap{T}$ where $O=\state{\cdot}\snap{S}$ and $T$ is either trivial or empirical. Let $\snap{T}$ be the weighted graph obtained from $\snap{S}$ by performing the following operations:
\begin{enumerate}
	\item Subtract one unit from $\witness{ab}$ for every proper pair $a,b\in\sens$ satisfying $\{a,b\}\subseteq O$. 
	\item For any $a\in\sens$, set $\state{a}$ for $\snap{T}$ to $1$ iff $\witness{a}>\witness{a^\ast}$ in $\snap{T}$.
\end{enumerate}
The set $\state{\snap{T}}$ is a complete $\ast$-selection by construction, so it remains to verify the consistency and synchrony conditions for the new snapshot. For every proper pair $a,b\in\sens$, the fact that $O$ is a $\ast$-selection implies that all but one of the edge counters in $\snap{T}\res{ab}$ coincide with their $\snap{S}$ counterparts, while the exceptional one -- let it be $\witness{ab}$ without loss of generality -- is smaller than its counterpart in $\snap{S}$ by one unit. Since the sum of edge counters in $\snap{S}\res{ab}$ is independent of the choice of square, we conclude the same is true for $\snap{T}\res{ab}$. To prove consistency we observe that $\witness{ab}+\witness{ab^\ast}$ suffering a decrease (of one unit) in the passage to $\snap{T}$ implies $a\in O$ and hence $\witness{ac}+\witness{ac^\ast}$ must suffer a decrease as well for any $c\notin\{a,a^\ast\}$, since $O$ is a complete $\ast$-selection. Finally, with $\witness{a}$ being well-defined in $\snap{T}$ for all $a\in\ppoc$, the choice of $\state{\snap{T}}$ guarantees that $\state{a}=1$ is only possible in $\snap{T}$ if $\witness{a}>0$ in $\snap{T}$. 

\subsection{Equivalences in probabilistic Snapshots}\label{proofs:adding equivalences}

Throughout this section, $\snap{S}$ is a probabilistic snapshot satisfying the triangle inequality \eqref{eqn:triangle inequality for snapshots}. The reason for the name is that the symmetric {\it dissimilarity measure} on $\sens\times\sens$ defined by
\begin{equation}
	\tri{ab}:=\witness{a^\ast b}+\witness{ab^\ast}\geq 0
\end{equation}
allows rewriting \eqref{eqn:triangle inequality for snapshots} in the form
\begin{equation}
	\tri{ac}\leq\tri{ab}+\tri{ac}
\end{equation}
for all $a,b,c\in\sens$.

Now let us turn to the purpose of this discussion, the analysis of equivalences in $\sens$ that are observed in $\snap{S}$. Let $Eq(\snap{S})$ denote the undirected graph by vertex set $\sens$ having edges $ab$ and $a^\ast b^\ast$ for every $ab\in\pocg$ satisfying \eqref{eqn:equivalence}. Let $\bar\sens$ denote the partition of $\sens$ induced by the connected components of $Eq(\snap{S})$, and let $eq:\sens\to\bar\sens$ denote the map sending $a\in\sens$ to its block in $\bar\sens$. In other words, $eq(a)$ is nothing more than the list of all sensations $b\in\sens$ deemed equivalent to $a$, either directly or through a finite chain of equivalences. 

Returning to the notation of appendix \ref{proofs:acyclicity lemma} and keeping in mind \eqref{eqn:weight extension} we observe that:
\begin{equation}\label{eqn:equiv1}
	a\equiv b\THEN \ori{ab}=\witness{a^\ast b}-\witness{ab^\ast}=0
\end{equation}
and since $\ori{\wild}$ is additive \eqref{eqn:cocycle identity 2} we conclude:
\begin{equation}\label{eqn:equiv2}
	b\in eq(a)\THEN \ori{ab}=0
\end{equation}
In particular, since $\ori{ab}>0$ whenever $ab\in\snapdir{\snap{S}}$, we conclude:
\begin{lemma} If $a,b\in\sens$ are connected by a directed path in $\snapdir{\snap{S}}$ then $eq(a)\cap eq(b)=\varnothing$.
\end{lemma}
Similarly, for the metric $\tri{\wild}$ on $\sens$, we have that $\tri{ab}=0$ if $a\equiv b$ by definition. The triangle inequality allows us to conclude:
\begin{lemma} If $b\in eq(a)$ then $\tri{ab}=0$.
\end{lemma}
\begin{proof} We have $b\in eq(a)$ iff there is a finite sequence of elements in $\sens$ of the form
\begin{equation}
	a=a_0\equiv a_1\equiv\ldots\equiv a_n=b\,,\quad n\geq 0
\end{equation}
Applying the triangle inequality $n-1$ times we obtain
\begin{equation}
	\tri{ab}\leq\sum_{i=1}^n\tri{a_{i-1}a_i}=0\,,
\end{equation}
as desired.
\end{proof}

In order to see that $\snapdir{\snap{S}}$ actually defines a weak poc set structure on $\bar\sens$ we need the following:
\begin{lemma} Given $\snap{S}$ as above, for all $a\in\sens$ we have:
\begin{enumerate}
	\item $eq(a^\ast)=eq(a)^\ast$;
	\item $eq(a^\ast)\neq eq(a)$.
\end{enumerate}
(Recall the convention $A^\ast=\set{x^\ast}{x\in A}$ for $A\subseteq\sens$)
\end{lemma}
\begin{proof} For (1) it suffices to observe that $a\equiv b$ implies $a^\ast\equiv b^\ast$ by construction. Assertion (2) follows from observing that
\begin{equation}
	eq(a^\ast)=eq(a)\IFF eq(a^\ast)\cap eq(a)\neq\varnothing\IFF a^\ast\in eq(a)
\end{equation}
However, by the preceding lemma, $a^\ast\in eq(a)$ would imply $\tri{aa^\ast}=0$ which, in turn, would contradict the obvious equality \begin{equation}
	\tri{aa^\ast}=\witness{a^\ast a^\ast}+\witness{aa}=1
\end{equation}
finishing the proof.
\end{proof}
The following proposition summarizes our progress thus far:
\begin{prop}\label{prop:adding equivalences} Let $\snap{S}$ be a probabilistic snapshot satisfying the triangle inequality. Then:
\begin{enumerate}
	\item The operation $eq(a)\mapsto eq(a^\ast)$ defines a fixpoint-free involution on $\bar\sens$.
	\item The directed graph $\pog$ with vertex set $\bar\sens$ and with an edge pointing from $eq(a)$ to $eq(b)$ iff there exist $a'\in eq(a)$ and $b'\in eq(b)$ such that $a'b'$ is an edge of $\snapdir{\snap{S}}$ is an acyclic poc graph.
	\item Let $\ppoc=\poc{\snapdir{\snap{S}}}$ and $\cl{\ppoc}=\poc{\pog}$. Then the quotient map $eq:\sens\to\bar{\sens}$ is a poc morphism of $\ppoc$ onto $\cl{\ppoc}$ with the property that every fiber of $eq$ is a transverse subset of $\ppoc$.
	\item For any subset $A\subset\sens$ one has 
	\begin{equation}\label{propagation with equivalences}
		\up{A}=eq\inv\left(\up{eq(A)}\right)
	\end{equation}
	In particular, propagation over $\cl{\ppoc}$ is equivalent to propagation over $\snapdir{\snap{S}}_0$ (see defn. \ref{defn:poc graph with equivalences}).\defstop
\end{enumerate}
\end{prop}

\subsection{Proof of proposition \ref{prop:representing in product}}\label{proofs:representing in product}
This proof requires the results of section \ref{subsection:duality theory}. The inclusions:
\begin{equation}
	inc_{act}:\pact\at{t}\into\ppoc\at{t}\,,\quad
	inc_{obs}:\pobs\at{t}\into\ppoc\at{t}
\end{equation}
satisfy $proj_{act}=inc_{act}^\circ$ and $proj_{obs}=inc_{obs}^\circ$, ans since $\sens=\pact\cup\pobs$ and $\pact\cap\pobs=\{\minP,\maxP\}$, we conclude that the identity map $\mathrm{id}_{\ppoc}$ of $\ppoc$ is, in fact, a surjective poc morphism
\begin{equation}
	\mathrm{id}_{\ppoc}:\pact\at{t}\vee\pobs\at{t}\to\ppoc\at{t}
\end{equation}
(see \ref{example:direct sum} for a definition of $\pact\at{t}\vee\pobs\at{t}$).
By proposition \ref{prop:some dual properties}, the dual of this map is a median-preserving embedding of cubical complexes:
\begin{equation}\label{eqn:splitting the sensorium}
	\mathrm{id}_{\ppoc}^\circ:\cube{\ppoc\at{t}}\into\cube{\pact\at{t}}\times\cube{\pobs\at{t}}\,,
\end{equation}
as required.

\subsection{Local Structure of Duals and Greedy Navigation}\label{subsection:greedy navigation}
In \cite{Allerton_2012} we suggested exploring the link between the convexity theory of duals of weak poc sets and planning in DBAs, yet the formal results contained therein proved insufficient for supporting the planning algorithms proposed in this paper. This section fills in this gap.

Throughout this section we fix a finite weak poc set $P$ and the median graph $\Gamma=\dual{P}$ (which is to say, $\Gamma$ is an arbitrary finite median graph). We study the problem of computing the image of a non-empty convex subset $V(S)$ of $\Gamma$ under the closest point projection of $\Gamma$ to the convex subset $V(T)$. 

\medskip
\subsubsection{Gates} We recall the following definitions and results from \cite{Roller-duality}:
\begin{defn}[Separator] Let $K,L\subseteq P^\circ$ be sets. The set
\begin{equation}\label{eqn:separator}
	\sep{K,L}=\set{a\in P}{K\subseteq V(a)\,,\;L\subseteq V(a^\ast)}
\end{equation}
is called the {\emph separator of $K$ and $L$}.\defstop
\end{defn}
The inequality $\ellone{u}{v}\geq\card{\sep{K,L}}$ follows immediately for all $u\in K$ and $v\in L$. This motivates:
\begin{defn}[Gate] Let $K,L\subseteq P^\circ$. A \emph{gate for $K,L$} is a pair of points $u\in K$, $v\in L$ such that $\ellone{u}{v}=\card{\sep{K,L}}$.\defstop
\end{defn}
The following result is well known in our setting:
\begin{prop}\label{prop:gates exist} Let $K,L$ be non-empty convex subsets of $\Gamma$ and let $u\in K$ and $v\in L$. Then $u,v$ form a gate for $K,L$ if and only if $\proj{K}{v}=u$ and $\proj{L}{u}=v$. Moreover, any pair of non-empty convex subsets of $\Gamma$ has a gate.
\end{prop}
We will apply this proposition without proof. An important consequence for us is the following:
\begin{lemma}\label{lemma:supporting halfspace} Suppose $K=\half{S}$ and $S\subset P$ is coherent. Then, for any $a\in P$, if $K\subseteq\half{a}$ then there exists $s\in S$ such that $s\leq a$.
\end{lemma}
\begin{proof} Let $u\in K$ and $v\in L:=\half{a^\ast}$ form a gate. Since $v\notin A$, there exists $s\in S$ such that $v\in\half{s^\ast}$. 

Suppose there were a $w\in B$ with $w\in\half{s}$, and consider $m=\med{u}{v}{w}$. Then $a\in v,w$ implies $a\in m$, but the inequality
\begin{equation}
	\ellone{u}{v}=\ellone{u}{m}+\ellone{m}{v}\geq\ellone{u}{m}
\end{equation}
implies $m=v$, since $v=\proj{L}{u}$. On the other hand, $s\in u,w$ implies $s\in m$ -- a contradiction.

Thus, we have shown that $L=\half{a^\ast}$ is contained in $\half{s^\ast}$. Equivalently, $a^\ast\leq s^\ast$, which is the same as $s\leq a$.
\end{proof}
The same kind of reasoning yields:
\begin{lemma}\label{lemma:general projection, intersecting} Suppose $K,L$ are non-empty convex subsets of $\dual{P}$. If $K\cap L\neq\varnothing$, then $\proj{K}{L}=\proj{L}{K}=K\cap L$.
\end{lemma}
\begin{proof} Clearly, if $v\in K\cap L$ then $\proj{L}(v)=v$, so $K\cap L\subset\proj{L}{K}$. For the reverse inclusion, suppose $v\in\proj{L}{K}$ and write $v=\proj{L}{u}$, $u\in K$. Pick any point $w\in K\cap L$. Setting $m=\med{w}{v}{u}$ we note that $m\in L$ (because $w,v\in L$) and
\begin{equation*}
	\ellone{u}{v}=\ellone{u}{m}+\ellone{m}{v}\geq\ellone{u}{m}\,.
\end{equation*}
The uniqueness of projection forces $v=\proj{L}{u}$ to coincide with $m$. However, since $w,u\in K$ we also have $m\in K$, showing $v\in K\cap L$.
\end{proof}

\subsubsection{The Coherent Projection}\label{coherent projection} We need to study a technical notion motivated by the necessity in correcting the observation of the current state as explained in section \ref{subsection:intro to model spaces and coherence}. We recall the following standard notation for partially ordered sets:
\begin{equation}
	\up{a}=\set{p\in P}{a\leq p}\,,\quad
	\down{a}=\set{q\in P}{q\leq a}\,,
\end{equation}
and
\begin{equation}
	\up{A}=\bigcup_{a\in A}\up{a}\,,\quad
	\down{A}=\bigcup_{a\in A}\down{a}\,.
\end{equation}
Note that in a poc set $P$, one has the identities
\begin{equation}
	\up{A^\ast}=\down{A}^\ast\,,\quad
	\down{A^\ast}=\up{A}^\ast\,.
\end{equation}

For an arbitrary subset $A$ of $P$ we can define the following `correction' of $A$:
\begin{prop}[Coherent Projection]\label{prop:coherent projection} Let $P$ be a finite poc set and $A\subseteq P$ be any subset. Then the set
\begin{equation}\label{eqn:coherent projection}
	\coh{A}:=\up{A}\minus\down{A^\ast}=\up{A}\minus\up{A}^\ast
\end{equation}
Is coherent, and satisfies the following properties:
\begin{enumerate}
	\item if $A$ is coherent then $\coh{A}=\up{A}$;
	\item $\up{\coh{A}}=\coh{A}$;
	\item $\coh{\coh{A}}=\coh{A}$;
	\item if $A\in P^\circ$ then $A=\coh{A}=\up{A}$.
\end{enumerate}
\end{prop}
\begin{proof} To show the coherence of $\coh{A}$, let $b,c\in\coh{A}$ with $b\leq c^\ast$; if $b\in\up{A}$ then $c^\ast\in\up{A}$ and $c\in\up{A}^\ast$, contradicting $c\in\coh{A}$.

For (1), $A$ is coherent iff $\up{A}$ and $\down{A^\ast}$ are disjoint. Therefore, if $A$ is coherent then $\coh{A}=\up{A}\minus\down{A^\ast}=\up{A}$.

For (2), let $a\in\coh{A}$ and $a\leq b$. Then $a\in\up{A}$ implies $b\in\up{A}$, and it suffices to verify $b\notin\up{A}^\ast$. Indeed, were there $c\in A$ with $b\ast\geq c$ then $a\leq b\leq c^\ast$ would have given $a\in\down{A^\ast}$ in contradiction of $a\in\coh{A}$. Thus, $b\in\coh{A}$, as required.

For (3) since $\coh{A}$ is coherent we have $\coh{\coh{A}}=\up{\coh{A}}$ by substituting $\coh{A}$ instead of $A$ in (1), and then we apply (2).

Finally, for (4), $A\in P^\circ$ means $A$ is a coherent complete $\ast$-selection, so $\coh{A}=\up{A}$ by (1) and it remains to show $\up{A}=A$. Were there $b\in\up{A}$ with $b\notin A$ we would have had $b^\ast\in A$, since $A$ is a complete $\ast$-selection. But then we would also have had $b,b^\ast\in\coh{A}$, contradicting the coherence of $\coh{A}$.
\end{proof}

\subsubsection{Computing the Projection Maps}
For a vertex $u\in P^\circ$ and any subset $A\subset u$, one defines:
\begin{equation}\label{eqn:flip}
	\flip{u}{A}:=(u\minus A)\cup A^\ast
\end{equation}
Clearly, $\flip{u}{A}$ is a $\ast$-selection. It is easily verified that $\flip{u}{A}$ is coherent if and only if there exists no pair $a\in A$ and $b\in u\minus A$ satisfying $b<a$. This observation was first made in \cite{Sageev-thesis}, leading to the following results in our setting:
\begin{lemma}\label{lemma:flipping} Let $P$ be a finite weak poc set and let $u\in P^\circ$ be any vertex. Then the set $N(u)$ of vertices adjacent to $u$ in $\Gamma=\dual{P}$ coincides with the set of all $\flip{u}{a}$, $a$ ranging over the \emph{minset of $u$}:
\begin{equation}\label{eqn:minset}
	\min(u):=\set{a\in u}{b<a\THEN b\notin u}
\end{equation}
\end{lemma}
More generally, the cubes in $\cube{P}$ are characterized as follows:
\begin{lemma}\label{lemma:where cubes come from} Let $P$ be a finite weak poc set and $u\in P^\circ$ be a vertex. Then the cubes of $\cube{P}$ incident to $u$ are in one-to-one correspondence with the transverse subsets of $\min(u)$.
\end{lemma}
A particular application of these observations is an explicit construction of a geodesic path in $\Gamma$ emanating from a given vertex $u$ and terminating at its unique closest point projection $\proj{\half{T}}{u}$:
\begin{prop}\label{prop:constructing geodesics} Let $P$ be a finite weak poc set and suppose $u\in P^\circ$ is a vertex. Let $T$ be a coherent subset of $P$. Then the following algorithm constructs a shortest path in $\Gamma$ from $u$ to $K=\half{T}$:
\begin{enumerate}
	\item Find an element $b\in T\minus u$; if no such element, stop and output $u$.
	\item Find an element $c\leq b^\ast$ with $c\in\min(u)$;
	\item Replace $u$ by $\flip{u}{c}$ and go to the first step.
\end{enumerate}
\end{prop}
\begin{proof} We have $u\in K$ iff $T\subset u$, which provides the stopping condition for the algorithm. Now, if $u\notin K$ and $b\in T\minus u$ then for all $v\in K$ one has $v\in\half{b}$ and $u\in\half{b^\ast}$. Since $c\leq b^\ast$, we have $u\in\half{c}\subseteq\half{b^\ast}$, implying $v\in\half{c^\ast}$ and $c\in u\minus v$. As a result:
\begin{equation}
	\ellone{v}{\flip{u}{c}}=\ellone{v}{u}-1
\end{equation} 
Having reduced $\ellone{u}{v}$ by a unit for all $v\in K$, we have reduced $\ellone{u}{K}$ by a unit as well.
\end{proof}
\begin{cor}[Projection of a Point]\label{cor:projection of a point} Let $P$ and $T$ be as above. Then the closest point projection to $K=\half{T}$ is given by the formula:
\begin{equation}
	\proj{K}{u}=(u\minus\down{T^\ast})\cup\up{T}=(u\cup\up{T})\minus\down{T^\ast}
\end{equation}
\end{cor}
\begin{proof} The second equality follows from the DeMorgan rules and the fact that $\up{T}\cap\down{T^\ast}=\varnothing$ (since $T$ is coherent). 

Set $K=\half{T}$ and proceed by induction on $\ellone{u}{K}$. If $\ellone{u}{K}=0$, then $u\in K$ and therefore $T\subset u$. In addition, $u$ is coherent and we conclude $\down{T^\ast}\cap u=\varnothing$, leaving us with
\begin{equation*}
	u\minus\down{T^\ast}\cup T=u\cup T=u\,,
\end{equation*}
as desired. Now suppose $n:=\ellone{u}{K}>0$. By the preceding proposition, there is $a\in\down{T^\ast}\cap u$ such that $v:=\flip{u}{a}\in P^\circ$, $\ellone{v}{K}=n-1$, and $\proj{K}{u}=\proj{K}{v}$. We thus have:
\begin{equation*}
	\proj{K}{u}=\proj{K}{v}=(v\minus\down{T^\ast})\cup\up{T}=(u\minus\down{T^\ast})\cup\up{T}\,,
\end{equation*}
the last equality being due to $a\in T^\ast$ and $a^\ast\in T$. Thus, the first identity has been proved.
\end{proof}

\medskip
\subsubsection{Projecting a Convex Set to a Convex Set}
\begin{prop}\label{prop:general projection} Let $K,L$ be non-empty convex subsets with $L=\half{S}$ and $K=\half{T}$. Then
\begin{equation}\label{eqn:general projection}
	\begin{array}{rcl}
	\proj{K}{L}&=&\half{(\up{S}\cup\up{T})\minus\down{T^\ast}}\\
		&=&\half{T}\cap\half{\up{S}\minus\up{T}^\ast}
	\end{array}
\end{equation}
\end{prop}
\begin{proof} Since $T$ is coherent, $\up{T}$ and $\down{T^\ast}=\up{T}^\ast$ are disjoint. This allows us to write:
\begin{eqnarray*}
	\half{(\up{S}\cup\up{T})\minus\up{T}^\ast}
		&=&	\half{\up{T}\cup(\up{S}\minus\up{T}^\ast)}\\
		&=&	\half{\up{T}}\cap\half{\up{S}\minus\up{T}^\ast}
\end{eqnarray*}
and the second equality in \eqref{eqn:general projection} follows from the identity $\half{T}=\half{\up{T}}$. Denote $R=\up{S}\minus\up{T}^\ast$ and $N=\half{R}$.

For every $u\in L=\half{S}$ we have $\up{S}\subset u$, implying $\proj{K}{u}$ contains $\up{T}\cup R$, by corollary \ref{cor:projection of a point}. Thus, $\proj{K}{L}\subset K\cap N$, as required.

For the converse, observe that the case $K\cap L\neq\varnothing$ was already dealt with in lemma \ref{lemma:general projection, intersecting}: if $K\cap L\neq\varnothing$, then 
\begin{equation*}
	\proj{K}{L}=K\cap L=\half{\up{S}}\cap\half{\up{T}}=\half{\up{S}\cup\up{T}}
\end{equation*}
In particular, $\up{S}\cup\up{T}$ is coherent, and hence does not intersect $\up{T^\ast}$, and the formula \eqref{eqn:general projection} holds.

Thus we may henceforth assume $K\cap L=\varnothing$. Equivalently, $\up{S}\cap\down{T^\ast}\neq\varnothing$. In fact, by lemma \ref{lemma:supporting halfspace} we have $\up{S}\cap\down{T^\ast}=\sep{A,B}$.

Starting with $v\in K\cap N$ we must show $v\in\proj{K}{L}$. Set $u=\proj{L}{v}$, $w=\proj{K}{u}$, and $m=\med{u}{v}{w}$. Then $m\in K$ since $v,w\in K$. Since $K\cap L=\varnothing$, we have $\ellone{u}{v}>0$ and $\ellone{u}{w}>0$. Consider the point $m$: we have $m\in I(u,w)$ and $m\in K$; by the choice of $w$, $m$ must equal $w$ and therefore $w\in I(u,v)$. Thus, $w=\proj{K}{u}\in I(u,v)$ and $u=\proj{L}{w}$. By proposition \ref{prop:gates exist}, the pair $u,w$ is a gate for $K,L$ and we have
\begin{equation*}
	u\minus w=\sep{L,K}=\up{S}\cap\down{T^\ast}\,.
\end{equation*}

Consider an element $a\in v\minus u$. If $\half{a}\cap L\neq\varnothing$, pick $u'\in\half{a}\cap L$. Then $m=\med{u}{v}{u'}$ will satisfy $m\in\half{a}\cap L$ as well as
\begin{equation*}
	\ellone{v}{L}=\ellone{v}{u}=\ellone{v}{m}+\ellone{m}{u}\,.
\end{equation*}
Now, $\ellone{u}{m}>0$ since $u\in\half{a^\ast}$ and a contradiction to $u\proj{L}{v}$ is obtained. Thus, $\half{a}\cap L$ must be empty, which means $L\subseteq\half{a^\ast}$. Applying lemma \ref{lemma:supporting halfspace} we obtain $a^\ast\in\up{S}$.

Overall, we have shown that  $v\minus u\subseteq\up{S}^\ast$. We will now verify that $v\minus w=\varnothing$, finishing the proof. Indeed, were it not so, there would have been $h\in v\minus w$. On one hand, $w\in I(u,v)$ implies $v\minus w\subset v\minus u$, and hence $h^\ast\in\up{S}$. On the other hand, $h\notin w$ means $h^\ast\in w$ and therefore $h^\ast\notin\sep{L,K}=\up{S}\cap\up{T}^\ast$, which forces $h^\ast\in R$. Since  $R\subset v$ (by choice of $v$), we have $h^\ast\in v$, contradicting our choice of $h$.
\end{proof}
We will need the following technical corollary for the purposes of propagation:
\begin{cor}\label{cor:projection by propagation} Let $S,T\subset P$ be subsets and suppose $S$ is coherent. Let $L=V(S)$ and $K=V(\coh{T})$. Then:
\begin{equation}
	\proj{K}{L}=(\up{S}\cup\up{T})\minus\up{T}^\ast=(\up{S}\minus\up{T}^\ast)\cup\coh{T}
\end{equation}
\end{cor}
\begin{proof} Recall that $\coh{T}=\up{T}\minus\up{T}^\ast$, and set $J=\up{T}\cap\up{T}^\ast$, so that $\up{T}=\coh{T}+J$ and $\up{T}^\ast=\coh{T}^\ast+J$. Then,
\begin{eqnarray*}
	(\up{S}\cup\up{T})\minus\up{T}^\ast
		&=& ((\up{S}\cup\coh{T}\cup J)\minus\coh{T}^\ast)\minus J\\
		&=& (\up{S}\cup\coh{T})\minus\coh{T}^\ast
\end{eqnarray*}
Since $\up{\coh{T}}=\coh{T}$, the last expression equals $\proj{K}{L}$, by the preceding proposition. The proof of the second equality is similar.
\end{proof}


\end{document}